\pdfoutput=1

\documentclass[11pt]{article}

\usepackage[preprint]{acl}

\usepackage{times}
\usepackage{latexsym}

\usepackage[T1]{fontenc}

\usepackage[utf8]{inputenc}

\usepackage{microtype}

\usepackage{inconsolata}

\usepackage{amsthm,amssymb,amsmath,mathtools}
\usepackage{xspace}
\mathtoolsset{showonlyrefs=true}
\usepackage{enumitem}
\usepackage{dsfont}
\usepackage{booktabs}
\usepackage{color}
\usepackage{multirow}
\usepackage{caption}
\usepackage{subcaption}
\usepackage{pgf, tikz, adjustbox}
\usepackage{tcolorbox}
\usepackage{xcolor}
\usepackage{float}
\usepackage{hyperref}

\newcommand{\EE}{\mathbb{E}}
\newcommand{\PP}{\mathbb{P}}
\newcommand{\RR}{\mathbb{R}}

\newcommand{\bx}{\boldsymbol{x}}
\newcommand{\by}{\boldsymbol{y}}
\newcommand{\bv}{\boldsymbol{v}}

\newcommand{\cE}{\mathcal{E}}
\newcommand{\cI}{\mathcal{I}}
\newcommand{\cM}{\mathcal{M}}
\newcommand{\cV}{\mathcal{V}}

\newcommand{\reg}{\mathrm{reg}}

\newcommand{\eg}{\emph{e.g.}\xspace}
\newcommand{\ie}{\emph{i.e.}\xspace}

\newcommand{\one}{\mathds{1}}

\newtheorem{proposition}{Proposition}
\newtheorem{theorem}{Theorem}
\newtheorem{definition}{Definition}

\title{Uncertainty in Language Models: Assessment through Rank-Calibration}

\author{Xinmeng Huang\thanks{The first two authors are listed alphabetically. Correspondence to: Xinmeng Huang <xinmengh@sas.upenn.edu> and Shuo Li <lishuo1@seas.upenn.edu>.}\thanks{University of Pennsylvania, Philadelphia (PA),
US} \And Shuo Li$^{*\dag}$ \And Mengxin Yu$^\dag$ \And  Matteo Sesia\thanks{University of Southern California,  Los Angeles (CA),
US} \And
Hamed Hassani$^\dag$ \AND Insup Lee$^\dag$ \And Osbert Bastani\thanks{Collaborative advising.}$^\dag$ \And Edgar Dobriban$^{\S\dag}$
  }

\begin{document}
    \maketitle

\begin{abstract}
Language Models (LMs) have shown promising performance in natural language generation.
However, as LMs often generate incorrect or hallucinated
responses, it is crucial to correctly quantify their uncertainty in responding to given inputs. 
In addition to verbalized confidence elicited via prompting, 
many uncertainty measures (\eg, semantic entropy and affinity-graph-based measures) have been proposed. 
However, these measures can differ greatly, and it is unclear how to compare them, partly because they take values over different ranges (\eg, $[0,\infty)$ or $[0,1]$). 
In this work, we address this issue by developing a novel and practical framework, termed {\em Rank-Calibration}, to assess uncertainty and confidence measures for LMs.
Our key tenet is that higher
uncertainty (or lower confidence) should imply lower generation quality, on average.
Rank-calibration quantifies deviations from this ideal relationship in a principled manner, without requiring ad hoc binary thresholding of the correctness score (\eg, ROUGE or METEOR).
The broad applicability and the granular interpretability of our methods are demonstrated empirically. \href{https://github.com/shuoli90/Rank-Calibration/tree/main}{\bf The code to replicate
our experiments is here}.
\end{abstract}

\section{Introduction}\label{sec:introduction}
Language Models (LMs), especially
Large Language Models (LLMs), have shown promising performance in Natural Language Generation (NLG). 
These models, fitted on huge text corpora, can produce responses resembling those of humans~\cite{touvron2023llama2,gpt-4}. 
However, since LMs often generate wrong or hallucinated responses~\citep{weidinger2021ethical,xiao2021hallucination,huang2024one}, 
it is crucial to correctly quantify their level of uncertainty in responding to particular inputs.

\definecolor{blue}{RGB}{30,100,255}
\definecolor{shallow red}{RGB}{249, 127, 114}
\begin{figure}[ht]
    \centering
    \hspace{-4mm}
    \includegraphics[width=0.53\linewidth]{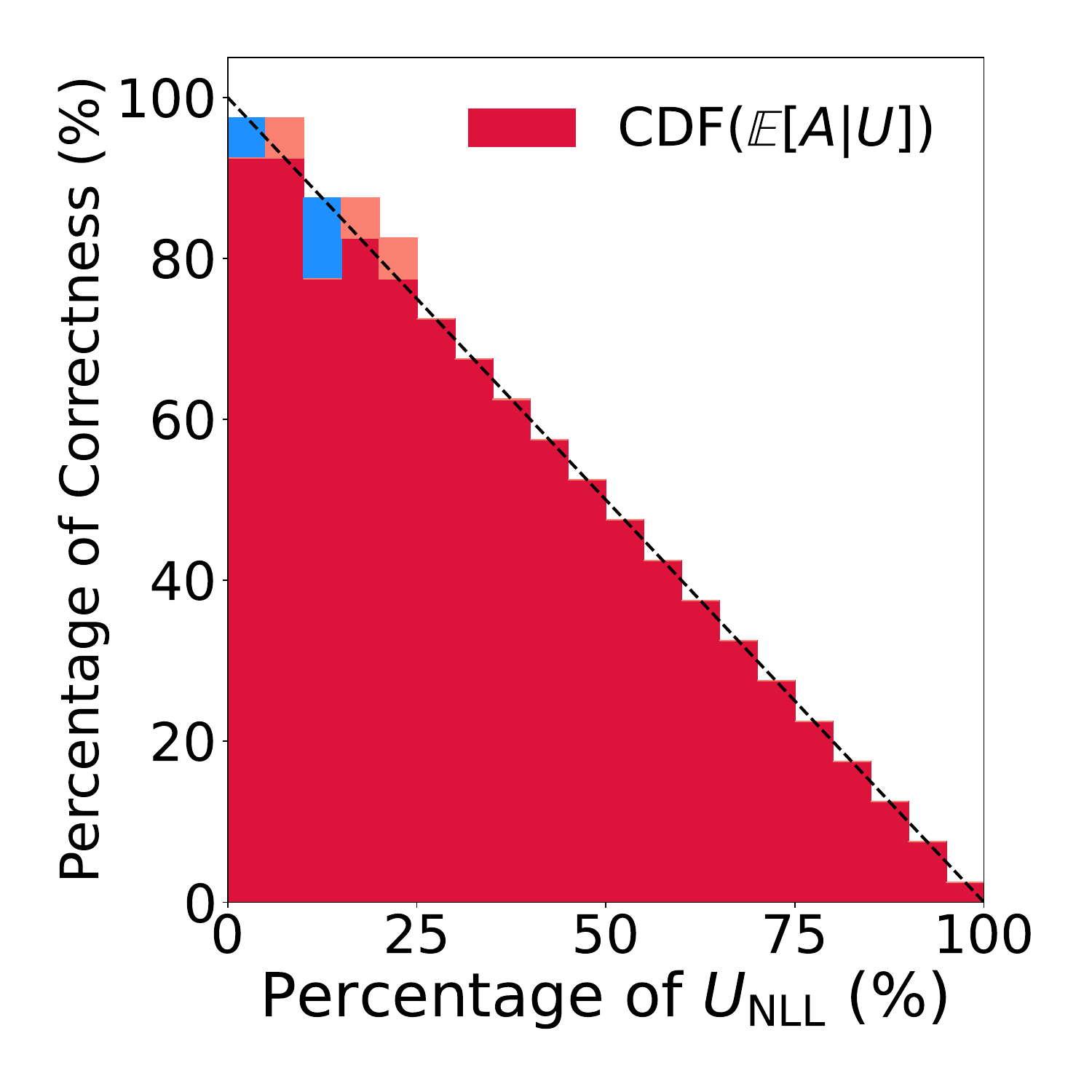}
    \hspace{-3mm}
    \includegraphics[width=0.53\linewidth]{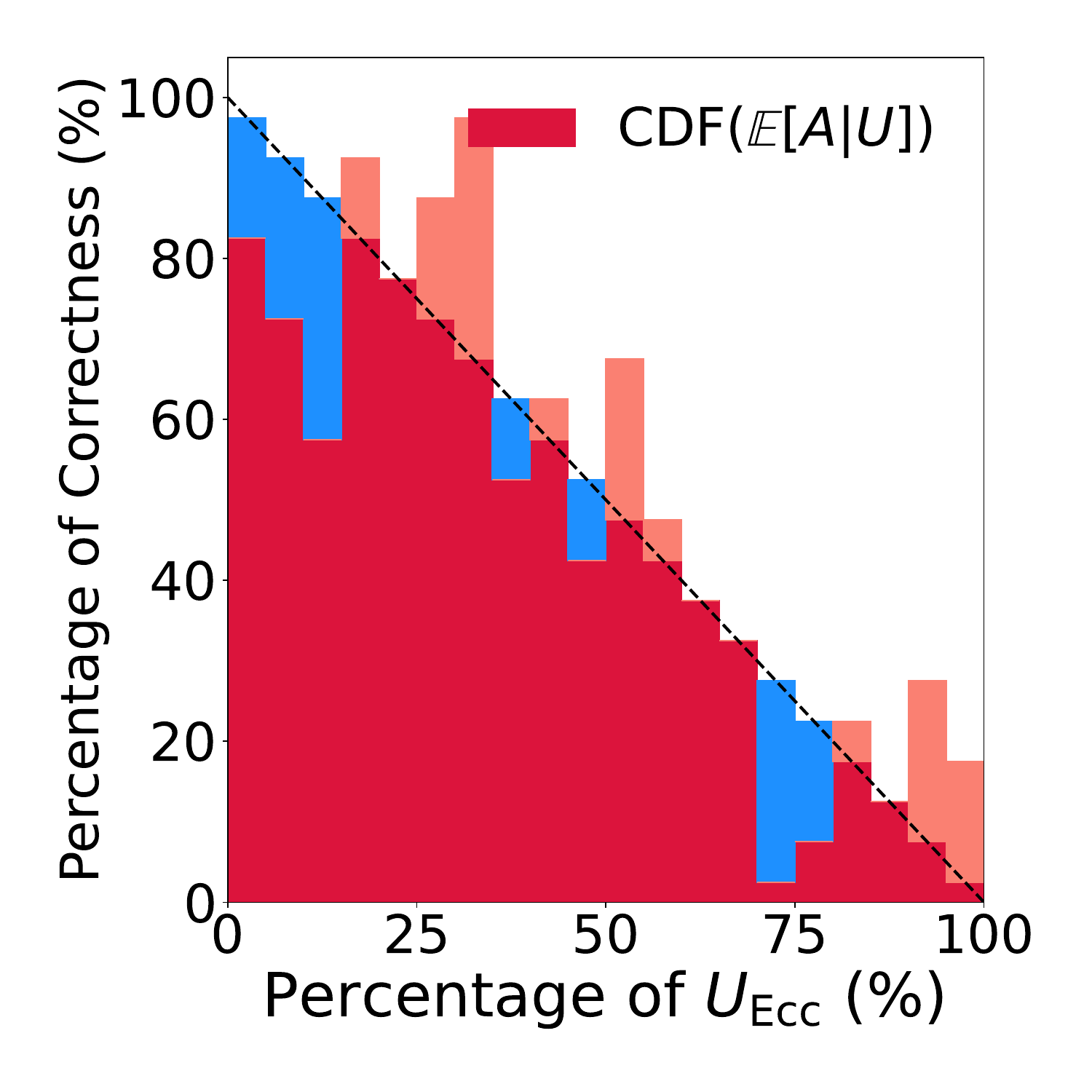}
    \hspace{-4mm}
    \vspace{-3mm}
    \caption{
    \emph{Indication diagrams} comparing two uncertainty measures, $U_{\rm NLL}$ (negative log-likelihood) and $U_{\rm Ecc}$ (eccentricity), for the GPT-3.5-turbo model on the TriviaQA benchmark. 
    The red bars indicate the average correctness of different outputs, as a function of the corresponding relative uncertainty levels. The \textcolor{blue}{blue} and \textcolor{shallow red}{shallow red} areas---deviating from the anti-diagonal line---indicate where the uncertainty measures are over-optimistic and pessimistic, respectively. Their sum is our \emph{rank-miscalibration} metric (\ie, \ref{eqn:rank-ece}), which here is lower for {\em $U_{\rm NLL}$ than $U_{\rm Ecc}$}.
    See Sec.~\ref{sec:rce+diagram} for details. 
    }
    \vspace{-5mm}
    \label{fig:chat_trivia}
\end{figure}

Uncertainty quantification is well-explored in supervised learning, specifically in 
classification~\cite[\eg,][etc]{lichtenstein1977calibration,gal2016dropout,lakshminarayanan2017simple}.
In classification,
a \emph{confidence measure}
is an estimate
of the probability that the predicted class $\widehat Y$ matches the true class label $Y$~\cite{lichtenstein1977calibration,lee2023t}.  
A confidence measure $C$
is considered \emph{calibrated} 
if it reflects the probability of correct prediction, \ie, 
$\PP(\widehat Y\hspace{-2pt} =\hspace{-2pt} Y\hspace{-2pt}\mid\hspace{-2pt}  C)\hspace{-2pt} =\hspace{-2pt} C$, for all values in $C$'s range. 
The Expected Calibration Error~\eqref{eqn:ece} measures the miscalibration of a confidence measure~\cite{frank2015regression,naeini2015obtaining}:
\begin{equation}\label{eqn:ece}
    \EE_{C}\hspace{-2pt} \left[\left|\PP(\widehat Y\hspace{-2pt} =\hspace{-2pt} Y\hspace{-2pt} \mid \hspace{-2pt} C)\hspace{-2pt} -\hspace{-2pt}C\right|\right].\tag{ECE}
\end{equation}

In classification, confidence measures are predominantly built on model logits~\citep{guo2017calibration,kull2019beyond}.
However, these methods are less suitable for NLG tasks. 
First, the label space is often too large to assess correctness via $\widehat Y= Y$, since LMs produce potentially long textual responses $\widehat Y$ for any given input. 
Second, for LMs, logits encode 
the likelihood of selecting the next token and do not necessarily capture linguistic sense~\cite{mielke2022reducing}. 
Third, even hand-crafted prompts intended to make LMs express confidence explicitly may not lead to reliable confidence values because elicitation is heavily tied to prompt formats~\cite{zhao2021calibrate,xiong2023llms}.

Recent works have studied \emph{uncertainty measures} as an alternative to confidence measures.
These capture the ``dispersion'' of  an LMs' potential outputs for a fixed input. 
\citet{kuhn2023semantic} introduce \emph{semantic entropy}, which incorporates linguistic invariances arising from the shared meaning of generated responses. \citet{lin2023generating} extend semantic entropy by leveraging the affinity matrices induced by entailment scores of generated outputs. 
Further,
\citet{chen2024inside} characterize differential entropy in the
embedding space with EigenScore, via the covariance of embeddings of potential responses. 

Uncertainty measures are more general and arguably more principled
than confidence measures for LMs, 
but they lack a universal assessment metric such as ECE. 
A key issue is that uncertainty measures are not necessarily commensurate. 
For instance, the semantic entropy  \citep{kuhn2023semantic} can take arbitrarily large
positive values, whereas the EigV measure of \citet{lin2023generating} depends on the number of responses generated.  
This makes it difficult to understand, evaluate, and compare uncertainty measures via a unified lens.

This paper develops a principled framework
to assess the quality of uncertainty and confidence measures for LMs.
We provide a novel and practical framework, termed \emph{Rank-Calibration}.
Specifically, our contributions are as follows.
\begin{itemize}[leftmargin=0.2in]
    \item We mathematically formalize the assessment of uncertainty/confidence measures for LMs in NLG tasks, going beyond binary correctness.

    \item We demonstrate empirically that existing assessment metrics (\eg, AUROC, ECE, etc) have several limitations, including a heavy dependence on the LM's performance, instability caused by ad hoc binarization of correctness scores, and incompatibility with diverse uncertainty ranges. 
    
    \item We address these limitations by starting from a basic principle: lower uncertainty/higher confidence should indicate higher-quality generation. We thus propose assessing uncertainty measures in terms of rank-calibration and introduce a suitable metric, the Rank-Calibration Error \eqref{eqn:rank-ece}.
    
    \item To make rank-calibration practical, we introduce the \ref{eqn:empirical-erce}---an estimate of \ref{eqn:rank-ece} based on a finite dataset. 
    Moreover, we introduce novel indication diagrams, previewed in Fig.~\ref{fig:chat_trivia}, that intuitively visualize the deviation of any uncertainty/confidence measure from the monotonicity required for rank-calibration.
    
    \item We experimentally demonstrate the broader applicability and granular interpretability of our proposed methods. Comprehensive ablation studies are conducted to examine its robustness.
\end{itemize}

\section{Correctness and Uncertainty for LMs}
\label{uqlm}
Let $\cV$ be the token vocabulary of an LM and $\cV^\star\hspace{-1pt}:=\hspace{-1pt}\cup_{\ell\geq 1}\cV^\ell$ the space of sequences of arbitrary length.
Given a query $\bx\in \cV^\star$, an LM $\cM$
can generate output $\widehat \by \hspace{-1pt}\triangleq \hspace{-1pt} (\widehat y_\ell)_{\ell\geq 1} \hspace{-2pt}\in \hspace{-2pt} \cV^\star$ 
by sequentially sampling from the distribution $\PP(\widehat \by \hspace{-2pt}\mid\hspace{-2pt} \bx)\hspace{-1pt}:=\hspace{-1pt} \prod_{\ell\geq 1} \PP(\widehat y_\ell \hspace{-2pt}\mid\hspace{-2pt} \bx, \widehat y_{<\ell})$. 
Here, $\widehat y_\ell\in\cV$ is the $\ell$-th generated token and $\PP\triangleq\PP^{\cM}$ is the generative distribution of $\cM$. 

We work with a deterministic {\em correctness function} $A\hspace{-2pt}:\hspace{-2pt}\cV^\star \times \cV^\star \hspace{-2pt}\to\hspace{-2pt} \RR$ mapping each pair $(\bx; \widehat \by)$ to a correctness
value $A(\bx; \widehat \by)$.
 In practice, correctness is often not a binary variable in NLG  tasks and can be assessed in at least two different ways. For the reader's convenience, the concepts and notations used in the paper are summarized in Table~\ref{tab:notations}.
\begin{itemize}
[leftmargin=0.2in]
    \item {\bf Reference matching.} Given certain reference answers $\{\by^{(m)}\}_{m=1}^M$ associated with $\bx$, a similarity score 
    between the output $\widehat \by$ and $\{\by^{(m)}\}_{m=1}^M$ can be interpreted as a correctness value. Similarity scores commonly utilized for this purpose include the {\em Rouge} score, {\em BLEU} score, and outputs of other discriminative LMs. 

    \item {\bf Human evaluation.} Correctness or quality may be evaluated by human experts, possibly integrating multiple opinions  (\eg, averaging). 
    This approach does not require reference answers and is as ``trustworthy'' as the humans involved.
\end{itemize}

\begin{table}[ht]
\centering
\begin{tabular}{ll}
\toprule
Notation &\hspace{-1em} Description \\
\midrule
$\cV$ &\hspace{-1em} Token vocabulary \\
$\cV^\star$ &\hspace{-1em} Space of token sequences\\
$\bx$ &\hspace{-1em} Input context, $\bx \in \cV^\star$ \\
$\widehat{\by}$ &\hspace{-1em} Gen. output $\widehat{\by} = (\widehat{y}_\ell)_{\ell \geq 1} \in \cV^\star$ \\
$\PP\triangleq\PP^{\cM}$ &\hspace{-1em} Generative dist. of LM $\cM$ \\
$A(\cdot\,; \,\cdot)$ &\hspace{-1em} A deterministic correctness function \\
$\{\by^{(m)}\}_{m=1}^M$ &\hspace{-1em} Reference answers for input $\bx$ \\
$U^{\cM}(\bx; \widehat{\by})$ &\hspace{-1em} Uncertainty measure  for LM $\cM$ \\
$C^{\cM}(\bx; \widehat{\by})$ &\hspace{-1em} Confidence measure for LM $\cM$ \\
$\reg(u)$ &\hspace{-1em} Regression fn. $\EE_{\bx,\widehat \by}[A\mid U\hspace{-2pt}=\hspace{-2pt}u]$ \\
\bottomrule
\end{tabular}
\caption{Summary of notations.}
\vspace{-4mm}
\label{tab:notations}
\end{table}

An \emph{uncertainty measure} 
is a (possibly random) function $U^{\cM}\hspace{-2pt}:\hspace{-2pt}\cV^\star \times \cV^\star \to \RR,(\bx;\widehat \by)\mapsto U^{\cM}(\bx;\widehat \by)$ associated with the LM that maps any  pair $(\bx;\widehat \by)$ to an
uncertainty value.\footnote{In special cases, the uncertainty measure
may only depend on the input $\bx$ and the LM $\cM$, 
not the output $\widehat \by$.} 
We will omit $\cM$ and write $U(\bx; \widehat \by)$, $\PP(\cdot \hspace{-2pt} \mid \hspace{-2pt} \bx)$ when the choice of the LM is clear. 
Some examples are reviewed below, while additional examples and details are in Appendix~\ref{app:measures}.
\begin{itemize}
[leftmargin=0.2in]
    \item {\bf NLL. } 
    In classification, the softmax of the last-layer logits determines a model's prediction~\cite {guo2017calibration}. In NLG tasks, one can  view the Negative Log-Likelihood (NLL),
    \begin{equation}
        U_{\rm NLL}(\bx, \widehat \by)\hspace{-2pt}:=\hspace{-2pt}-\ln(\PP(\widehat \by \mid \bx)),
    \end{equation}
    as an indicator of uncertainty 
    where $\widehat \by= (\widehat y_\ell)_{\ell \geq 1}$ is a generated response.  
    A natural extension accounting for the length of responses applies length normalization;
    this is also known as the \emph{Perplexity} measure~\citep{jelinek1977perplexity}.

\definecolor{generation}{RGB}{81,145,243}
\definecolor{correctness-justification}{RGB}{255, 217, 102}
\definecolor{uncertainty-quantification}{RGB}{169, 209, 142}
\definecolor{correctness-discretization}{RGB}{192, 0, 30}
\definecolor{evaluation}{RGB}{237, 125, 79}
\begin{figure*}[t]
    \centering
    \vspace{-5mm}
    \includegraphics[width=0.8\textwidth]{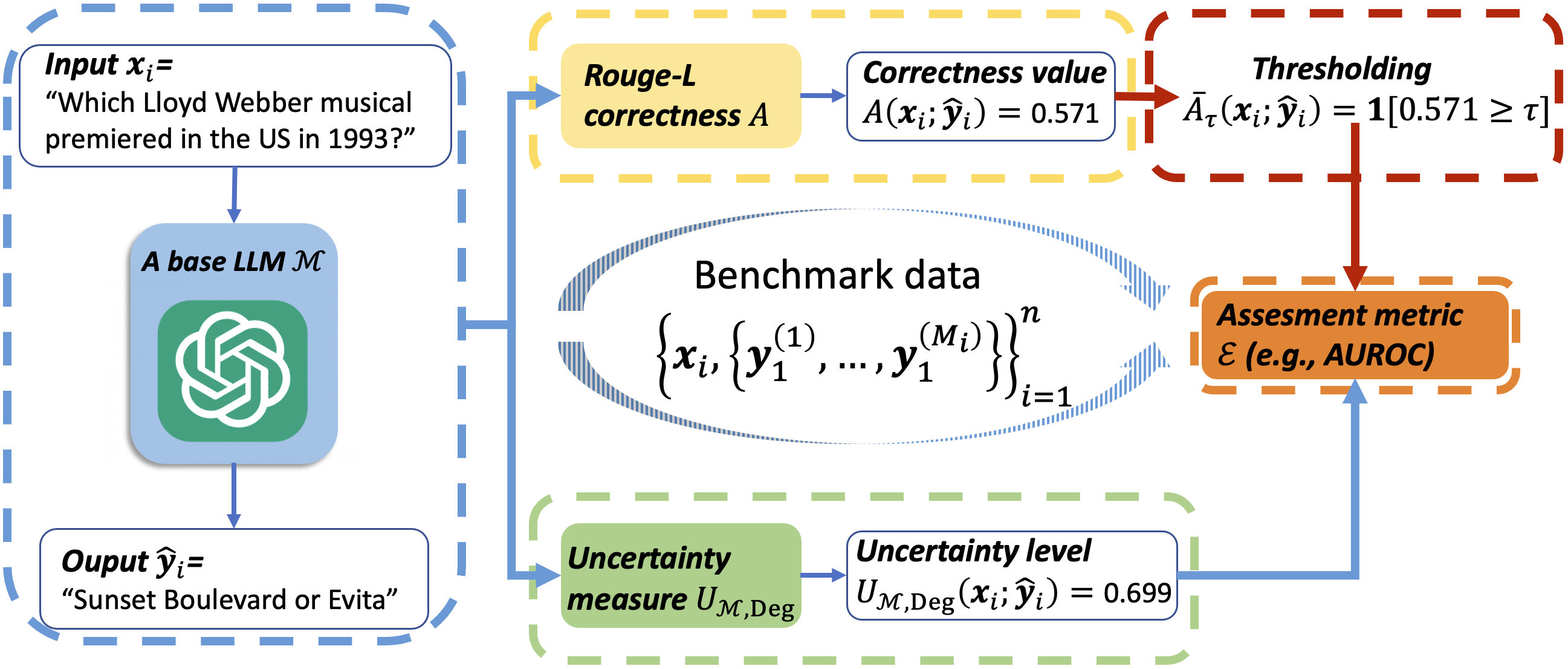}
    \caption{Common workflow for assessing the quality of an LM uncertainty/confidence measure. The key ingredients are: a base LM $\cM$ (\eg,  Llama-2-7b-chat), a correctness function $A$ (\eg, the Rouge-L score), a benchmark dataset $\{\bx_i,\{\by_i^{(m)}\}_{m=1}^{M_i}\}_{i=1}^n$ (\eg, TriviaQA), an assessment metric $\cE$ (\eg, AUROC), and the uncertainty measure $U$ (\eg, $U_{\rm Deg}$). 
   The workflow proceeds in five stages: \textcolor{generation}{generation}, \textcolor{correctness-justification}{correctness calculation}, \textcolor{correctness-discretization}{correctness discretization}, \textcolor{uncertainty-quantification}{uncertainty quantification}, and  \textcolor{evaluation}{evaluation}. Notably, the threshold $\tau$ in \textcolor{correctness-discretization}{correctness discretization} is usually chosen heuristically \citep[etc]{kuhn2023semantic,xiong2023llms,lin2023generating}, which can be problematic, as demonstrated in Sec.~\ref{sec:case-study}. Our proposed \ref{eqn:rank-ece}-based assessment {\em removes} this stage by using the correctness values directly.}
    \label{fig:assessment-pipeline}
\end{figure*}

\item {\bf Entropy. } 
The predictive entropy of the distribution $\PP(\cdot \mid \bx)$ is large when the same input may lead to diverse outputs, and it is defined as
\begin{equation}\label{eqn:entropy}
    U_{{\rm E}}(\bx)\hspace{-1mm}:=\hspace{-1mm}-\EE_{\widehat \by\sim \PP(\cdot \mid \bx)}[\ln( \PP(\widehat \by \mid \bx))].
\end{equation}
\citet{malinin2021uncertainty} propose a variant of this, $U_{{\rm E{\text -}LN}}(\bx)$, utilizing the length-normalized log-likelihood $\ln( \PP(\widehat \by \hspace{-2pt}\mid \hspace{-2pt} \bx))/{\rm len}(\widehat \by)$. 
\citet{kuhn2023semantic} argue that different responses with the same meaning should be viewed as equals in this context, regardless of token-level differences. They propose the semantic entropy, 
\begin{equation}
    U_{{\rm SE}}(\bx)\hspace{-1mm}:=\hspace{-1mm}-\EE_{\widehat \by\sim \PP(\cdot \mid \bx)}[\ln( \PP(c(\widehat \by) \mid \bx))],
\end{equation}
where $c(\widehat \by)$ is the semantic concept of  $\widehat \by$, provided by another language modeling method.

\item {\bf Affinity graph.} \citet{lin2023generating} calculate uncertainty using a weighted adjacency graph built upon semantic affinities.
Consider an \emph{affinity} model $e$, mapping pairs of responses $\widehat \by, \widehat \by ^\prime $  
to values in $[0,1]$.
Given $K$ independent samples $\{\widehat \by^{(k)}\}_{k=1}^K$ 
 from $\PP(\cdot \hspace{-2pt}\mid \hspace{-2pt}\bx)$, the model $e$ induces a symmetric adjacency matrix $W\hspace{-2pt}=\hspace{-2pt}[w_{i,j}]_{i,j= 1}^K$, with  $w_{i,j}\hspace{-2pt}=\hspace{-2pt}(e(\widehat \by^{(i)};\widehat \by ^{(j)})+e(\widehat \by ^{(j)};\widehat \by^{(i)}))/2$ for all $i,j$. 
Let $D=[\mathds{1}[j=i]\sum_{k=1}^K w_{k,j}]_{i,j=1}^K$ 
be the corresponding degree matrix
and  $\{\lambda_k\}_{k=1}^K$ be the eigenvalues 
of the \emph{Laplacian} $L\hspace{-2pt}=\hspace{-2pt}I\hspace{-2pt}-\hspace{-2pt}D^{-1/2}WD^{-1/2}$. 
Then, the uncertainty measures proposed in \citet{lin2023generating}
include
\begin{align}
    &U_{\rm EigV}(\bx):=\sum_{k= 1}^K\max\{0, 1-\lambda
    _k\},\\
    &U_{\rm Deg}(\bx):=1-{\rm trace}(D)/K^2,\\
    &U_{\rm Ecc}(\bx):=\|[\bv_1,\bv_2,\dots, \bv_K]\|_2,
\end{align}
where $\{\bv_k\}_{k= 1}^K$ are suitable vectors associated with $L$, see \citet{lin2023generating}.
Intuitively, $U_{\rm EigV}(\bx)$ approximately counts the connected components in the graph represented by $W$, while $U_{\rm Deg}(\bx)$ and $U_{\rm Ecc}(\bx)$ reflect the diversity of outputs. 
\end{itemize}

The diverse uncertainty measures reviewed above produce outputs with different ranges. 
For instance, $U_{\rm NLL}$, $U_{\rm SE}$, and $U_{\rm EigV}$ can yield any number in $[0,\infty)$, whereas $U_{\rm Deg}$ and $U_{\rm Ecc}$ are bounded in $[0,1]$; see Fig.~\ref{fig:auroc_range} [bottom] for a visual illustration. 
This mismatch in output ranges motivates the need for a novel unified assessment framework.

As we shall see, our assessment framework can handle not only any uncertainty measure but also the closely related concept of \emph{confidence measures} \citep{zhao2021calibrate,mielke2022reducing,xiong2023llms}.
A confidence measure can be cast as a (possibly random) function $C^{\cM}\hspace{-2pt}:\hspace{-2pt}\cV^\star\times \cV^\star\hspace{-2pt}\to\hspace{-2pt} [0,1]$, $(\bx;\widehat \by )\hspace{-2pt}\mapsto\hspace{-2pt} C^{\cM}(\bx;\widehat \by)$ with output taking values in $[0,1]$.
Intuitively, confidence and uncertainty measures serve similar purposes, although in a complementary way---high confidence should correlate with low uncertainty, and vice versa.

With this notation in place, we are now ready to state our goals and give a more detailed preview of our proposed framework.
Given a benchmark dataset $\hspace{-2pt}\{(\bx_i, \{\by_i^{(m)}\}_{m=1}^{M_i})\}_{i=1}^n\hspace{-2pt}$, where each $M_i\geq 0$ denotes the number of reference answers for $\bx_i$, 
we aim to quantify the performance of an uncertainty measure $U$  (or a confidence measure $C$) as follows. 
First, we obtain the paired values of uncertainty and correctness $\{(U(\bx_i, \widehat \by_i), A(\bx_i; \widehat \by_i))\}_{i=1}^n$ by independently sampling $\widehat \by _i \hspace{-2pt}\sim \hspace{-2pt}\PP(\cdot \hspace{-2pt}\mid\hspace{-2pt} \bx_i)$ for each $1\hspace{-2pt}\leq \hspace{-2pt}i\hspace{-2pt}\leq\hspace{-2pt} n$.
Then, we evaluate
 $\cE(\{(U(\bx_i, \widehat \by_i) , A(\bx_i; \widehat \by_i))\}_{i=1}^n)$ for each $1\hspace{-2pt}\leq \hspace{-2pt}i\hspace{-2pt}\leq\hspace{-2pt} n$, using a suitable metric $\cE$.\footnote{
A common practice is to map the correctness values to $\{0,1\}$ by thresholding at an ad hoc value before feeding them into the evaluation metric; see Sec.~\ref{sec:case-study} for a discussion of the limitations of this approach.} 
To account for the randomness in sampling $\widehat \by_i$, we may draw multiple independent responses $\{\widehat \by_i^{(k)}\}_{k=1}^{K} \overset{iid}{\sim} \PP(\cdot \hspace{-2pt}\mid\hspace{-2pt} \bx_i)$ and take the average as the final result $\sum_{k=1}^K\cE(\{(U(\bx_i, \widehat \by_i^{(k)}),  A(\bx_i, \widehat \by_i^{(k)}))\}_{i=1}^n)/K$.

The closest works have been discussed in Sec.~\ref{sec:introduction} and~\ref{uqlm}, and more related works are reviewed in Appendix~\ref{app:related-works}.

\section{Limitations of Existing Assessments}\label{sec:case-study}

This section illustrates some limitations of existing assessments for LM uncertainty measures via a case study applying the \textit{GPT-3.5-turbo}~\citep{ouyang2022training} model on the \textit{TriviaQA} benchmark~\cite{joshi-etal-2017-triviaqa}. 
We use the validation set of TriviaQA, which contains $11,322$ question-answer pairs (after deduplication). We use the same prompt template as that in \citet{lin2023generating}. The template is shown in Appendix~\ref{app:datasets}.

The uncertainty measures examined here include the negative log-likelihood $U_{\rm NLL}$, the semantic entropy $U_{\rm SE}$~\cite{kuhn2023semantic}, the affinity-graph-based measures $U_{\rm EigV}$, $U_{\rm Ecc}$, and $U_{\rm Deg}$~\citep{lin2023generating}, with the affinity determined by the NLI model~\cite{he2021deberta}, and
the {verbalized confidence} $C_{\rm Verb}$~\citep{xiong2023llms};
see the definitions in Appendix~\ref{app:measures}. 
These include both white box and grey box measures,\footnote{The grey-box oracle refers to the access to model logits, which is partly feasible for commercial LMs, while the black-box oracle only relies on generated outputs.} as well as a diversity of prompt strategies.
We use the \textit{Rouge-L} score as the correctness function $A$. 
We follow a common assessment pipeline \citep {kuhn2023semantic,lin2023generating,xiong2023llms}, as depicted in Fig.~\ref{fig:assessment-pipeline}.
The assessment metrics are detailed in Appendix~\ref{app:metrics}.

\begin{figure}[ht]
\centering
  \includegraphics[width=\linewidth]{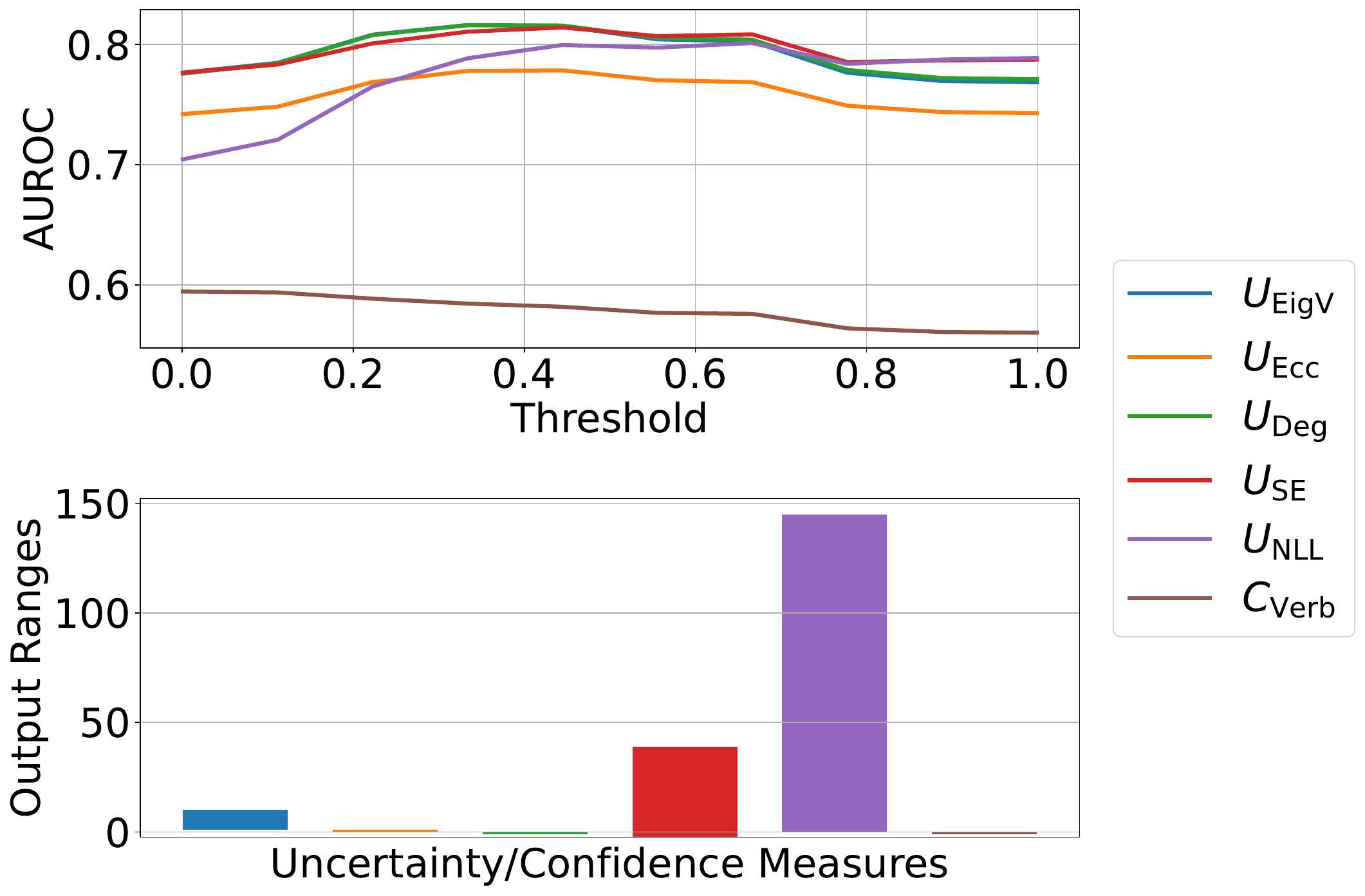}
\caption{Top: AUROCs of uncertainty/confidence measures with various thresholds. Bottom: Output ranges of uncertainty/confidence measures. Both results are for {\em GPT-3.5-turbo} on the {\em TriviaQA} benchmark.
}
\vspace{-4mm}
\label{fig:auroc_range}
\end{figure}

\paragraph{Ad hoc correctness thresholding.}
Most existing assessment metrics (\eg, AUROC, AUPRC, ECE, etc) are rooted in classification and require binary labels (\ie, $A\hspace{-2pt}\in \hspace{-2pt}\{\textit{True}\text{ or }\textit{False}\}$). Consequently, an ad hoc threshold $\tau \in \RR$ is often introduced to map continuous correctness values to binary labels, \ie, $\Bar{A}_\tau (\bx;\widehat \by)\hspace{-2pt}:=\hspace{-2pt}\one [A(\bx;\widehat \by)\hspace{-1pt}\geq \hspace{-1pt}\tau]$~\cite{lin2023generating, kuhn2023semantic}. 
Thus, the response is viewed as  \textit{correct} if the correctness value $A(\bx;\widehat \by)$ is at least $\tau$, and \textit{incorrect} otherwise. 

However, thresholding can lead to inconsistencies. 
Taking AUROC as an example, we plot the assessed results of uncertainty/confidence measures under varying thresholds in Fig.~\ref{fig:auroc_range} [top]. 
The relative AUROC results of distinct measures vary drastically with the choice of $\tau$.
For example, $U_{\rm NLL}$ appears inferior to other methods if $\tau < 0.2$, but it becomes the best measure if $\tau > 0.8$.
This is especially concerning given that there seems to be no principled way to set this threshold. 
The same limitation also affects other metrics (\eg, AUPRC, AUARC) and configurations; see Appendix~\ref{sec:auarcs}.

\paragraph{Diverse output ranges.}
The second limitation of existing assessments is rooted in the diverse output ranges of the uncertainty or confidence measures. As shown in Fig.~\ref{fig:auroc_range} [bottom],
the output ranges of different uncertainty measures vary significantly. For example, the values of $U_{\rm SE}$ can be higher than $100$ while the values of $U_{\rm Ecc}$ and $U_{\rm Deg}$ are small by definition.
This diversity of output ranges prevents the direct use of calibration-based metrics such as ECE,
which takes variables with inputs in $[0,1]$. 

\paragraph{Strong dependence on LM performance.} $\hspace{-3pt}$
While the quality of uncertainty/confidence measures should be disentangled from the generation performance of the LM, there is often a strong relation between the two concepts.
 We argue that 
 many existing metrics (\eg, AUROC, AUPRC, AUARC) can be misleading due to this entanglement.
Taking AUARC as an example, 
if the base LM is powerful
and all correctness values of its responses are high (\eg, within $[0.9, 1.0]$), 
then the evaluated AUARC will be high for any uncertainty/confidence measure, regardless of its quality.
This is undesirable because our goal is to provide an overall assessment of the uncertainty measure, which may in the future need to be applied to different LMs.
While the ECE metric provides a limited ``disentangling'' effect, in the sense that it can reflect that highly accurate models may be poorly calibrated (\ie, with high ECE values)~\citep{guo2017calibration}, it is not applicable to uncertainty measures in general.

\paragraph{Desiderata of evaluation.}\label{sec:desderata}
The aforementioned challenges suggest that the evaluation of LM uncertainty measures should take into account the following key desiderata: (1) avoidance of ad hoc correctness thresholding, (2) applicability to diverse output ranges of uncertainty measures, and (3) decoupling from the generative performance of the LM. Moreover, the evaluation framework should be practical.
We view these criteria as important, but {\em not necessarily exhaustive} for an ideal assessment. 
Future research may identify other requisites and further improve our framework accordingly.

\section{Rank-Calibration}
In this section, we introduce a novel assessment framework 
satisfying the criteria outlined in Sec.~\ref{sec:desderata}.

\subsection{Rank-Calibration \& RCE}
Define the regression function $\reg(\cdot)\hspace{-2pt}:\hspace{-2pt}\RR \hspace{-2pt}\to \hspace{-2pt}\RR, u\hspace{-2pt}\mapsto \hspace{-2pt}\EE_{\bx, \widehat \by}[A(\bx; \widehat \by )\mid U(\bx;\widehat \by)\hspace{-1pt}=\hspace{-1pt}u]$, representing the \emph{expected correctness level $A$ conditional on an uncertainty level} $U\hspace{-1pt}=\hspace{-1pt}u$. 
Here, $\bx$ 
is a random query sampled from the distribution associated with a specific benchmark dataset, 
while $\widehat \by \hspace{-1pt}\mid \hspace{-1pt} \bx \hspace{-1pt}\sim\hspace{-1pt} \PP(\cdot \hspace{-1pt}\mid\hspace{-1pt} \bx)$ is a random output sampled from the generative distribution of the LM.
We start from the observation that, ideally, 
a lower uncertainty level should correspond to higher generation accuracy. 
This is equivalent to saying that the regression function should ideally be \emph{monotone decreasing}.

Since $U$ is a random variable depending on $(\bx;\widehat \by)$,
$\reg(U)$ is also random. 
If $\reg(\cdot)$ is monotonically decreasing, then $U\leq u$ implies $\reg(U)\geq \reg(u)$. 
Thus, for any value $u$ in the range of $U$,
\begin{equation}\label{eqn:eqiv-dist}
\PP(U\leq u)=\PP(\reg(U)\geq \reg(u)).
\end{equation}
Equation~\eqref{eqn:eqiv-dist} 
suggests a direct relation between 
an uncertainty level $u$ 
and its corresponding expected correctness level $\reg(u)$. 
For example, 
for a value of $u$ in the 
in bottom $10\%$ of the distribution of $U$, 
the expected correctness level $\reg(u)\hspace{-2pt}=\hspace{-2pt}\EE[A\hspace{-1pt}\mid\hspace{-1pt} U\hspace{-2pt}=\hspace{-2pt}u]$ is in the top $10\%$ in the distribution of $\reg(U)\hspace{-2pt}=\hspace{-2pt}\EE[A\mid U]$. 
We call this desired property of uncertainty measures {\em Rank-Calibration}.
\begin{definition}[\sc Rank-Calibration]
    We say that an uncertainty measure $U$ is rank-calibrated if \eqref{eqn:eqiv-dist} holds for any $u$ in $U$'s range: on average, a lower uncertainty implies a higher generative quality.
\end{definition}
Rank-calibration is related to, yet distinct from, the usual notion of calibration in the classification context~\citep{lichtenstein1977calibration,guo2017calibration}. We defer the detailed discussion to Sec.~\ref{sec:comparison}.  We remark that the principle of rank calibration is also discussed in a concurrent work~\citep{zhang2024luq}. Unlike our work and \cite{zhang2024luq}, \cite{penha2021calibration} use the terminology ``rank'' to denote the relevance comparison of candidate responses in the binning of ECE calculation.

To quantify the distance of a given uncertainty measure from the ideal rank-calibration,
we propose the following Rank-Calibration Error \eqref{eqn:rank-ece}, inspired by ECE for calibration.
\begin{definition}[\sc Rank-Calibration Error]
The RCE of an uncertainty measure $U$ is defined as
\begin{equation}\label{eqn:rank-ece}
\EE_{U}\hspace{-2pt}\left[\left|\PP_{U'}(\reg(U^\prime)\hspace{-1pt}\geq \hspace{-1pt}\reg(U))\hspace{-1pt}-\hspace{-1pt}\PP_{U'}(U^\prime\hspace{-2pt}\leq\hspace{-1pt} U)\right|
\right], \tag{RCE}
\end{equation}
where $U^\prime$ is an independent copy of $U$.
\end{definition}

\paragraph{Extension to confidence measures.} 
While primarily motivated by uncertainty measures with incommensurate ranges, rank-calibration also applies to confidence measures. 
Ideally, 
{\em higher values of a confidence measure
should imply higher generation accuracy}. 
Thus, defining $\overline{\reg}(c):=\EE[A\mid C=c]$ 
for all $c$ in the range of $C$, 
we can adapt \ref{eqn:rank-ece} to 
\begin{equation}\label{eqn:rank-ece-c}
\EE_{C}\hspace{-2pt}\left[\left|\PP_{C^\prime}(\overline{\reg}(C^\prime)\hspace{-1pt}\geq \hspace{-1pt}\overline{\reg}(C))\hspace{-1pt}-\hspace{-1pt}\PP_{C^\prime}(C^\prime\hspace{-2pt}\geq\hspace{-1pt} C)\right|\right],
\end{equation}
where $C'$ is an independent copy of $C$.
This gauges deviations from 
the equivalence between $C\geq c$ and $\overline{\reg}(C)\geq \overline{\reg} (c)$.
Since rank-calibration provides a different perspective from calibration---see Sec.~\ref{sec:comparison}---\eqref{eqn:rank-ece-c} serves as a supplement to ECE in assessing confidence measures.

\subsection{Comparison with Classical Calibration}\label{sec:comparison}
For a binary correctness value function $A$ taking values in $\{0,1\}$,
rank-calibration relaxes classical calibration by absorbing all {\em strictly decreasing} transformations. 
\begin{theorem}\label{thm:equiv}
Suppose the correctness  function $A$ takes values in $\{0,1\}$.
    If an uncertainty measure $U$ is rank-calibrated, \ie, its  \ref{eqn:rank-ece} is zero,
    then there exists a unique strictly decreasing transformation 
    $g^\star\hspace{-2pt}:\RR\hspace{-1pt}\to\hspace{-1pt}[0,1]$ such that $C_{g^\star}:=g^\star(U)$ is calibrated, \ie, its ECE is zero. 
    If a confidence measure $C$ is calibrated, then for any strictly decreasing transformation $h\hspace{-2pt}:\RR\hspace{-1pt}\to\hspace{-1pt}\RR$, the induced uncertainty measure $U_h:= h(C)$ is rank-calibrated.
\end{theorem}
\begin{proof}
If $U$ is rank-calibrated, the regression function 
$u\mapsto\reg(u)\hspace{-1pt}=\hspace{-1pt}\EE[A\hspace{-2pt}\mid \hspace{-2pt}U\hspace{-2pt}=\hspace{-2pt}u]\in [0,1]$ is strictly 
decreasing over all values in $U$'s range with positive density (or mass).
Moreover, $\PP(A\hspace{-2pt}=\hspace{-2pt}1\mid \reg(U)\hspace{-2pt}=\hspace{-2pt}\reg(u))\hspace{-1pt}=\hspace{-1pt}\EE[A\hspace{-2pt}\mid\hspace{-2pt} U\hspace{-2pt}=\hspace{-2pt}u]\hspace{-1pt}=\hspace{-1pt} \reg(u)$. 
Therefore, $\reg(U)$ is a calibrated confidence measure,
and $\reg$ is strictly decreasing. The uniqueness follows as $\PP(A\hspace{-2pt}=1\mid g(U))=\EE[A\hspace{-2pt}\mid\hspace{-2pt} U]\hspace{-1pt}=\hspace{-1pt}\reg(U)$ for any strictly monotone function.

On the other hand, if $C$ is calibrated, then $C\hspace{-2pt}=\hspace{-2pt}\PP(A\hspace{-2pt}=\hspace{-2pt}1\hspace{-2pt}\mid \hspace{-2pt}C)\hspace{-2pt}=\hspace{-2pt}\EE[A\hspace{-2pt}\mid\hspace{-2pt} C]$ almost surely.
For any strictly decreasing $h$, we have $\EE[A\hspace{-2pt}\mid \hspace{-2pt}U_h\hspace{-1pt}]=\hspace{-1pt}\EE[A\hspace{-2pt}\mid\hspace{-2pt} C]\hspace{-1pt}=\hspace{-1pt}C$ almost surely because $h$ is a one-to-one map. Therefore, for any given $c$ and 
 uncertainty value $u_h\hspace{-1pt}=\hspace{-1pt}h(c)$, it holds almost surely that
\begin{align}
    &U_h=h(C) \leq u_h= h(c)\;\Longleftrightarrow\;C\geq c\\
    \;\Longleftrightarrow\; &\EE[A\mid C]\geq \EE[A\mid C=c]\\
    \;\Longleftrightarrow\; &\EE[A\mid U_h]\geq \EE[A\mid U_h=u_h],
\end{align}
which implies $U_h$ is rank-calibrated.
\end{proof}

Theorem~\ref{thm:equiv} implies
that, for a binary correctness function, 
one can construct a calibrated confidence measure from an uncertainty measure with monotone transformations if and only if the uncertainty measure is rank-calibrated.
However, \ref{eqn:rank-ece} and ECE
gauge different quantities:
ECE captures the absolute difference between the predicted and true probabilities, 
while \ref{eqn:rank-ece} 
reflects the deviation from a monotonic correspondence between uncertainty and the expected correctness.
These two notions are generally not directly comparable. 

For example, consider the special case where a continuous-valued confidence measure $C$ is completely uninformative and the regressed correctness $\overline \reg\hspace{-2pt} :c\mapsto \EE[A\hspace{-2pt} \mid\hspace{-2pt} C=c]$ is a constant for all confidence levels $c$. 
Then, the RCE defined in \eqref{eqn:rank-ece-c} reports a large value of $1/2$, reflecting its poor indicativeness. However, the ECE can be large or small depending on the averaged distance between $C$'s output and $\overline \reg$. More generally, we find no relation in the results of ECE and \ref{eqn:rank-ece} through the following result, 
proved in Appendix~\ref{app:prop1_prf}.

\begin{proposition}\label{prop:erce-vs-ece}
Let the correctness function $A\hspace{-2pt}\in\hspace{-2pt}\{0,1\}$ be binary. For any $\alpha,\beta\in (0,1/2]$, there is a confidence measure $C$ such that its \ref{eqn:rank-ece} is $\alpha$ while the ECE is $\beta$. 
\end{proposition} 

\vspace{-1mm}
\subsection{Empirical RCE \& Indication Diagram}\label{sec:rce+diagram}
Now, as in Sec.~\ref{uqlm},
consider a dataset $\{(u_i, a_i)\}_{i=1}^n$ of uncertainty and correctness values computed over a benchmark dataset where  each
$u_i\hspace{-1pt}=\hspace{-1pt}U(\bx;\widehat \by_i)$, $a_i\hspace{-1pt}=\hspace{-1pt}A(\bx_i ;\widehat \by_i)$, and $\widehat \by_i$ is a response generated by the LM.
The true value of \ref{eqn:rank-ece} is unknown, as it refers to an average over the distribution from which the data are drawn.

\paragraph{Empirical RCE.} 
The  \ref{eqn:rank-ece} 
involves the unknown
probabilities 
$\PP(U\leq u)$ and $\PP(\reg(U)\geq \reg(u))$, which generally need to be estimated. 
Estimating the latter is challenging as the regression function is also unknown and needs to be estimated. 

To address this, we adopt a piecewise constant regression or binning strategy, 
as in non-parametric statistics~\cite{tsybakov2009}.
First, we group the uncertainty values $\{u_i\}_{i=1}^n$ into $B$
equal-mass intervals, each containing $\lceil n/B\rceil$---or, when needed, $\lfloor n/B\rfloor$---elements. 
The boundaries 
of the $b$-th ($1\leq b\leq B$) bin 
are the $(b-1)/B$-th and $b/B$-th quantiles of $(u_i)_{i=1}^n$. 
Let $\cI_b\subseteq \{1,\dots,n\}$ be the set
of indices of the datapoints whose uncertainty values fall into the $b$-th bin. 
The expected correctness level over the $b$-th bin can be estimated as 
\vspace{-2mm}
\begin{equation*}
    {\rm crc}_b:=\frac{1}{|\cI_b|}\sum_{i\in \cI_b}a_i,
\end{equation*}
when $|\cI_b|>0$. From now on, we will interpret $0/0:=0$; and we extend to $|\cI_b|=0$ in this way.
Clearly, 
${\rm crc}_b$ is an unbiased estimator of $\EE[A\hspace{-1pt}\mid \hspace{-1pt}U\hspace{-2pt}\in\hspace{-2pt} \text{ the $i$-th bin}]$, which approximates $\reg(U)$ accurately given a narrow bin and abundant data. 
We similarly estimate the average uncertainty within the $b$-th bin as
\begin{equation*}
    {\rm uct}_b=\frac{1}{|\cI_b|}\sum_{i\in \cI_b}u_i.
\end{equation*}

 As ${\rm crc}_b$ and ${\rm uct}_b$ estimate the per-bin averages of $\reg(U)$ and $U$, 
for each $b$,
we  estimate $\PP(U\leq u_i)$ and $\PP(\reg(U)\geq \reg(u_i))$ for $i\in \cI_b$ as follows:
\vspace{-1.5mm}
\begin{equation*}
\vspace{-1.5mm}
\begin{aligned}    &\widehat\PP(\reg(U)\hspace{-1pt}\geq\hspace{-1pt} \reg(u_i))\hspace{-2pt}:=\hspace{-2pt}\frac{1}{B-1}\sum_{b^\prime \neq b}\one[{\rm crc}_{b^\prime}\hspace{-1pt}\geq\hspace{-1pt} {\rm crc}_{b}],\vspace{-1mm}\\
    &\widehat \PP(U\hspace{-1pt}\leq\hspace{-1pt} u_i)\hspace{-2pt}:=\hspace{-2pt}\frac{1}{B-1}\sum_{b^\prime \neq b}\one[{\rm uct}_{b^\prime }\hspace{-1pt}\leq\hspace{-1pt} {\rm uct}_{b}].
    \end{aligned}
    \vspace{-1.5mm}
\end{equation*}

A rank-calibrated measure has $\widehat \PP(U\hspace{-1pt}\leq\hspace{-1pt} u_i)\approx\widehat\PP(\reg(U)\hspace{-1pt}\geq\hspace{-1pt} \reg(u_i)) $ for all $1\leq i\leq n$.  We thus compute the empirical Rank-Calibration Error esti-\\
\noindent mator \eqref{eqn:empirical-erce} by taking an average of the per-bin rank differences of correctness and uncertainty values. 
More precisely,
\begin{equation}\label{eqn:empirical-erce}
    \frac{1}{n}\sum_{i=1}^n\left|\widehat\PP(\reg(U)\hspace{-1pt}\geq\hspace{-1pt} \reg(u_i))\hspace{-2pt}-\hspace{-2pt}\widehat\PP(U\hspace{-1pt}\leq\hspace{-1pt} u_i)\right|.\tag{Empirical RCE}
\end{equation}
The difference between
the estimated probabilities for a given bin represent the ranking gap
(\ie, \textcolor{blue}{blue} and \textcolor{shallow red}{shallow red} areas in Fig.~\ref{fig:chat_trivia}). 
We use the
\ref{eqn:empirical-erce} as the main metric to assess uncertainty and confidence measures in the paper.

\paragraph{Indication diagram.}
Similar to reliability diagrams representing miscalibration~\citep{lichtenstein1977calibration,niculescu2005predicting}, we can also visualize
rank-miscalibration in diagrams (\eg, Fig.~\ref{fig:chat_trivia}).
In particular, we plot the relative percentile (between $0\%$ and $100\%$) of
the expected correctness level (\ie, $\reg(U)$) as a function of the relative percentile of uncertainty (\ie, $U$). 
We term these plots \emph{indication diagrams}.
If a measure is rank-calibrated---\ie, if \eqref{eqn:eqiv-dist} holds---then
the indication diagram should lie on the anti-diagonal line ${\rm percent}(\reg(u))\hspace{-1pt}=\hspace{-1pt}1\hspace{-2pt}-\hspace{-2pt}{\rm percent}(u)$. 
Deviations from this line represent rank-miscalibration.

\paragraph{Advantages of rank-calibration.}
We summarize the advantages of the rank-calibration framework by revising the desiderata from Sec.~\ref{sec:desderata}.
First, the empirical RCE does not require any thresholding of the correctness values. 
Second, rank-calibration assesses the monotonicity of uncertainty values by leveraging relative ranks, 
which makes it independent of the output range. 
Third, similar to ECE, the \ref{eqn:rank-ece} is not directly tied to the generation performance of the LM. 
Finally, our assessment is practical for any uncertainty/confidence measures.

\begin{table*}[!t]
\vspace{-10mm}
    \centering
    \resizebox{1.8\columnwidth}{!}{%
\begin{tabular}{lllllllll} \toprule
 Dataset & Correctness & Temperature  & $U_{\rm Ecc}$ & $U_{\rm Deg}$ & $U_{\rm EigV}$ & $U_{\rm NLL}$ & $U_{\rm SE}$ & $C_{\rm Verb}$ \\ \midrule
 \multirow[c]{8}{*}{nq-open} & \multirow[c]{2}{*}{bert} & 0.6 & 0.199$_{\pm 0.040}$ & \bfseries 0.046$_{\pm 0.008}$ & 0.052$_{\pm 0.010}$ & 0.101$_{\pm 0.015}$ & 0.062$_{\pm 0.010}$ & nan \\ \cline{3-9}
  &  & 1.0 & 0.236$_{\pm 0.033}$ & \bfseries 0.035$_{\pm 0.008}$ & 0.038$_{\pm 0.007}$ & 0.097$_{\pm 0.017}$ & 0.055$_{\pm 0.012}$ & nan \\ \cline{2-9}
  & \multirow[c]{2}{*}{meteor} & 0.6 & 0.190$_{\pm 0.039}$ & \bfseries 0.062$_{\pm 0.008}$ & 0.067$_{\pm 0.010}$ & 0.176$_{\pm 0.018}$ & 0.072$_{\pm 0.009}$ & nan \\ \cline{3-9}
  &  & 1.0 & 0.224$_{\pm 0.034}$ & \bfseries 0.044$_{\pm 0.006}$ & 0.046$_{\pm 0.007}$ & 0.209$_{\pm 0.023}$ & 0.074$_{\pm 0.015}$ & nan \\ \cline{2-9}
  & \multirow[c]{2}{*}{rougeL} & 0.6 & 0.198$_{\pm 0.039}$ & \bfseries 0.053$_{\pm 0.011}$ & 0.057$_{\pm 0.010}$ & 0.167$_{\pm 0.013}$ & 0.060$_{\pm 0.012}$ & nan \\ \cline{3-9}
  &  & 1.0 & 0.227$_{\pm 0.035}$ & 0.035$_{\pm 0.007}$ & \bfseries 0.033$_{\pm 0.006}$ & 0.211$_{\pm 0.021}$ & 0.069$_{\pm 0.016}$ & nan \\ \cline{2-9}
  & \multirow[c]{2}{*}{rouge1} & 0.6 & 0.199$_{\pm 0.039}$ & \bfseries 0.054$_{\pm 0.010}$ & 0.057$_{\pm 0.010}$ & 0.167$_{\pm 0.014}$ & 0.061$_{\pm 0.013}$ & nan \\ \cline{3-9}
  &  & 1.0 & 0.227$_{\pm 0.035}$ & 0.034$_{\pm 0.007}$ & \bfseries 0.033$_{\pm 0.006}$ & 0.212$_{\pm 0.021}$ & 0.069$_{\pm 0.015}$ & nan \\ \cline{1-9}
 \multirow[c]{8}{*}{squad} & \multirow[c]{2}{*}{bert} & 0.6 & 0.208$_{\pm 0.033}$ & 0.065$_{\pm 0.014}$ & 0.075$_{\pm 0.017}$ & \bfseries 0.048$_{\pm 0.007}$ & 0.063$_{\pm 0.012}$ & nan \\ \cline{3-9}
  &  & 1.0 & 0.276$_{\pm 0.039}$ & 0.067$_{\pm 0.011}$ & 0.063$_{\pm 0.010}$ & \bfseries 0.038$_{\pm 0.006}$ & 0.098$_{\pm 0.012}$ & nan \\ \cline{2-9}
  & \multirow[c]{2}{*}{meteor} & 0.6 & 0.216$_{\pm 0.038}$ & 0.303$_{\pm 0.026}$ & 0.265$_{\pm 0.022}$ & \bfseries 0.063$_{\pm 0.013}$ & 0.182$_{\pm 0.029}$ & nan \\ \cline{3-9}
  &  & 1.0 & 0.300$_{\pm 0.046}$ & 0.292$_{\pm 0.035}$ & 0.250$_{\pm 0.027}$ & \bfseries 0.064$_{\pm 0.011}$ & 0.274$_{\pm 0.021}$ & nan \\ \cline{2-9}
  & \multirow[c]{2}{*}{rougeL} & 0.6 & 0.239$_{\pm 0.036}$ & 0.177$_{\pm 0.026}$ & 0.143$_{\pm 0.020}$ & \bfseries 0.052$_{\pm 0.011}$ & 0.127$_{\pm 0.020}$ & nan \\ \cline{3-9}
  &  & 1.0 & 0.304$_{\pm 0.036}$ & 0.179$_{\pm 0.033}$ & 0.137$_{\pm 0.024}$ & \bfseries 0.053$_{\pm 0.012}$ & 0.210$_{\pm 0.027}$ & nan \\ \cline{2-9}
  & \multirow[c]{2}{*}{rouge1} & 0.6 & 0.238$_{\pm 0.037}$ & 0.183$_{\pm 0.027}$ & 0.148$_{\pm 0.022}$ & \bfseries 0.053$_{\pm 0.010}$ & 0.129$_{\pm 0.021}$ & nan \\ \cline{3-9}
  &  & 1.0 & 0.303$_{\pm 0.035}$ & 0.185$_{\pm 0.033}$ & 0.143$_{\pm 0.025}$ & \bfseries 0.053$_{\pm 0.012}$ & 0.213$_{\pm 0.026}$ & nan \\ \cline{1-9}
 \multirow[c]{8}{*}{triviaqa} & \multirow[c]{2}{*}{bert} & 0.6 & 0.140$_{\pm 0.024}$ & 0.062$_{\pm 0.016}$ & 0.061$_{\pm 0.015}$ & \bfseries 0.020$_{\pm 0.004}$ & 0.027$_{\pm 0.007}$ & nan \\ \cline{3-9}
  &  & 1.0 & 0.213$_{\pm 0.030}$ & 0.025$_{\pm 0.006}$ & 0.034$_{\pm 0.006}$ & \bfseries 0.014$_{\pm 0.002}$ & 0.036$_{\pm 0.006}$ & nan \\ \cline{2-9}
  & \multirow[c]{2}{*}{meteor} & 0.6 & 0.145$_{\pm 0.027}$ & 0.067$_{\pm 0.017}$ & 0.064$_{\pm 0.015}$ & \bfseries 0.034$_{\pm 0.009}$ & 0.075$_{\pm 0.016}$ & nan \\ \cline{3-9}
  &  & 1.0 & 0.206$_{\pm 0.032}$ & \bfseries 0.035$_{\pm 0.007}$ & 0.046$_{\pm 0.005}$ & 0.049$_{\pm 0.008}$ & 0.084$_{\pm 0.007}$ & nan \\ \cline{2-9}
  & \multirow[c]{2}{*}{rougeL} & 0.6 & 0.141$_{\pm 0.021}$ & 0.062$_{\pm 0.014}$ & 0.061$_{\pm 0.014}$ & \bfseries 0.024$_{\pm 0.005}$ & 0.034$_{\pm 0.005}$ & nan \\ \cline{3-9}
  &  & 1.0 & 0.204$_{\pm 0.035}$ & 0.027$_{\pm 0.006}$ & 0.040$_{\pm 0.004}$ & \bfseries 0.022$_{\pm 0.002}$ & 0.051$_{\pm 0.007}$ & nan \\ \cline{2-9}
  & \multirow[c]{2}{*}{rouge1} & 0.6 & 0.141$_{\pm 0.021}$ & 0.062$_{\pm 0.014}$ & 0.062$_{\pm 0.013}$ & \bfseries 0.024$_{\pm 0.005}$ & 0.034$_{\pm 0.006}$ & nan \\ \cline{3-9}
  &  & 1.0 & 0.203$_{\pm 0.035}$ & 0.027$_{\pm 0.006}$ & 0.040$_{\pm 0.004}$ & \bfseries 0.022$_{\pm 0.002}$ & 0.051$_{\pm 0.007}$ & nan \\ \bottomrule
\end{tabular}
}
    \caption{RCE results for Llama-2-chat with various experimental configurations.}
    \label{tab:llama-2-chat}
    \vspace{-4mm}
\end{table*}

\vspace{-1mm}
\section{Experiments}
\vspace{-1mm}
We provide more comprehensive experiments and justify the advantages of our assessment.

\subsection{Experiment Setup} 
We consider both open-source and commercial LMs, including 
\textit{Llama-2-7b}, \textit{Llama-2-7b-chat~\citep{touvron2023llama2}} (an instruction fine-tuned version of \textit{Llama-2-7b}), and \textit{GPT-3.5-turbo}~\citep{ouyang2022training}. See Appendix~\ref{app:model} for more details.

We conduct assessments on the validation sets of four datasets: 
TriviaQA~\citep{joshi-etal-2017-triviaqa}, Natural Questions~\citep{kwiatkowski-etal-2019-natural}, SQuAD-1~\citep{rajpurkar2016squad}, and Meadow~\citep{wang-etal-2020-cord}. 
For assessment over the open-ended and challenging 
Meadow, we only use the more advanced model GPT-3.5-turbo. 
To account for randomness in the evaluation, 
we repeat experiments bootstrapping each 
dataset 20 times. 
See more details of datasets in Appendix~\ref{app:datasets}. 

We use multiple correctness functions, including the \textit{Rouge-L} score,
\textit{BERT similarity}, and \textit{ChatGPT  evaluation}, all widely applied before \citep{kuhn2023semantic,xiong2023llms}. 
ChatGPT correctness is only used for GPT-3.5-turbo with temperature 1.0. See Appendix~\ref{app:correctness} for more details.

The uncertainty/confidence measures to be assessed are the same as in Sec.~\ref{sec:case-study}, (\ie, $U_{\rm NLL}$, $U_{\rm SE}$, $U_{\rm Ecc}$, $U_{\rm Deg}$, $U_{\rm EigV}$, and $C_{\rm Verb}$). 
We first illustrate that our proposed assessment has broad applicability and granular interpretability.
Furthermore, we qualitatively show that uncertainty measures with lower RCE values reliably indicate correctness. Finally, we study robustness by empirically checking the impact of temperature and correctness functions on RCE~\citep{demvsar2006statistical}. 
 More results for different configurations are given in Table~\ref{tab:all}.

\subsection{Broader Applicability}
Previous assessments have some limitations 
in open-ended tasks. First, as shown in Fig.~\ref{fig:eval_meadow} [top], the correctness distribution in open-ended tasks (\eg, the  Meadow dataset) is less concentrated around zero and one compared to the TriviaQA correctness distribution. Consequently, if correctness were binarized with thresholding, the assessed results would be highly impacted by the threshold choice, as illustrated in Fig.~\ref{fig:eval_meadow} [bottom]. As such, using continuous-valued correctness scores is common in open-ended tasks~\citep{cohan2018discourseaware, Uppalapati2023ACS}. 
Since \ref{eqn:rank-ece} does not require thresholding, our rank-calibration assessment does not suffer from the above issue.

\begin{figure}[ht]
\centering
\hspace{-4mm}
  \includegraphics[width=0.53\linewidth]{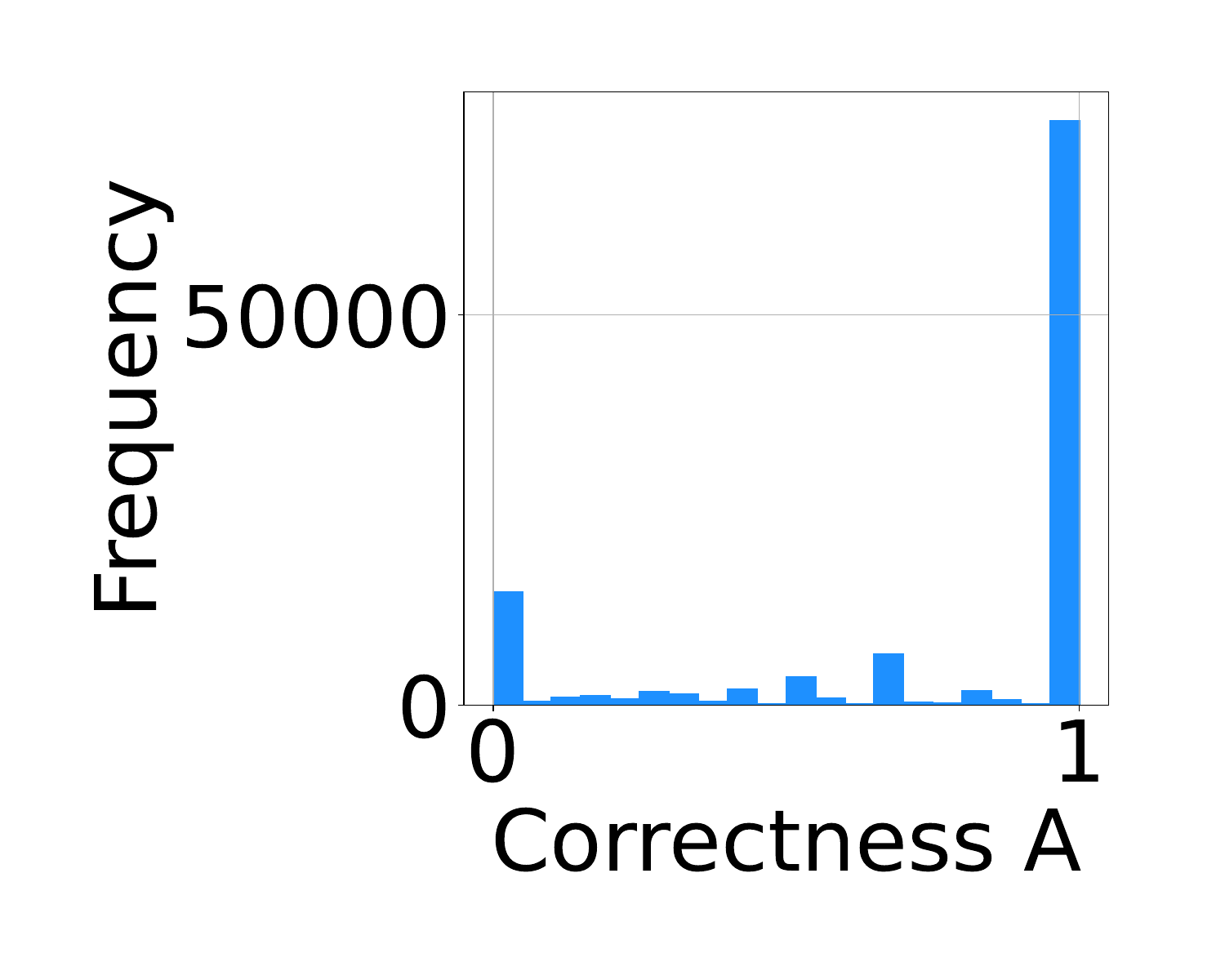}
  \hspace{-3mm}
  \includegraphics[width=0.52\linewidth]{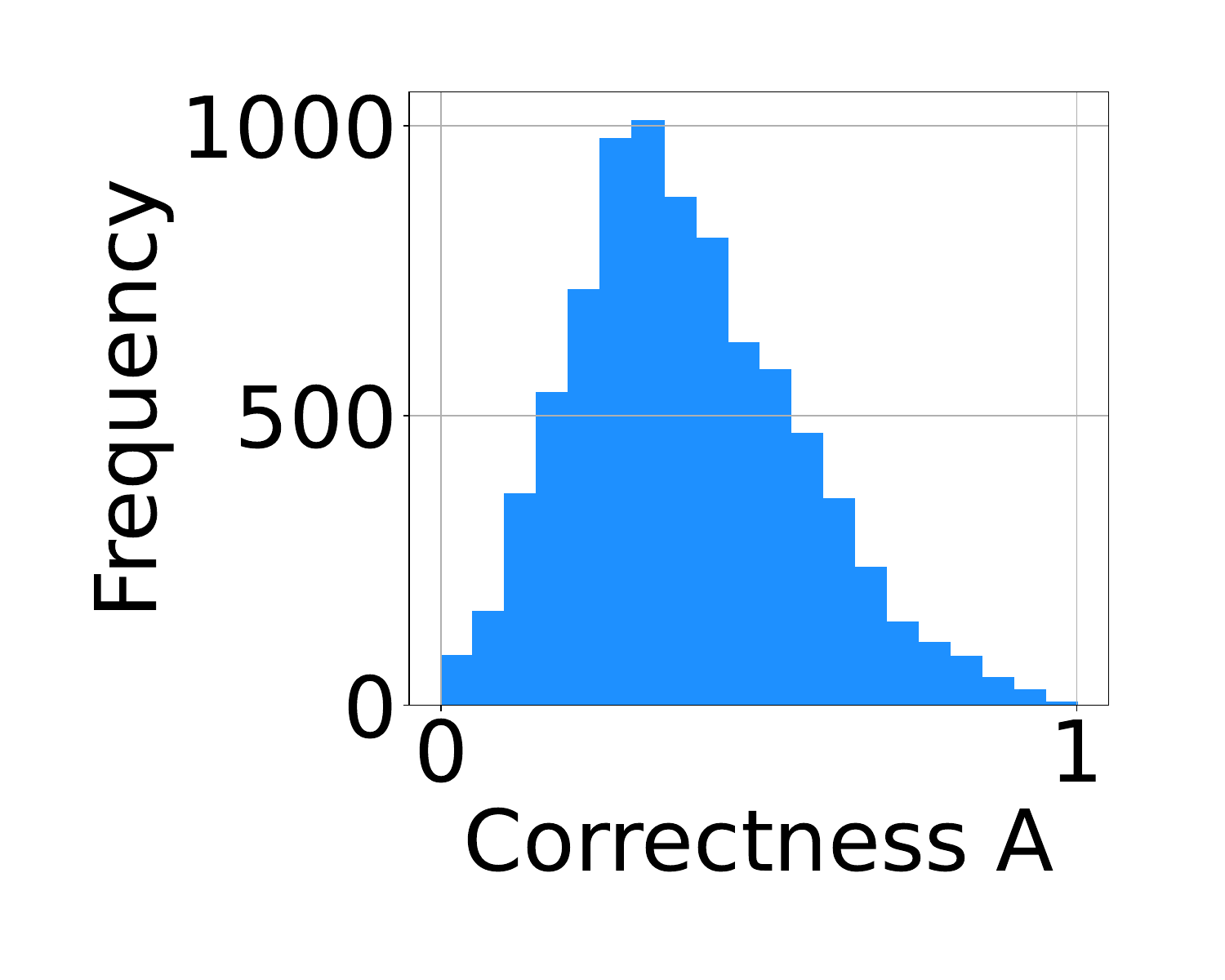}
  \vspace{-5mm}
  \hspace{-4mm}
  \\
  \includegraphics[width=0.8\linewidth]{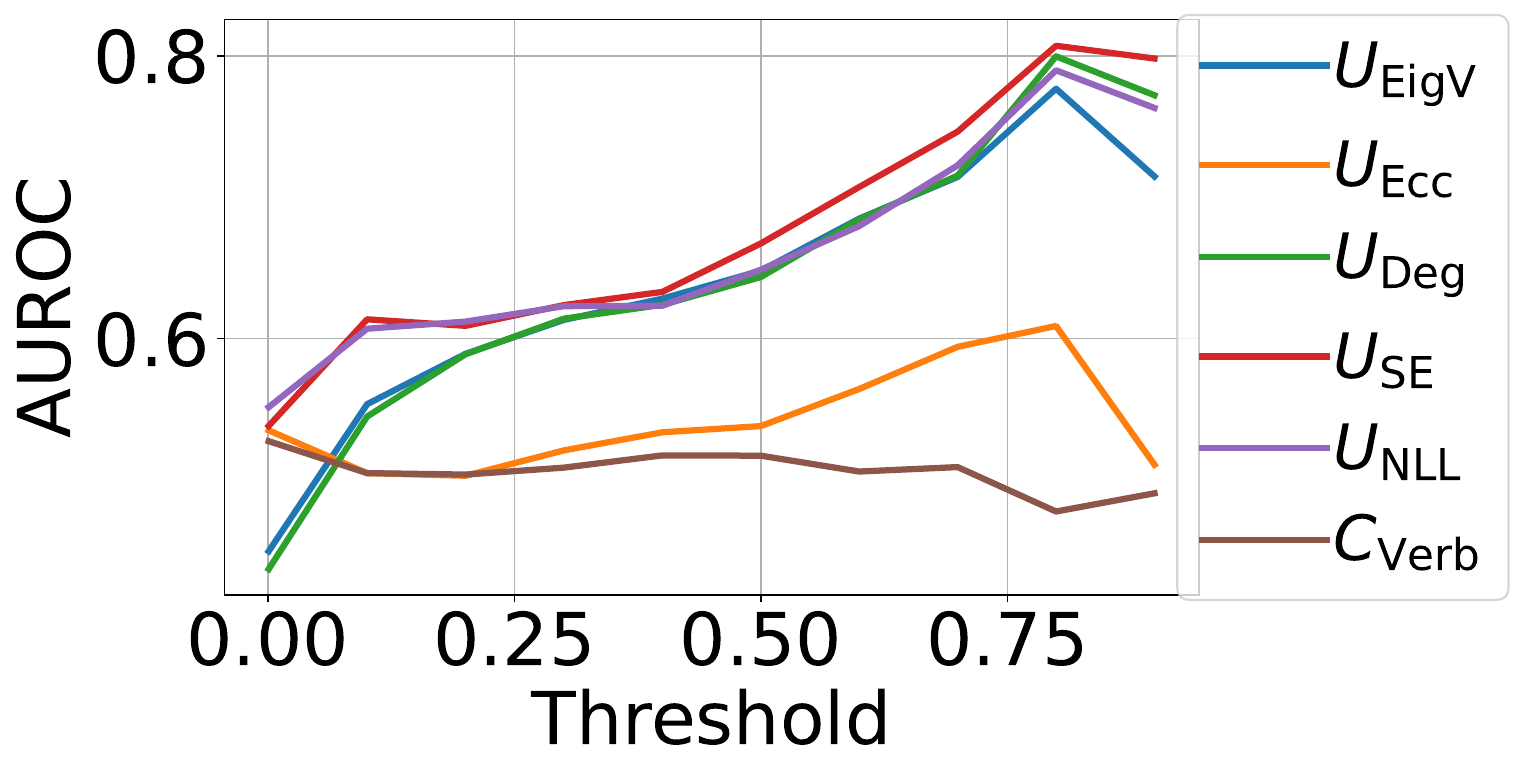}
\caption{Top: Rouge-L correctness distributions of GPT-3.5-turbo on the {\em TriviaQA (left)} and {\em Meadow (right)} benchmarks. Bottom: AUROCs of assessed measures
for GPT-3.5-turbo on {\em Meadow}, with Rouge-L correctness and various thresholds. 
}
\vspace{-5mm}
\label{fig:eval_meadow}
\end{figure}

\subsection{Granular Interpretability} 
Beyond the rank-calibration error, 
the indication diagrams can be instrumental in understanding the performance of uncertainty measures. 
We show the indication diagrams of $U_{\rm NLL}$ and $U_{\rm SE}$  for GPT-3.5-turbo  on TriviaQA in Fig.~\ref{fig:chat_trivia}. 
More indication diagrams can be found in the Appendix.

First, indication diagrams consistently reflect the effect of rank-miscalibration. The indication diagram of $U_{\rm NLL}$ (Fig.~\ref{fig:chat_trivia} [left]) has more overlap between the red and blue bars, compared to that of $U_{\rm Ecc}$ (Fig.~\ref{fig:chat_trivia} [right]), reflecting a lower RCE level (0.038 with $U_{\rm NLL}$ v.s. 0.151 with $U_{\rm Ecc}$). 
The high overlap suggests that the relative ranks of uncertainty values are more aligned with those of correctness levels, leading to better rank-calibration.

Second, indication diagrams can shed light onto which uncertainty levels may be problematic. 
For example, in Fig.~\ref{fig:chat_trivia} [right], 
we observe that
for an uncertainty in the top 75th percentile, 
$U_{\rm Ecc}$ tends to be overpessimistic: 
$U_{\rm Ecc}$ assigns high uncertainty values to high-quality generations.

\subsection{Qualitative Illustration}


To illustrate the effectiveness of the RCE as an evaluation metric for uncertainty measures, we present two TriviaQA instances and contrast $U_{\rm NLL}$
(having \ref{eqn:rank-ece} 0.037) with $U_{\rm SE}$ (having \ref{eqn:rank-ece} 0.051) for GPT-3.5. 
Here, $\bx$ is the question input, $\by$ is the answer in the dataset, 
$\widehat \by$ is the LM response, and $\PP(U \hspace{-1pt}\leq \hspace{-1pt} u)$ signifies the relative magnitudes of LM's uncertainty level according to $U_{\rm NLL}$ and $U_{\rm SE}$.

\noindent \tcbset{colback=blue!5!white,colframe=blue!75!black,fonttitle=\bfseries,width=0.45\textwidth,nobeforeafter}
\begin{tcolorbox}[box align=center]
{\fontfamily{qcr}\selectfont 
    \noindent $\bx$: On September 28th, NASA announced that what had been detected on Mars? \\
    $\by$: flowing water \\
    $\widehat\by$: Possible signs of life \\
    $\PP(U_{\rm SE} \leq u)$: 0.813\\
    $\PP(U_{\rm NLL} \leq u)$: 0.930
}
\end{tcolorbox}

\noindent
\tcbset{colback=red!5!white,colframe=red!75!black,fonttitle=\bfseries,width = 0.45\textwidth}
\begin{tcolorbox}[box align=center]
{\fontfamily{qcr}\selectfont 
    \noindent $\bx$: ``Feel Like Making Love'' and ``The First Time Ever I Saw Your Face'' were hit singles for which female artist?  \\
    $\by$: roberta flack\\
    $\widehat \by$: Roberta Flack \\
    $\PP(U_{\rm SE} \leq u)$: 0.864\\
    $\PP(U_{\rm NLL} \leq u)$: 0.046
}
\end{tcolorbox}

\vspace{1mm}
In the first instance, the generation is factually incorrect and
$U_{\rm NLL}$ assigns a high uncertainty value to the response, \ie $\PP(U_{\rm NLL} \hspace{-1pt}\leq \hspace{-1pt}u)\approx 1$. 
In the second scenario, where the generation is correct, $U_{\rm NLL}$ succeeds in providing a lower uncertainty level, \ie $\PP(U_{\rm NLL} \hspace{-1pt}\leq\hspace{-1pt} u)\hspace{-1pt}\approx\hspace{-1pt} 0$. 
Yet, $U_{\rm SE}$ 
assigns \emph{a lower uncertainty to a poorer generation and a higher uncertainty to a better generation}! 
These instances showcase 
that $U_{\rm NLL}$ is more reliable than $U_{\rm SE}$ here, which is consistent with the RCE-assessed results. 
Additional qualitative results are given in Table~\ref{tab:all_qual}.



\subsection{Post-hoc Recalibration}
Recalibrating 
uncertainty/confidence measures with poor rank-calibration can be of interest; for ECE, this is sometimes known as Mincer-Zamowitz regression \cite{mincer1969evaluation}. 
As discussed in Sec.~\ref{sec:comparison}, an ECE-calibrated 
 measure is also \ref{eqn:rank-ece}-calibrated. However, \ref{eqn:rank-ece} is invariant to monotone transformations, 
 which means that approaches like
 Platt scaling~\citep{platt1999probabilistic} and isotonic regression~\citep{zadrozny2002transforming} will not improve rank-calibration.
 Therefore, we suggest using histogram binning (or, piecewise constant regression), which includes non-monotone transforms~\citep{zadrozny2001obtaining}. 
 Table~\ref{tab:all-SE-GPT-calibrated} and Fig.~\ref{fig:gpt_1.0_cal_meteor} and~\ref{fig:gpt_1.5_cal_triviaqa} list the RCE results of $U_{\rm SE}$ for GPT-3.5-turbo before and after calibration. 
We observe the calibrated measure is significantly better rank-calibrated, showing the effectiveness of this strategy. 
 See the more experimental details and results in Appendix~\ref{app:recalibrastion}.

\subsection{Robustness Analysis}
We conduct ablation studies to analyze the robustness of our assessment to key hyperparameters, including temperatures, correctness scores, and sample sizes. 
We further propose a method to make robust comparisons between uncertainty measures via the \textit{Critical Difference (CD) Diagram} ~\citep{demvsar2006statistical}. 
Detailed information and results are in Appendix~\ref{sec:robustness}.

\begin{table}
    \centering
    \resizebox{0.9\columnwidth}{!}{%
    \small
\begin{tabular}{lllll} \toprule Dataset & Correctness & Temperature  & $U_{\rm SE}$ & $U_{\rm SE, cal}$  \\ \midrule 
\multirow[c]{4}{*}{meadow} & bert & 1.0 &  0.177$_{\pm 0.027}$ & \bfseries 0.083$_{\pm 0.016}$ \\  \cline{2-5} 
& meteor & 1.0 & 0.132$_{\pm 0.018}$ & \bfseries 0.066$_{\pm 0.015}$ \\ \cline{2-5}
& rougeL & 1.0 & 0.113$_{\pm 0.022}$ & \bfseries 0.063$_{\pm 0.014}$ \\ \cline{2-5}
& rouge1 & 1.0 & 0.113$_{\pm 0.018}$ & \bfseries 0.061$_{\pm 0.012}$ \\ \cline{1-5}
\multirow[c]{4}{*}{nq-open} & bert & 1.0 & 0.050$_{\pm 0.007}$ & \bfseries 0.026$_{\pm 0.007}$ \\ \cline{2-5}
& meteor & 1.0 & 0.060$_{\pm 0.009}$ & \bfseries 0.033$_{\pm 0.011}$ \\ \cline{2-5}
& rougeL & 1.0 & 0.052$_{\pm 0.008}$ & \bfseries 0.030$_{\pm 0.010}$ \\ \cline{2-5}
& rouge1 & 1.0 & 0.051$_{\pm 0.008}$ & \bfseries 0.029$_{\pm 0.010}$ \\ \cline{1-5}
\multirow[c]{4}{*}{squad} & bert & 1.0 &   0.113$_{\pm 0.013}$ & \bfseries 0.050$_{\pm 0.013}$ \\ \cline{2-5}
& meteor & 1.0 & 0.086$_{\pm 0.014}$ & \bfseries 0.046$_{\pm 0.010}$ \\ \cline{2-5}
& rougeL & 1.0 & 0.100$_{\pm 0.011}$ & \bfseries 0.037$_{\pm 0.008}$ \\ \cline{2-5}
& rouge1 & 1.0 & 0.103$_{\pm 0.011}$ & \bfseries 0.039$_{\pm 0.007}$ \\ \cline{1-5}
\multirow[c]{12}{*}{triviaqa} & \multirow[c]{3}{*}{bert} & 0.5 & 0.052$_{\pm 0.009}$ & \bfseries 0.030$_{\pm 0.010}$ \\ \cline{3-5}
&  & 1.0 & 0.052$_{\pm 0.012}$ & \bfseries 0.027$_{\pm 0.008}$ \\ \cline{3-5}
&  & 1.5 & 0.081$_{\pm 0.009}$ & \bfseries 0.029$_{\pm 0.007}$ \\ \cline{2-5}
& \multirow[c]{3}{*}{meteor} & 0.5 & 0.234$_{\pm 0.019}$ & \bfseries 0.058$_{\pm 0.015}$ \\ \cline{3-5}
&  & 1.0 &  0.209$_{\pm 0.012}$ & \bfseries 0.047$_{\pm 0.014}$ \\ \cline{3-5}
&  & 1.5 & 0.176$_{\pm 0.015}$ & \bfseries \bfseries 0.036$_{\pm 0.012}$ \\ \cline{2-5}
& \multirow[c]{3}{*}{rougeL} & 0.5 & 0.050$_{\pm 0.008}$ & \bfseries 0.028$_{\pm 0.007}$ \\ \cline{3-5}
&  & 1.0 & 0.059$_{\pm 0.009}$ & \bfseries 0.026$_{\pm 0.007}$ \\ \cline{3-5}
&  & 1.5 & 0.104$_{\pm 0.007}$ & \bfseries 0.028$_{\pm 0.006}$ \\ \cline{2-5}
& \multirow[c]{3}{*}{rouge1} & 0.5 & 0.050$_{\pm 0.008}$ & \bfseries 0.028$_{\pm 0.006}$ \\ \cline{3-5}
&  & 1.0 & 0.060$_{\pm 0.009}$ & \bfseries 0.027$_{\pm 0.006}$ \\ \cline{3-5}
&  & 1.5 & 0.105$_{\pm 0.008}$ & \bfseries 0.028$_{\pm 0.008}$ \\  \bottomrule
\end{tabular}
}
    \caption{RCE results of $U_{\rm SE}$ and $U_{\rm SE, cal}$ after rank-calibration for GPT-3.5-turbo with various experimental configurations.}
    \label{tab:all-SE-GPT-calibrated}
    \vspace{-4mm}
\end{table}

\section{Conclusion}
This paper investigates the limitations of common assessments for LM uncertainty/confidence measures. 
We develop an alternate
framework, termed rank-calibration, to assess their quality. Our approach does not require
binarizing correctness at ad hoc thresholds
and is compatible with uncertainty measures taking values in any output range. 
We experimentally
show the broad applicability and the granular interpretability of our method, and provide a comprehensive robustness analysis. Future directions include developing uncertainty measures with guaranteed rank-calibration and enhancing generative pipelines of LMs (\eg, the retrieval-augmented generation) with rank-calibrated measures.


\section*{Limitation \& Broader Impact}
The empirical RCE  estimate has not been subjected to a thorough statistical analysis. 
The performance of assessed uncertainty and confidence measures (\eg, the vanilla verbalized confidence $C_{\rm Verb}$) have not been optimized, since the paper focuses on a new assessment approach rather than benchmarking. Human correctness evaluation is not performed, due to our limited budget.

This work is designed to unveil the issues in the existing approaches for evaluating LM uncertainty/confidence measures, and to introduce an alternate, principled assessment to the LM community. We believe there are no ethical concerns associated with our research.

\bibliography{cal_references}

\newpage
\appendix

\onecolumn
\section{Additional Related Work}\label{app:related-works}
\paragraph{Uncertainty measures in supervised learning.}
The quantification of uncertainties in model outputs in supervised learning has a long history~\citep[\eg,][etc]{lichtenstein1977calibration}. 
Overparametrized models such as neural networks pose unique challenges to estimate uncertainty and improve model calibration\citep{guo2017calibration,papadopoulos2001confidence,riquelme2018deep}. 
Various approaches have been introduced to mimic Bayesian inference~\citep{gal2016dropout}, to utilize simple deep ensembles~\citep{lakshminarayanan2017simple,jain2020maximizing}, and to identify training samples that are out-of-distribution~\citep{liang2018enhancing,papernot2018deep}.
Nonetheless, 
it is not clear how to adapt these strategies to language modeling, where the output can be text with complex structure.

\paragraph{Uncertainty measures in language modeling.}
To gauge the uncertainty level associated with the outputs of LMs, \citet{kuhn2023semantic} introduces the concept of semantic entropy, which integrates linguistic consistencies stemming from shared meanings. In a similar vein, \citet{kadavath2022language,lin2022teaching,xiong2023llms} encourage LMs to analyze their own responses and come up with a ``probability'' that a response is correct. 
In related work,~\citet{manakul2023selfcheckgpt} 
uses sampling to identify instances of fabricated information. Recently, \citet{tian-etal-2023-just} explore methods for deriving confidence measures for reinforcement-learning-trained LMs. \citet{lin2023generating} draw a distinction between estimating uncertainty and confidence for LMs. Similarly, \citet{chen2023quantifying} introduce a method for detecting bad and speculative responses
from a pre-trained LM with a confidence score. \citet{tanneru2023quantifying} propose two novel measures to quantify the uncertainty of LM-generated explanations.
Although considerable research focuses on developing uncertainty and confidence measures for LMs, the evaluation of their effectiveness is less studied. 

\paragraph{Assessments of uncertainty measures.}
Early 
assessment of confidence measures in classification scenarios leveraged proper scoring rules~\citep{savage1971elicitation,degroot1983comparison,gneiting2007strictly}, such as the Brier score~\citep{brier1950verification} and the KL divergence~\citep{winkler1996scoring}.
Other assessments include plotting calibration curves, also known as reliability diagrams (estimated
probabilities against predicted ones)~\citep{frank2015regression}. More recently, the ECE metric---or mean absolute calibration error---has gained popularity in machine learning~\citep{frank2015regression,naeini2015obtaining}, along with many variants~\citep{kumar2019verified,nixon2019measuring,gupta2021calibration,lee2023t,si2022re}. 

In the realm of uncertainty quantification for LMs, the assessment based on ECE remains viable. 
However, it necessitates the introduction of ad hoc threshold to derive binary labels. Moreover, 
the applicability of ECE is limited,
as it does not directly apply 
to LM uncertainty measures that fall outside the 
interval $[0,1]$. 
Our work introduces an assessment centered around rank-calibration, a critical property that ideal uncertainty measures should satisfy. 
This assessment is applicable to both confidence and uncertainty measures and eliminates the need for thresholding the correctness values.

\section{Common Uncertainty/Confidence Measures for  LMs}\label{app:measures}
In this section, we introduce common measures of uncertainty  and confidence in detail.
\begin{itemize}
[leftmargin=0.2in]
    \item {\bf NLL \& Perplexity. } Let $\widehat \by= (\widehat y_\ell)_{\ell \geq 1}$ be the generated response.  
    Then the Negative Log-Likelihood (NLL) is 
    \begin{equation}
        U_{\rm NLL}(\bx, \widehat \by):=- \ln(\PP(\widehat \by\mid \bx))=-\sum_{\ell\geq 1}\ln(\PP(\widehat y_\ell \mid \bx, \widehat y_{<\ell})).
    \end{equation}
    A natural extension accounts for the variable length of responses by applying length normalization. Suppose that the number of tokens of the response $\widehat \by$ is ${\rm len}(\widehat \by)$, the length-normalized NLL is defined as 
    \begin{equation}
        U_{\rm NLL{\text -}LN}(\bx, \widehat \by):=-\frac{1}{{\rm len}(\widehat \by)}\sum_{\ell=1}^{{\rm len}(\widehat \by)}\ln(\PP(\widehat y_\ell \mid \bx, \widehat y_{<\ell})).
    \end{equation}
    Roughly speaking, this can be viewed as the average nats per token in the generated text; if using $\log_2$ instead of $\ln$, it would be the average bits per token.
    The exponential of the length-normalized NLL
    is known as the Perplexity:
    $U_{{\rm Perp}}(\bx;\widehat \by):=\exp(U_{\rm NLL{\text -}LN}(\bx, \widehat \by))$~\citep{jelinek1977perplexity}.
The perplexity can also be viewed as the inverse of the geometric mean of the token-wise probabilities.

\item {\bf Entropy. } Entropy is a well-known type of uncertainty measure. 
The predictive entropy of the distribution $\PP(\cdot \mid \bx)$ is defined as
\begin{equation}
    U_{{\rm E}}(\bx):=-\EE_{\widehat \by\sim \PP(\cdot \mid \bx)}[\ln( \PP(\widehat \by \mid \bx))].
\end{equation}
Entropy gauges the information one has about the potential output given the input, and has high values
when outputs are diverse. 
\citet{malinin2021uncertainty} propose a variant $U_{{\rm E{\text -}LN}}(\bx)$ using the length-normalized log-likelihood $\ln( \PP(\widehat \by \mid \bx))/{\rm Length}(\widehat \by)$. 
\citet{kuhn2023semantic} argues that responses with an identical meaning should be viewed as equal; even if they differ at the token level. They thus propose the semantic entropy 
\begin{equation}
    U_{{\rm SE}}(\bx):=-\EE_{\widehat \by\sim \PP(\cdot \mid \bx)}[\ln( \PP(c(\widehat \by) \mid \bx))],
\end{equation}
where $c(\widehat \by)$ is a semantic concept of output $\widehat \by$, as determined by another machine learning method. We can similarly define the length-normalized semantic entropy as
\begin{equation}
    U_{{\rm SE{\text -}LN}}(\bx):=\-\EE_{\widehat \by\sim \PP(\cdot \mid \bx)}[\ln( \PP(c(\widehat \by) \mid \bx))/{\rm len}(\widehat \by)].
\end{equation}

\item {\bf Affinity graph.} Recently, \citet{lin2023generating} use a weighted adjacency graph built upon semantic affinities between 
outputs to reflect uncertainty.
Given an 
\emph{entailment-contradiction} affinity model $e$
that maps pairs 
$\widehat \by, \widehat \by ^\prime $ 
of responses 
to values in $[0,1]$, 
$e$ induces a symmetric adjacency matrix $W\hspace{-2pt}=\hspace{-2pt}[w_{i,j}]_{i,j= 1}^K$ 
with responses $\{\widehat \by^{(k)}\}_{k=1}^K$ sampled 
 from $\PP(\cdot \hspace{-2pt}\mid \hspace{-2pt}\bx)$,
 where for all $i,j$,
 $w_{i,j}\hspace{-2pt}=\hspace{-2pt}(e(\widehat \by^{(i)};\widehat \by ^{(j)})+e(\widehat \by ^{(j)};\widehat \by^{(i)}))/2$. 
Let $D=[\mathds{1}[j=i]\sum_{k=1}^K w_{k,j}]_{i,j=1}^K$ 
be the matrix of degrees
and  $\{\lambda_k\}_{k=1}^K$ be the eigenvalues 
of the \emph{Laplacian} $L\hspace{-2pt}=\hspace{-2pt}I\hspace{-2pt}-\hspace{-2pt}D^{-1/2}WD^{-1/2}$. 
Measures proposed in \citet{lin2023generating}
include
\begin{align}
    &U_{\rm EigV}(\bx):=\sum_{k= 1}^K\max\{0, 1-\lambda
    _k\},&\\
    &U_{\rm Deg}(\bx):=1-{\rm trace}(D)/K^2,&C_{\rm Deg}(\bx; \widehat \by^{(i)}):=D_{i,i}/K,\\
    &U_{\rm Ecc}(\bx):=\|[\bv_1,\bv_2,\dots, \bv_K]\|_2.
\end{align}
where $\{\bv_k\}_{k= 1}^K$ are certain centralized vectors associated with the spectral decomposition of $L$.
Here, $U_{\rm EigV}(\bx)$ is approximates the number
of connected components in the graph represented by $W$, while $U_{\rm Deg}(\bx)$ and $U_{\rm Ecc}(\bx)$ reflect the diversity of outputs. 

\item {\bf Verbalized confidence.} Verbalized confidence generally refers to the textual confidence output by an LM. 
For example, if an LM is highly uncertain about its answer, it may inform the user by saying, \eg, ``I am only 20\% confident in this answer.'' 
This is often implemented by feeding handcrafted prompts to advanced LMs such as GPT-4~\citep{gpt-4}.
Many prompting strategies have been used in the literature to enhance this procedure~\citep{zhao2021calibrate,kadavath2022language,lin2022teaching,xiong2023llms}. Since optimizing the prompting strategy is not our focus and we do not want confidence elicitation to interfere with the generation of responses, we adopt a simple post-hoc strategy here by feeding a query-response pair to an LM and asking it how confident it believes the response correctly addresses the query. This post-hoc strategy is similar to the one used by \citet{kadavath2022language}. We use the following specific prompt format:

{\fontfamily{qcr}\selectfont
    \noindent Read the question and answer below. \\
    \{question\} \{generation\}\\
    Provide a numeric confidence that indicates your certainty about this answer. \\
    For instance, if your confidence level is 80\%, it means you are 80\% certain that this answer is correct and there is a 20\% chance that it is incorrect. \\
    Use the following format to provide your confidence: Confidence: [Your confidence, a numerical number in the range of 0-100]\%."
}
\end{itemize}

\section{Common Evaluation Metrics}\label{app:metrics}
In this section, we review evaluation metrics that have been commonly used to assess LM uncertainty/confidence measures. These metrics usually require binary correctness values.
\begin{itemize}
[leftmargin=0.2in]
    \item {\bf AUROC. } AUROC refers to the area under the Receiver-Operating Curve (ROC). The ROC plots the true positive rate (a.k.a. recall) against the false positive rate (a.k.a. $1-$ specificity) at various thresholds of uncertainty levels. The true positive rate is on the $y$-axis, and the false positive rate is on the $x$-axis. An AUROC value of $1$ may represent a perfect uncertainty measure; a value of $0.5$ suggests no discriminative ability (equivalent to random uncertainty levels). The AUROC can be more useful for evaluation in imbalanced scenarios where correct responses are much more (or less) frequent than incorrect responses.
    \item {\bf AUPRC. } AUPRC refers to the area under the Precision-Recall Curve (PRC), which plots the positive predictive value (a.k.a. precision) against the true positive rate (a.k.a. recall) at various threshold settings. Precision is on the $y$-axis, and recall is on the $x$-axis.  Similar to AUROC, it is valuable in imbalanced dataset scenarios but focuses more on the performance of the positive (minority) class (\ie, correct responses). Variants of AURPC include AUPRC-Positive and AUPRC-Negative,
    which focus on gauging the ability of uncertainty measures to identify correct responses and
incorrect responses, respectively.
    \item {\bf AUARC. } AUARC refers to the area under the Accuracy-Rejection Curve (ARC) that plots the accuracy of generation against a rejection rate (the proportion of generated responses for which the model abstains from making a prediction). The curve shows how the accuracy of generation improves as it is allowed to reject uncertain responses.
A higher AUARC value means that an LM can generate more correct responses as it increasingly avoids uncertain (based on the level of specific uncertainty measures) cases. This metric is useful for evaluating uncertainty measures in scenarios where LMs can defer responses for which they are not confident.
    \item {\bf ECE.} ECE stands for the expected calibration error, a metric used to evaluate the calibration of confidence measures, particularly in classification tasks. Calibration refers to how well the confidence levels align with the actual proportion of correct generation. For an ideally calibrated confidence measure, if the confidence level is 70\%, then approximately 70\% of generated responses should be correct.
ECE quantifies the difference between the confidence levels and the realized correct proportion.  A lower ECE indicates better calibration, meaning the confidence measure is more reflective of the actual correct proportion.
A confidence measure with an ECE close to zero is considered well-calibrated.
\end{itemize}

\section{Proof of Proposition~\ref{prop:erce-vs-ece}}
\label{app:prop1_prf}
{\bf Case 1. $\alpha=1/2$.} 
Consider the continuous case $C\sim {\rm Unif}[1/2-\beta,1/2+\beta]$ and $\reg(C)\equiv 1/2+\beta$ almost surely (\ie, $A\sim {\rm Bernoulli}(1/2+\beta)$). Then $\PP_{C^\prime}(\overline{\reg}(C^\prime)\geq \overline{\reg}(C))\equiv 1$ for almost surely. Since $C$ is continuous-valued,  $\PP_{C^\prime}$ follows the uniform distribution over $[0,1]$. We thus have 
\begin{equation}
    {\rm RCE}=\int_0^1|1-p|{\rm d}p=\frac{1}{2}.
\end{equation}
On the other hand, 
\begin{equation}
    {\rm ECE}=\int_{1/2-\beta}^{1/2+\beta}\frac{|1/2+\beta-c|}{2\beta}{\rm d}c=\beta.
\end{equation}

\noindent{\bf Case 2. $\alpha\in(0,1/2)$.} Consider the case $\reg(C)\equiv 1/2+\beta$ almost surely. We construct the marginal distribution of $C$ as follows. Let $\PP(C=c_k)=p_k$ for $1\leq k\leq K$ with $K\geq (1-2\alpha)^{-1}$. Here $p_1=\cdots=p_{K-1}=p$ while $p_K=1-(K-1)p$ where $p$ is the non-negative root of $(K-1)p^2+(1-(K-1)p)^2=1-2\alpha$. Since $K\geq (1-2\alpha)^{-1}$, such $p\in(0,(K-1)^{-1}]$ exists. 
Moreover, we  let $\{c_k\}_{k=1}^K$ satisfy $0\leq c_1<\dots< c_{K-1}\leq 1/2+\beta$, $c_k+c_{K-k}\equiv 1$ with $c_k\neq 1/2$ for all $1\leq k< K$, $c_K=1/2$.
Then, by definition, we can calculate
\begin{align}
    {\rm RCE}=&\sum_{k=1}^Kp_k\left(1-\sum_{\ell\geq k}p_\ell\right)=\sum_{1\leq \ell <k\leq K}p_kp_\ell\\
    =&\frac{\left(\sum_{k=1}^Kp_k\right)^2-\sum_{k=1}^Kp_k^2}{2}=\frac{1-\sum_{k=1}^Kp_k^2}{2}=\alpha.
\end{align}
On the other hand, we have 
\begin{equation}
    {\rm ECE}=\sum_{k=1}^K\left|\frac{1}{2}+\beta-c_k\right|p_k=\beta+\frac{1}{2}-\sum_{k=1}^Kc_kp_k=\beta.
\end{equation}
This finishes the proof.


\section{Additional Experiment Details}

\subsection{Model Setup}
\label{app:model}
Following \citet{lin2023generating}, we set the temperature to 0.6 for the two Llama-2 models and 1.0 for the GPT model. We quantize the two Llama-2 models to 16 bits.  To ablate the influence of temperature, we also use generated responses of Llama-2-7b-chat with temperature 1.0. 

\subsection{Datasets}
\label{app:datasets}

\paragraph{Dataset Descriptions.}
TriviaQA is a challenging reading comprehension dataset, containing question-answer pairs whose answers can be found on Wikipedia and the web. 
Similar to previous works, we use TriviaQA as an open-domain QA benchmark. 
Natural Question is a question-answering dataset containing questions issued to the Google search engine. We use Natural Questions as an open-domain QA benchmark. SQuAD-1 is a reading comprehension dataset containing questions posed by crowdworkers based on Wikipedia articles. We include SQuAD-1 as a reading comprehension benchmark, where the annotated contexts are provided in the prompt.  
Meadow is created by research groups working on COVID-19 problems. 
We use this dataset for open-ended generation, where the LM is expected to provide a title for a paper given the abstract of the paper. The correctness is justified by comparing the generated title to the true title. 

\paragraph{Dataset Setup.}
TriviaQA contains 11,322 data points, Natural Questions contains 3,600 data points, SQuAD-1 contains 10,570 data points, and Meadow contains 1,000 data points. The prompt templates used are similar to those in \citet{kuhn2023semantic,lin2023generating}, and are as follows:

\noindent \textbf{TriviaQA:} following from \citet{lin2023generating}, we use the exact same prompt used in \citet{touvron2023llama}:

{\fontfamily{qcr}\selectfont
\noindent Answer these questions: \\
 In Scotland, a bothy/bothie is a? \\
A: House \\
 \{question\} \\
A:
}

\noindent \textbf{Natural Question:}
Similar to \citet{lin2023generating}, we use an in-context learning prompt with five demonstrations: 

{\fontfamily{qcr}\selectfont
    \noindent  where are the fa cup semi finals played. [SEP] A: the new Wembley Stadium.[SEP] \\
     who was alf married to in home and away [SEP] A: Ailsa Stewart.[SEP] \\
     what is the name of the first book in the twilight series [SEP] A: Twilight.[SEP] \\
     when is tornado season in the united states [SEP] A: March through June.[SEP] \\
     where did the idea of a messiah come from [SEP] A: Judaism.[SEP] \\ 
     {question} [SEP] A:
}

\noindent \textbf{SQuAD-1:} Each data point in SQuAD-1 is a (question, context, reference) triplet, where the context is annotated to provide useful information to answer the question. We prompt SQuAD-1 using zero-shot prompting:

{\fontfamily{qcr}\selectfont
    \noindent Answer the following question based on the context. \\
     \{question\} \\
    Context: \{context\} \\
    A: 
}

\noindent \textbf{Meadow:} Each data point in Meadow is a (abstract, title) pair. We prompt Meadow using one-shot prompting:

{\fontfamily{qcr}\selectfont
    \noindent Abstract: Coronavirus disease 2019 (COVID-19) threatens vulnerable patient populations, resulting in immense pressures at the local, regional, national, and international levels to contain the virus. Laboratory-based studies demonstrate that masks may offer benefits in reducing the spread of droplet-based illnesses, but few data are available to assess mask effects via executive order on a population basis. We assess the effects of a county-wide mask order on per-population mortality, intensive care unit (ICU) utilization, and ventilator utilization in Bexar County, Texas. METHODS: We used publicly reported county-level data to perform a mixed-methods before-and-after analysis along with other sources of public data for analyses of covariance. We used a least-squares regression analysis to adjust for confounders. A Texas state-level mask order was issued on July 3, 2020, followed by a Bexar County–level order on July 15, 2020. We defined the control period as June 2 to July 2 and the postmask order period as July 8, 2020–August 12, 2020, with a 5-day gap to account for the median incubation period for cases; longer periods of 7 and 10 days were used for hospitalization and ICU admission/death, respectively. Data are reported on a per-100,000 population basis using respective US Census Bureau–reported populations. RESULTS: From June 2, 2020 through August 12, 2020, there were 40,771 reported cases of COVID-19 within Bexar County, with 470 total deaths. The average number of new cases per day within the county was 565.4 (95\% confidence interval [CI] 394.6–736.2). The average number of positive hospitalized patients was 754.1 (95\% CI 657.2–851.0), in the ICU was 273.1 (95\% CI 238.2–308.0), and on a ventilator was 170.5 (95\% CI 146.4–194.6). The average deaths per day was 6.5 (95\% CI 4.4–8.6). All of the measured outcomes were higher on average in the postmask period as were covariables included in the adjusted model. When adjusting for traffic activity, total statewide caseload, public health complaints, and mean temperature, the daily caseload, hospital bed occupancy, ICU bed occupancy, ventilator occupancy, and daily mortality remained higher in the postmask period. CONCLUSIONS: There was no reduction in per-population daily mortality, hospital bed, ICU bed, or ventilator occupancy of COVID-19-positive patients attributable to the implementation of a mask-wearing mandate. [SEP]
    Title: Analysis of the Effects of COVID-19 Mask Mandates on Hospital Resource Consumption and Mortality at the County Level [SEP] \\
    Abstract: \{abstract\} [SEP] \\
    Title:
}

\subsection{Correctness Functions}
\label{app:correctness}

\paragraph{Rouge score.}  Recall-Oriented Understudy for Gist Evaluation  (Rouge) score
has originally been designed to evaluate machine translation or text summarization tasks. 
The Rouge score counts the overlapping n-grams between generated reference texts. 
Widely used n-grams include unigrams (Rouge-1), bigrams (Rouge-2), and the longest common subsequence (Rouge-L). Specifically, it is computed through
$$
\text{ROUGE} = \frac{|(\text{n-gram} \in \text{Generation}) \cap (\text{n-gram}) \in \text{Reference}|}{|\text{Reference}|}.
$$

\paragraph{METEOR score.} 
The Metric for Evaluation of Translation with Explicit Ordering (METEOR) score has also  been originally designed to evaluate machine translation and text summarization. Different from the Rouge score, the METEOR score considers the accuracy and fluency of the generation, as well as word order. The calculation of the METEOR score can be found in \citet{banerjee-lavie-2005-meteor}.

\paragraph{BERT-similarity.}
The BERT-similarity is based on sentence-bert~\citep{reimers2019sentencebert}. Specifically, in the first step, reference and generation texts are encoded as 768-dimensional feature vectors, respectively. Then, the correctness values are computed by calculating the cosine similarity between reference and generation vectors. In our implementation, we use sentence-Bert with \textit{bert-nli-mean-tokens} pre-trained weights as the encoding model. 

\paragraph{ChatGPT evaluation.}
ChatGPT evaluation 
is calculated by prompting GPT-3.5-turbo 
with the question, reference, and generation; and asking it to evaluate the correctness of the generation. The template used in calculating ChatGPT correctness follows that in \citet{lin2023generating}:

{\fontfamily{qcr}\selectfont
    \noindent Rate the level of consistency between the answer to the question and the reference answer, from 0 to 100.\\
    Question: In Scotland a bothy/bothie is a?\\
    Reference: House\\
    Answer: House\\
    Rating: 100.\\
    Question: Where in England was Dame Judi Dench born?\\
    Reference: York\\
    Answer: London\\
    Rating: 0.\\
    Question: \{question\}\\
    Reference: \{reference\}\\
    Answer: \{generated\}\\
    Rating:
}

\subsection{Inconsistency due to Correctness Thresholding}
\label{sec:auarcs}
We provide more evidence to show the inconsistency of AUARC and AUPRC metrics caused by the ad hoc correctness thresholding. The plots are in Fig~\ref{fig:auarcs}, ~\ref{fig:rouge_meadow}, \ref{fig:bert_triviaqa}, \ref{fig:rouge_llama_triviaqa}, and \ref{fig:rouge_llama_10}.
\begin{figure}
    \centering
    \includegraphics[width=0.4\textwidth]{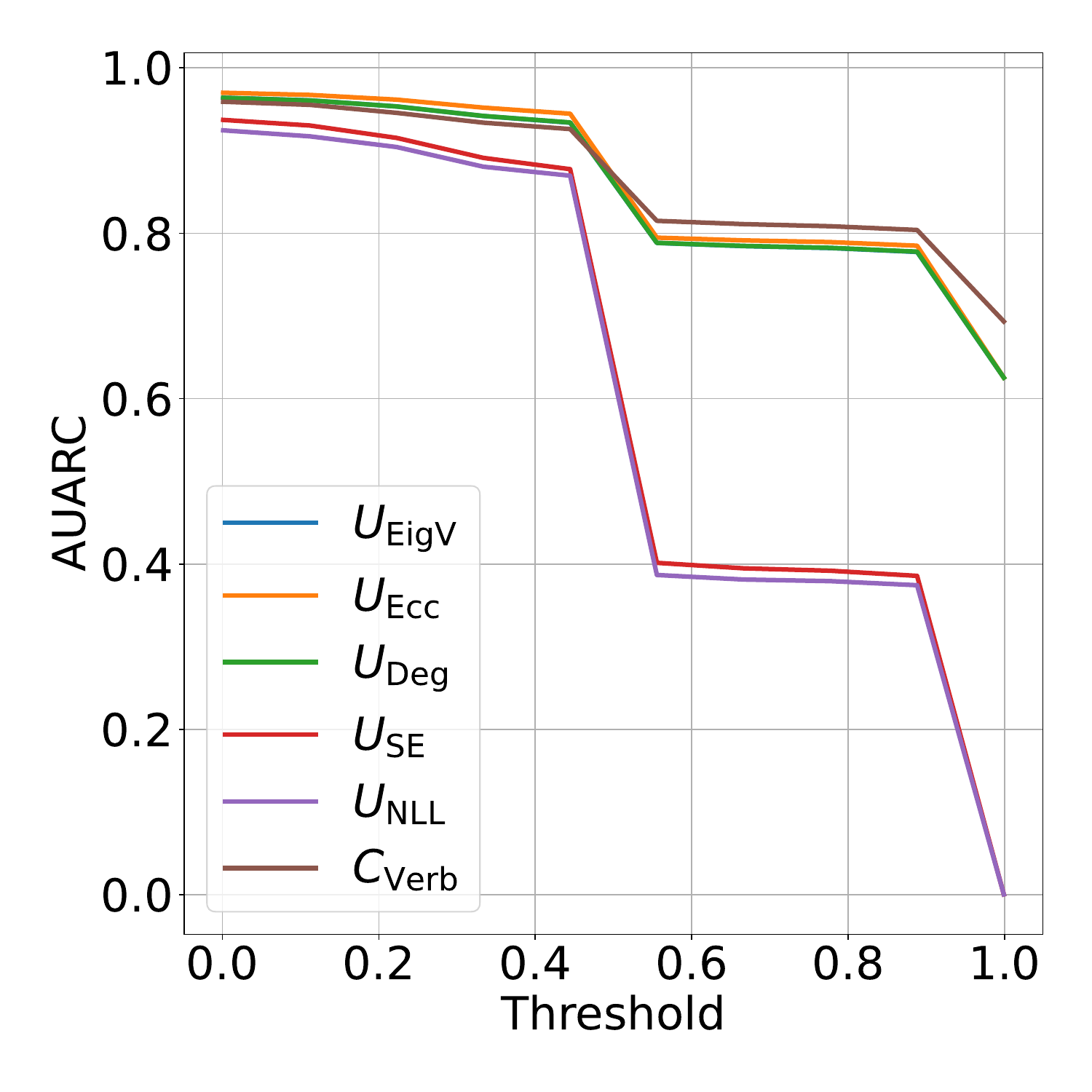}
    \includegraphics[width=0.4\textwidth]{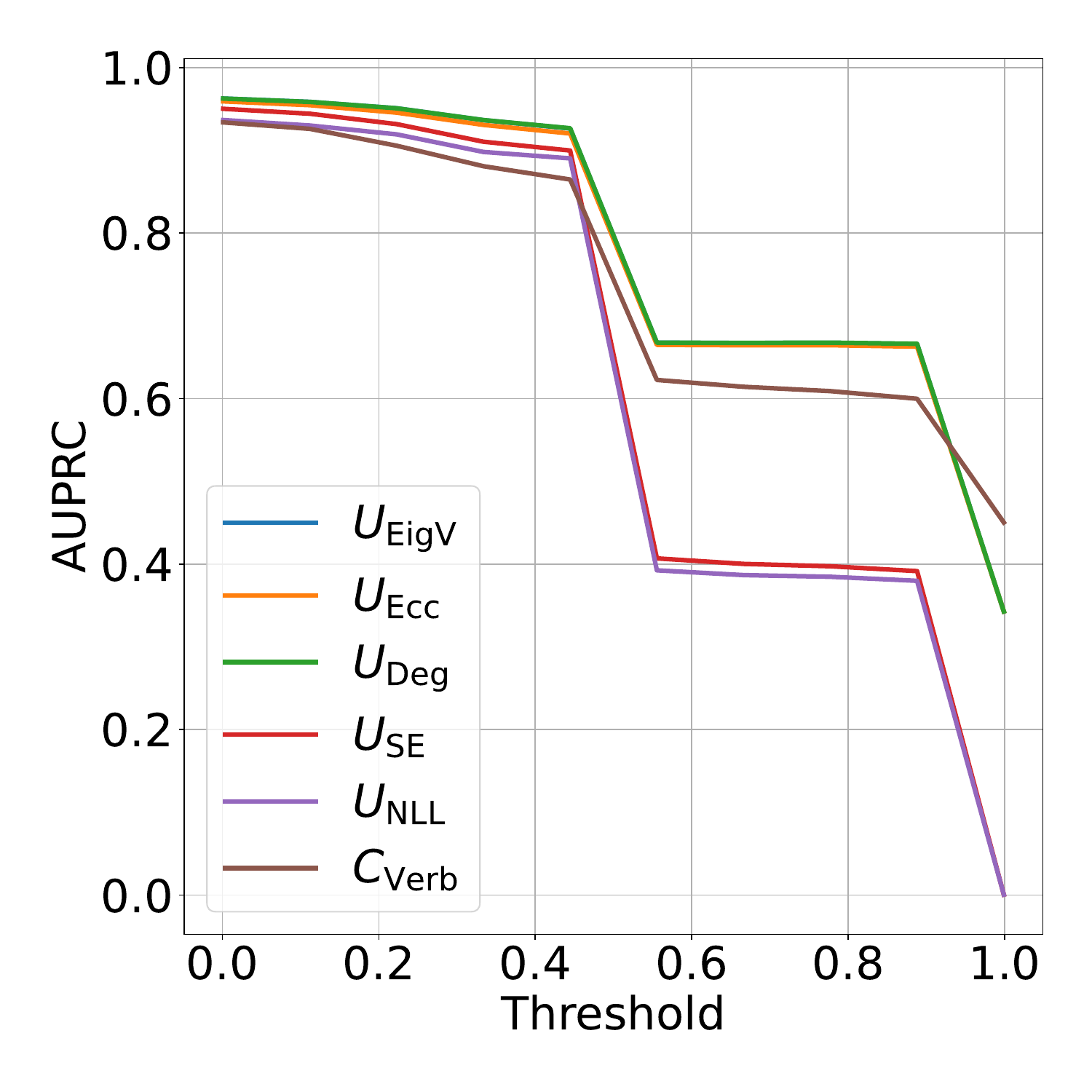}
    \caption{The assessed results for {\em AUARC (left)} and {\em AUPRC (right)} of uncertainty/confidence measures for GPT-3.5-turbo on the TriviaQA benchmark using the METEOR correctness score with varying thresholds. }
    \label{fig:auarcs}
\end{figure}

\begin{figure}
\centering
\begin{subfigure}{.24\textwidth}
  \centering
  \includegraphics[width=\linewidth]{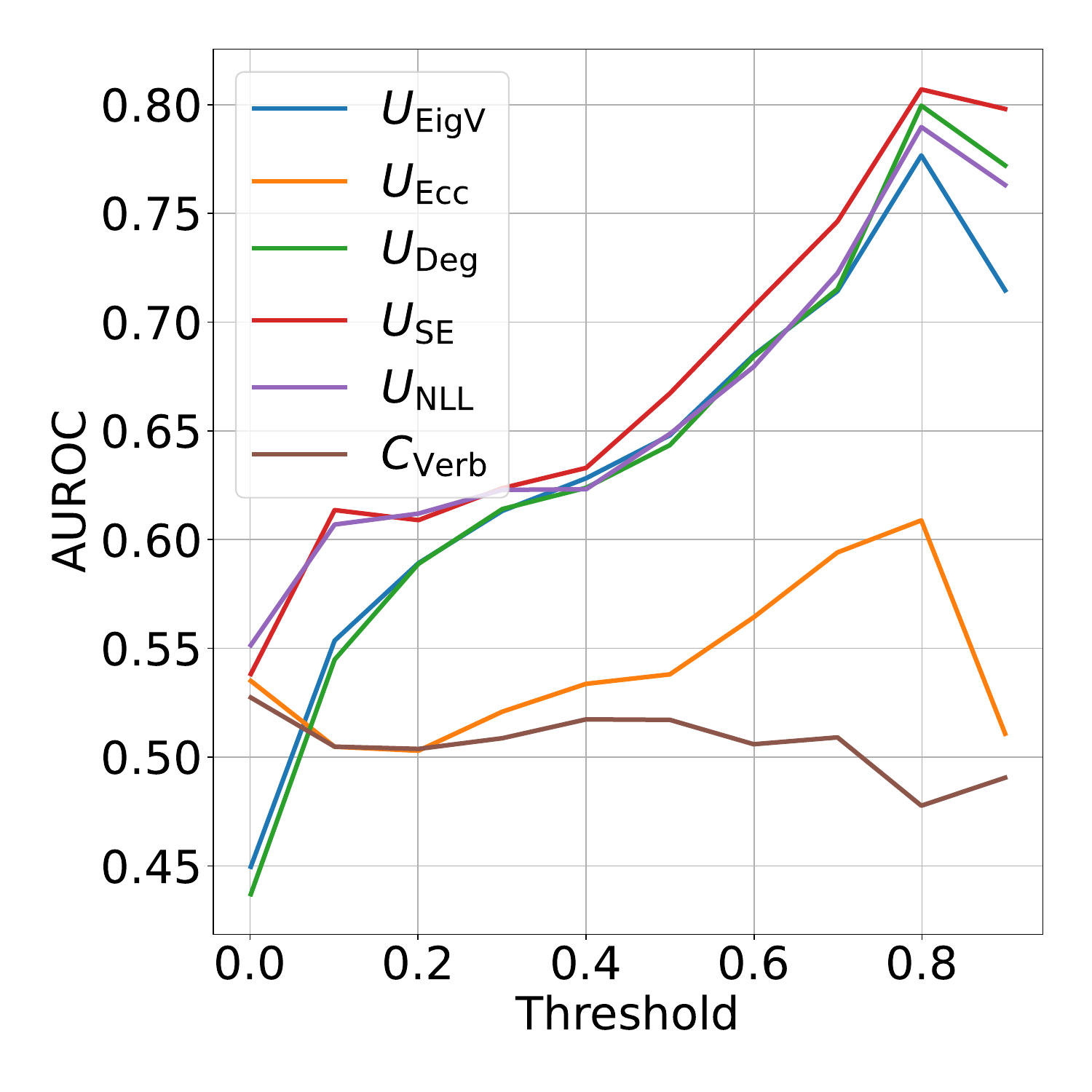}
  \caption{AUROC}
  \label{fig:auroc_meadow}
\end{subfigure}%
\begin{subfigure}{.24\textwidth}
  \centering
  \includegraphics[width=\linewidth]{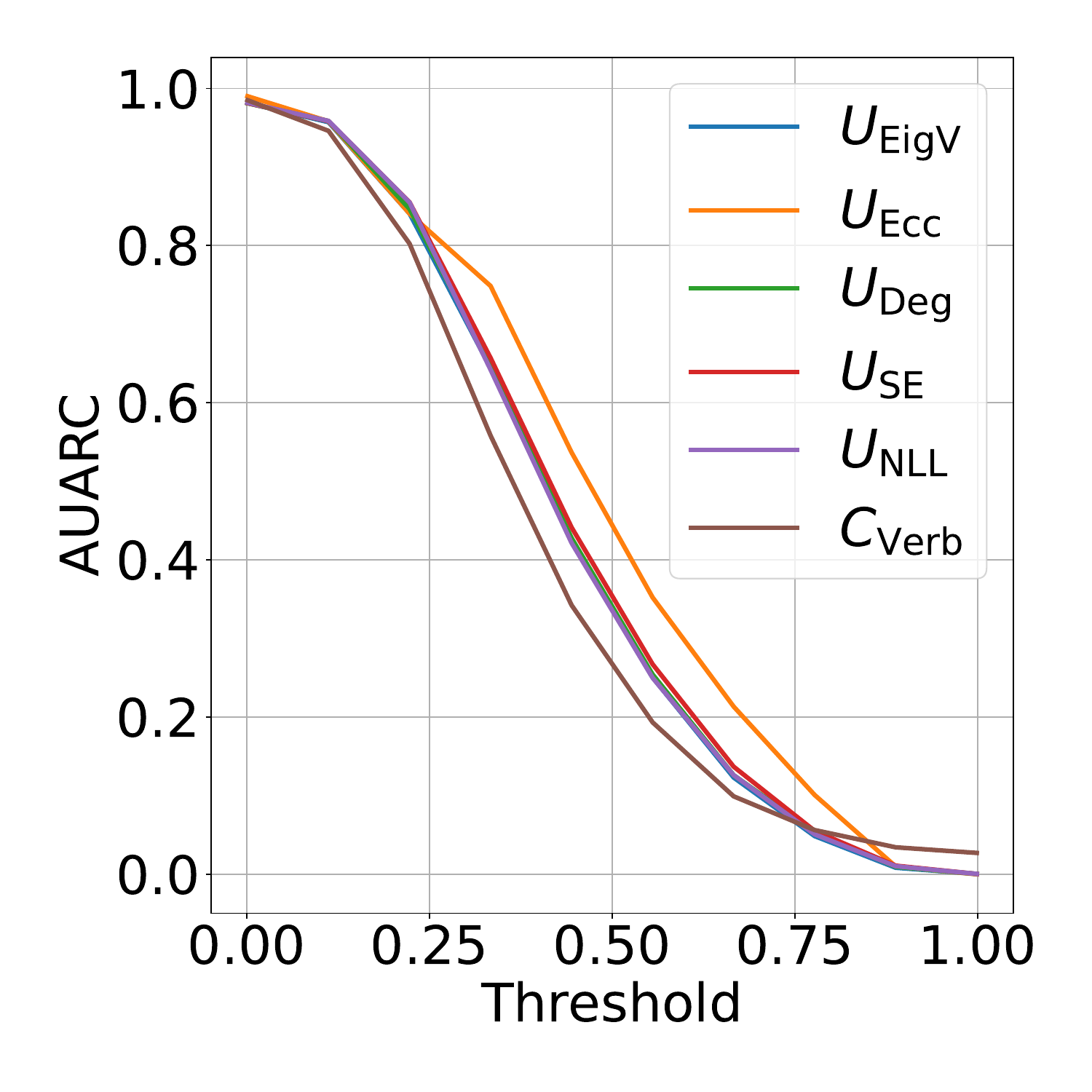}
  \caption{AUARC}
  \label{fig:auarc_meadow}
\end{subfigure}
\begin{subfigure}{.24\textwidth}
  \centering
  \includegraphics[width=\linewidth]{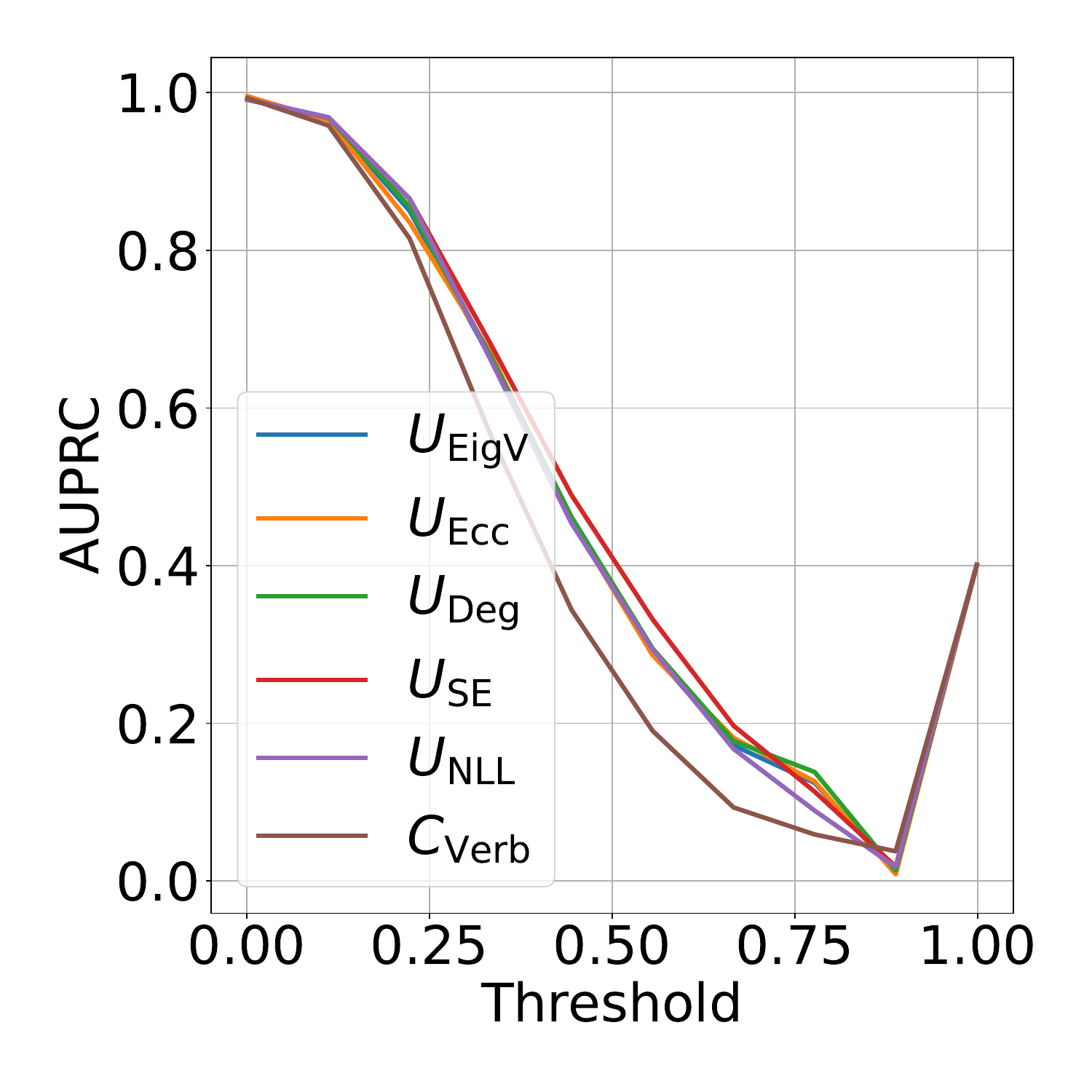}
  \caption{AUPRC}
  \label{fig:auprc_meadow}
\end{subfigure}%
\begin{subfigure}{.24\textwidth}
  \centering
  \includegraphics[width=\linewidth]{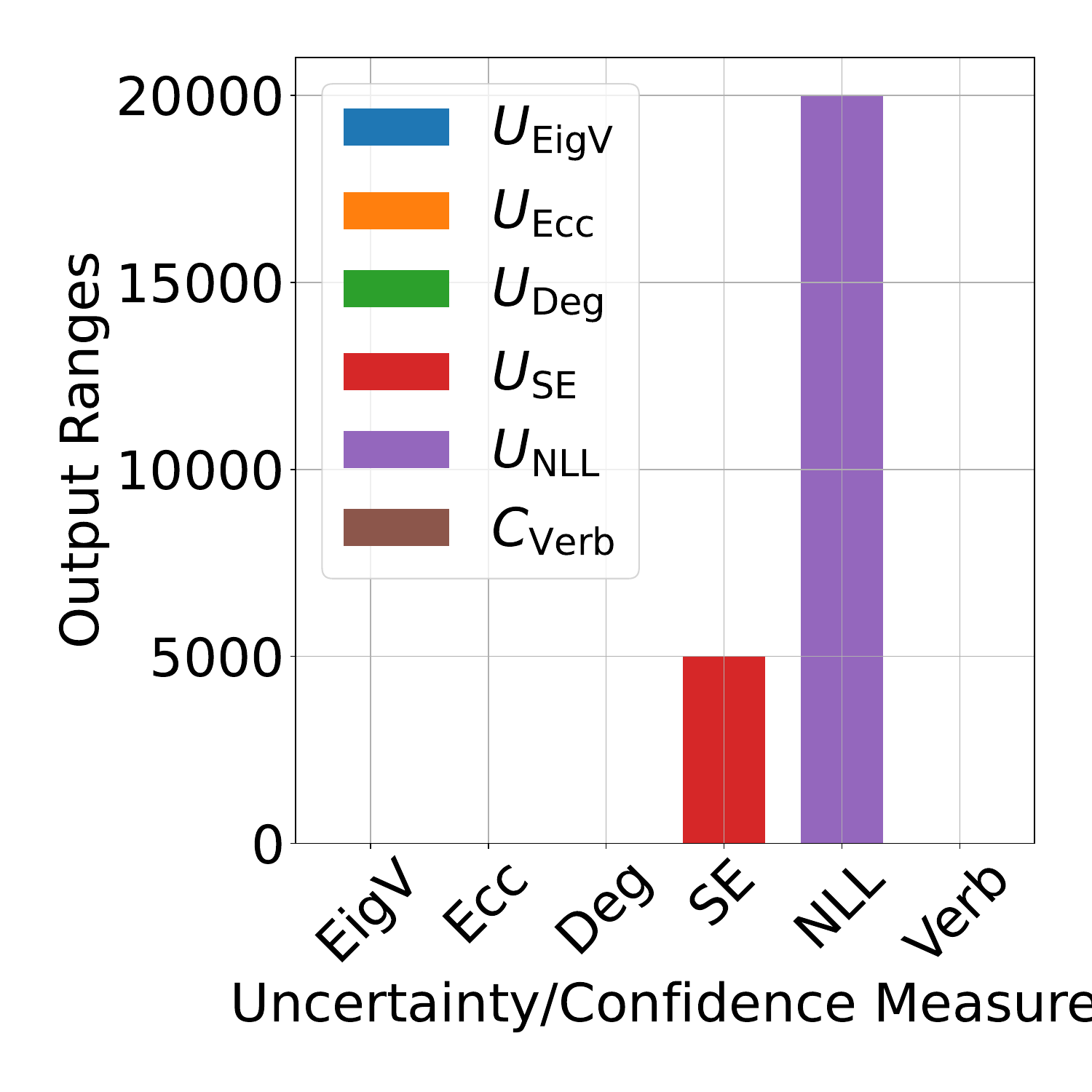}
  \caption{Output ranges 
  }
  \label{fig:unc_meadow}
\end{subfigure}
\caption{Results for Meadow using GPT-3.5-turbo and the Rouge score.}
\label{fig:rouge_meadow}
\end{figure}

\begin{figure}
\centering
\begin{subfigure}{.24\textwidth}
  \centering
  \includegraphics[width=\linewidth]{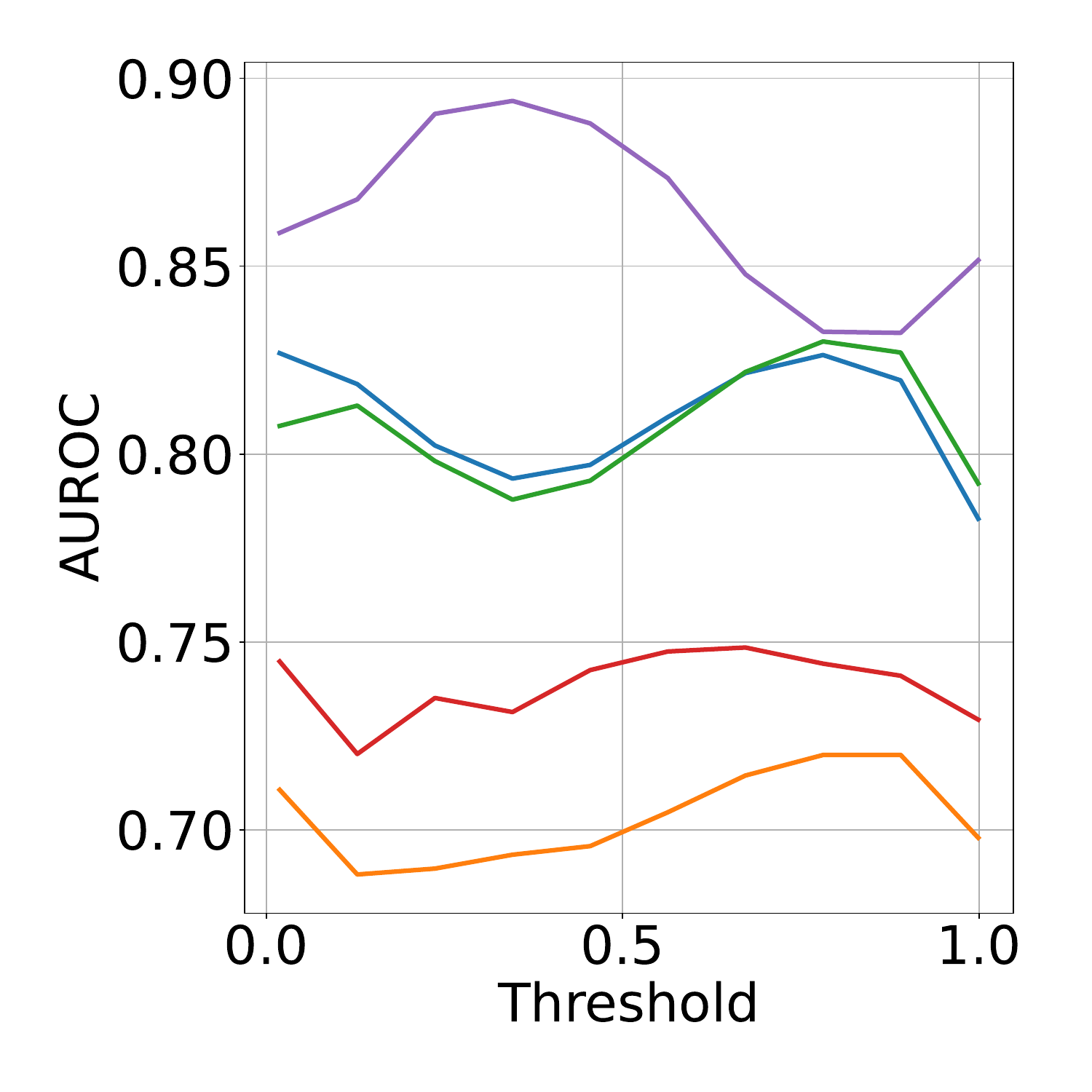}
  \caption{AUROC}
  \label{fig:auroc_bert}
\end{subfigure}%
\begin{subfigure}{.24\textwidth}
  \centering
  \includegraphics[width=\linewidth]{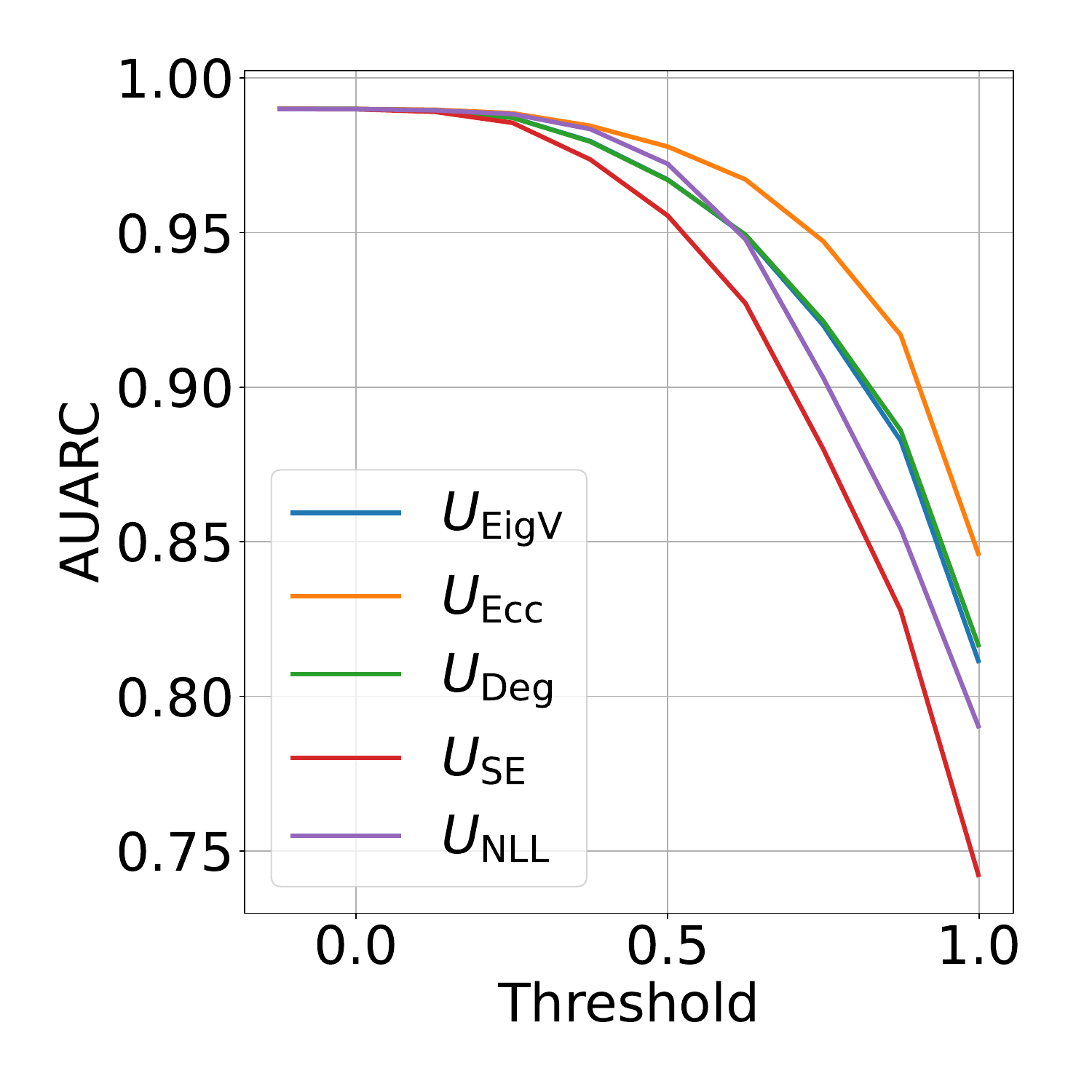}
  \caption{AUARC}
  \label{fig:auarc_bert}
\end{subfigure}
\begin{subfigure}{.24\textwidth}
  \centering
  \includegraphics[width=\linewidth]{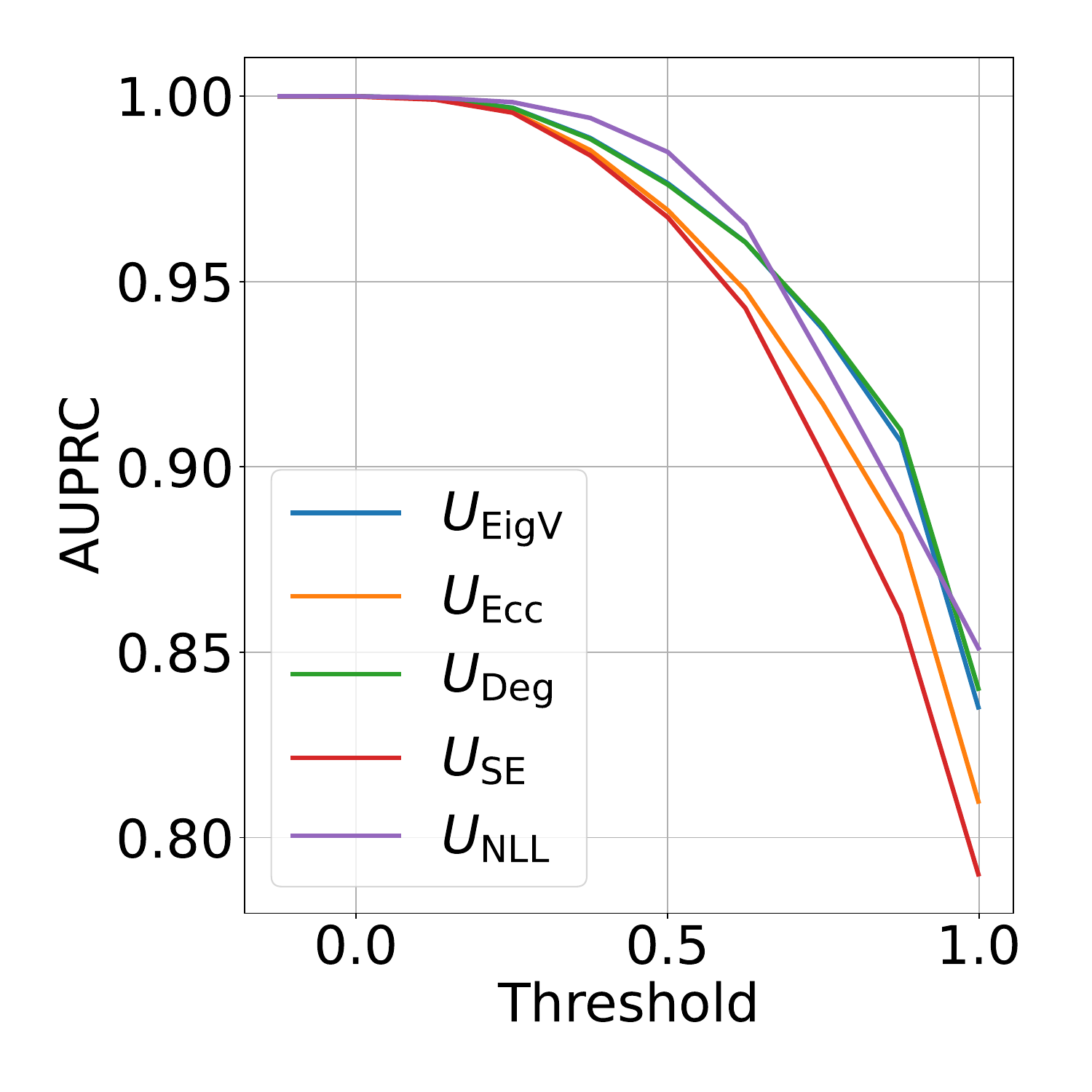}
  \caption{AUPRC}
  \label{fig:auprc_bert}
\end{subfigure}%
\begin{subfigure}{.24\textwidth}
  \centering
  \includegraphics[width=\linewidth]{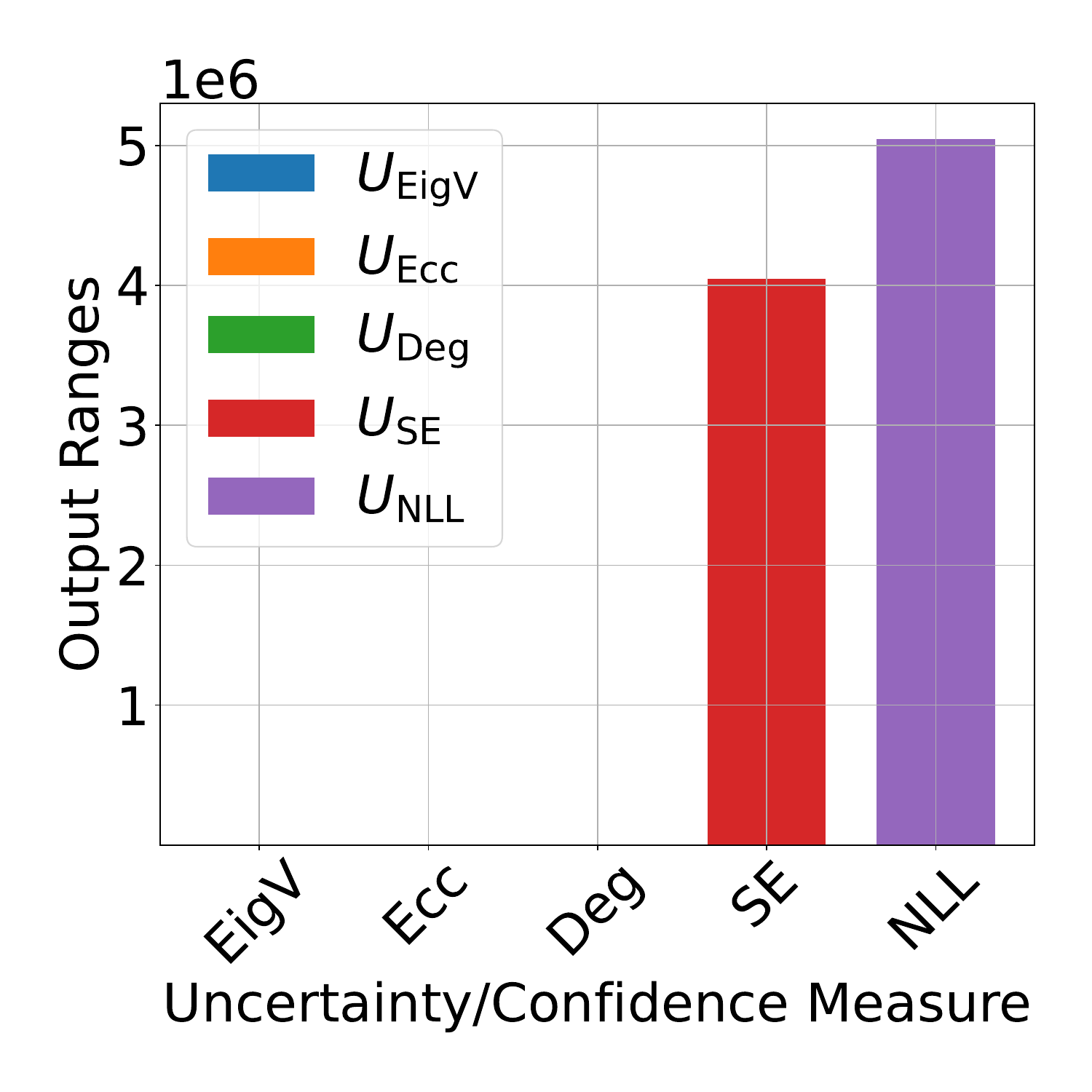}
  \caption{Output ranges}
  \label{fig:unc_bert}
\end{subfigure}
\caption{Results for TriviaQA using GPT-3.5-turbo with temperature 1.5 and the bert-similarity metric.}
\label{fig:bert_triviaqa}
\end{figure}

\begin{figure}
\centering
\begin{subfigure}{.24\textwidth}
  \centering
  \includegraphics[width=\linewidth]{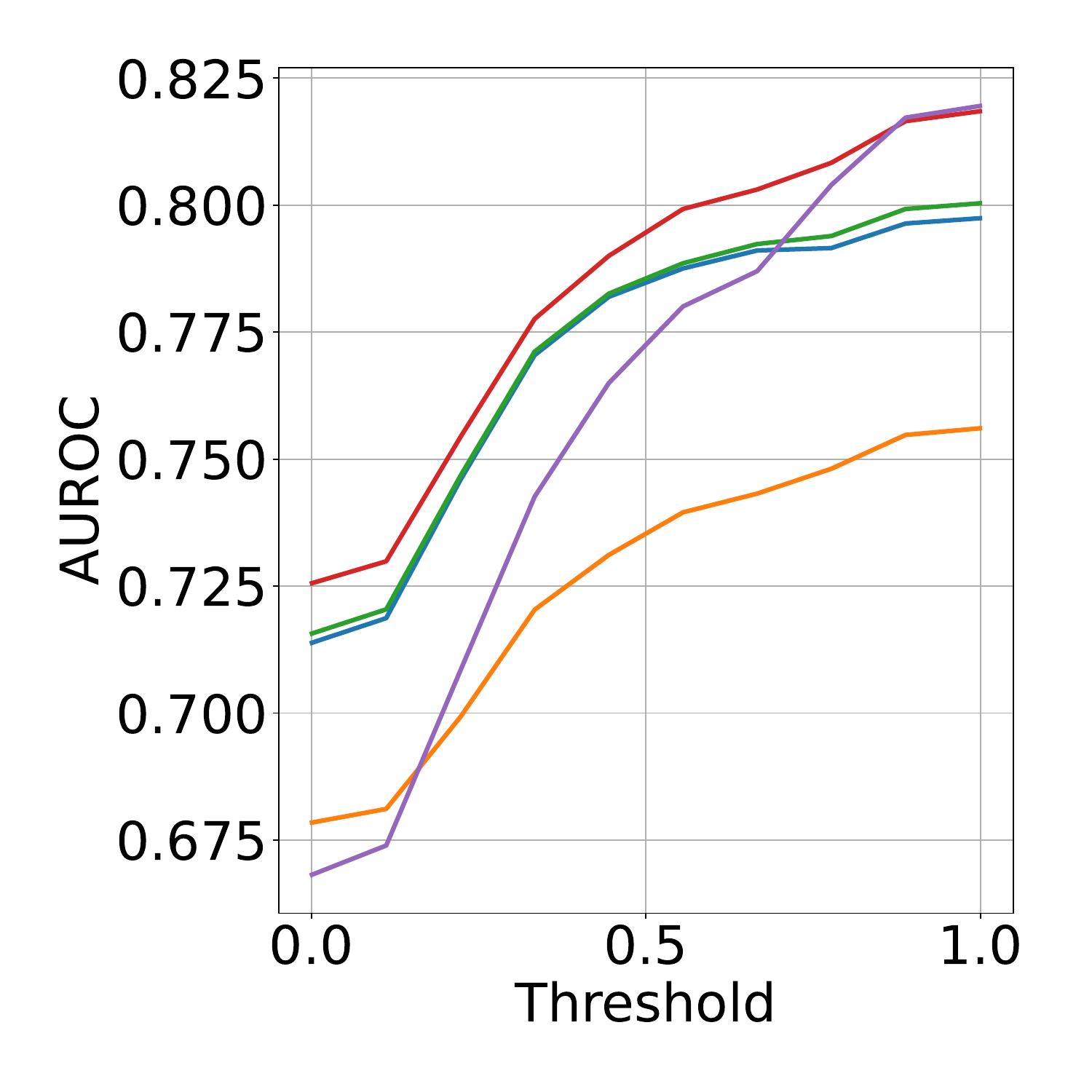}
  \caption{AUROC}
\end{subfigure}%
\begin{subfigure}{.24\textwidth}
  \centering
  \includegraphics[width=\linewidth]{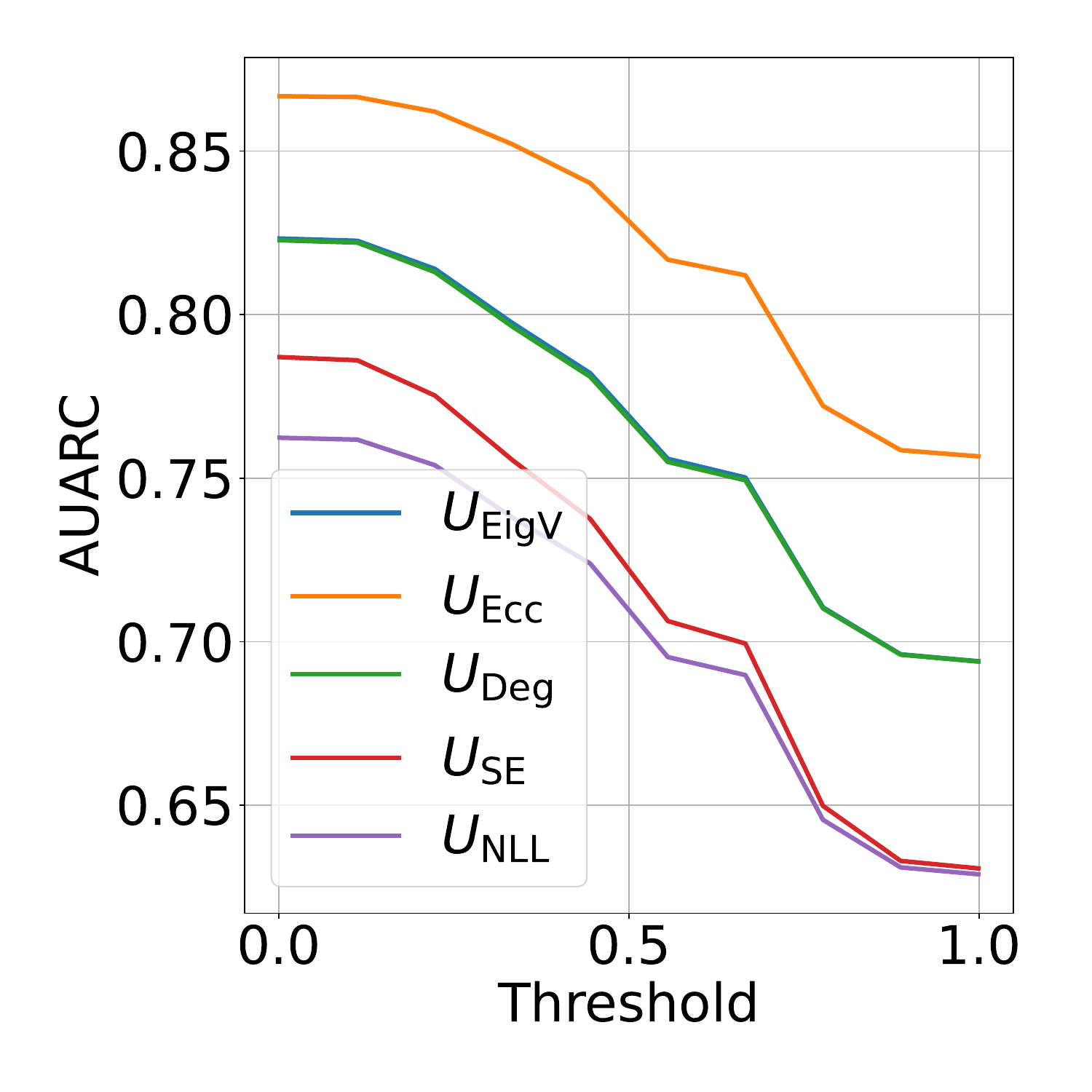}
  \caption{AUARC}
\end{subfigure}
\begin{subfigure}{.24\textwidth}
  \centering
  \includegraphics[width=\linewidth]{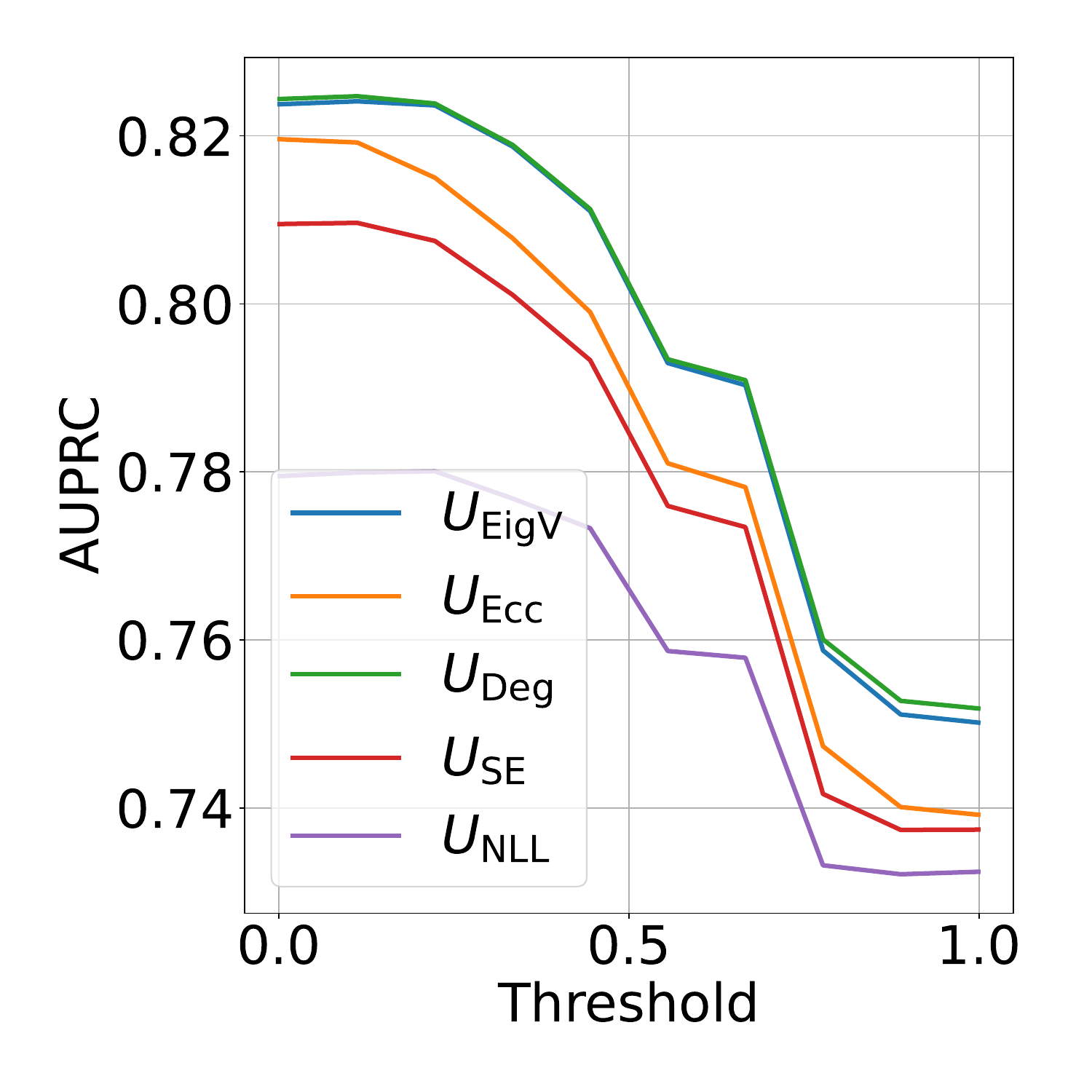}
  \caption{AUPRC}
\end{subfigure}%
\begin{subfigure}{.24\textwidth}
  \centering
  \includegraphics[width=\linewidth]{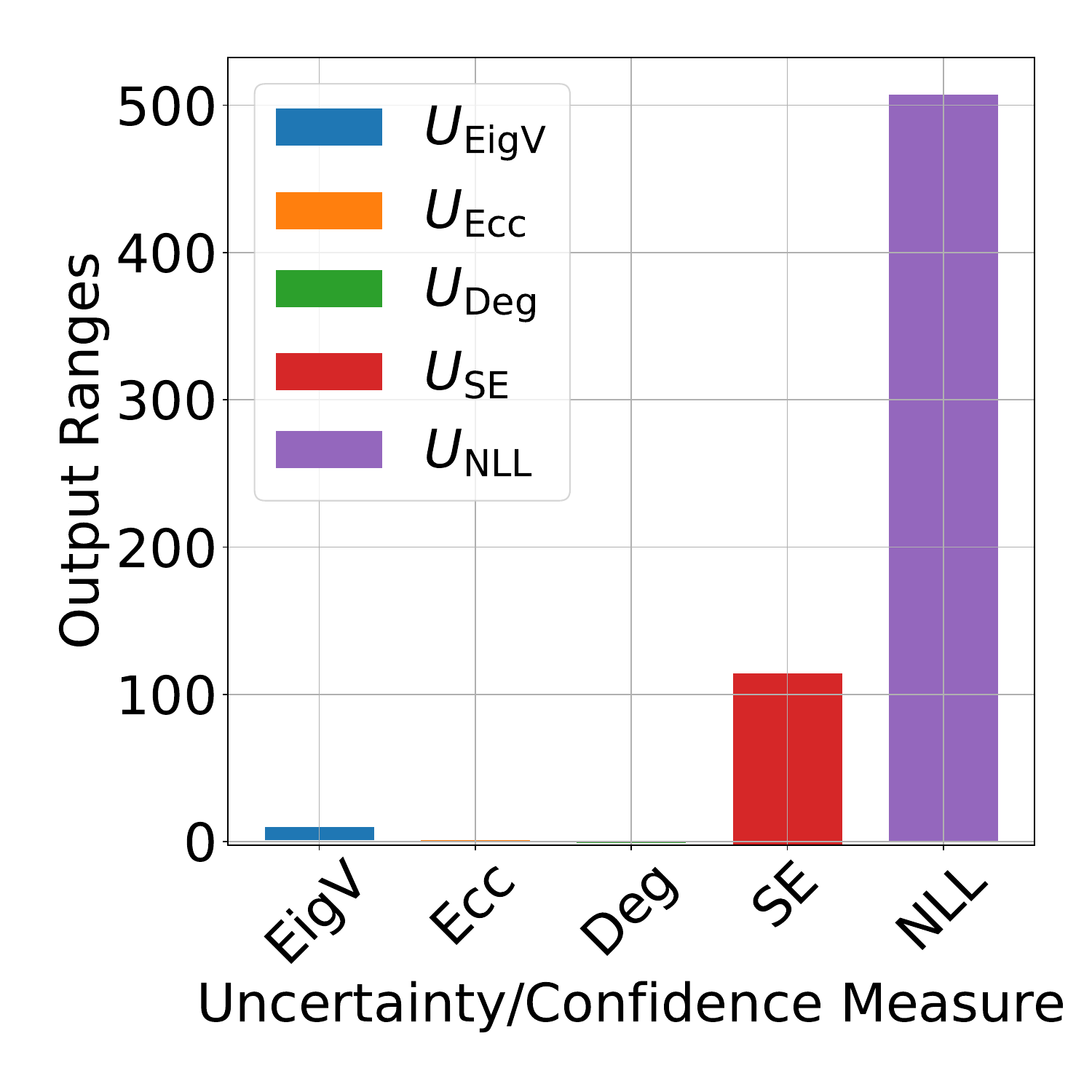}
  \caption{Output ranges}
\end{subfigure}
\caption{Results for TriviaQA using Llama-2-7b-chat and the Rouge score.}
\label{fig:rouge_llama_triviaqa}
\end{figure}

\begin{figure}[ht]
\centering
\begin{subfigure}{.24\textwidth}
  \centering
  \includegraphics[width=\linewidth]{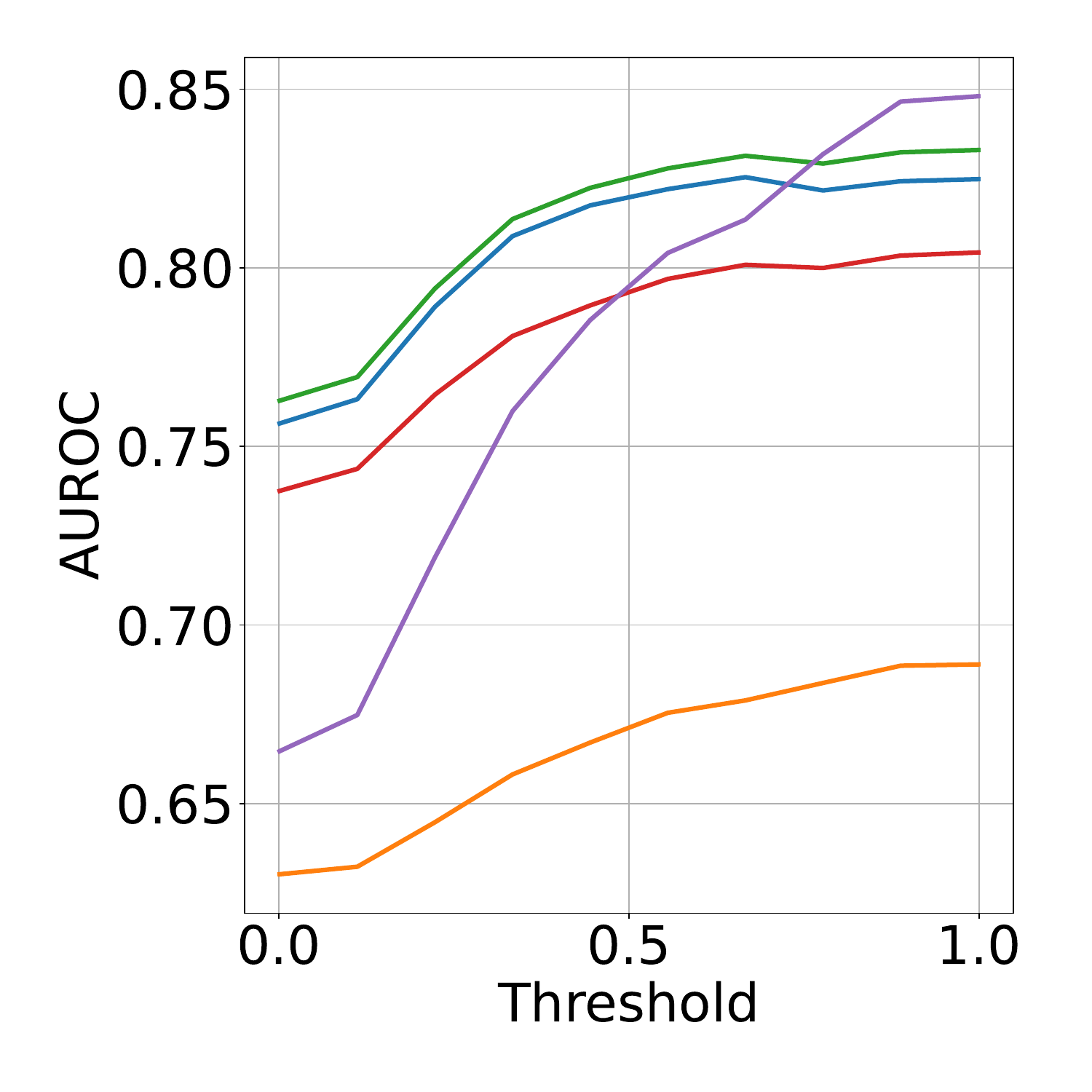}
  \caption{AUROC}
\end{subfigure}%
\begin{subfigure}{.24\textwidth}
  \centering
  \includegraphics[width=\linewidth]{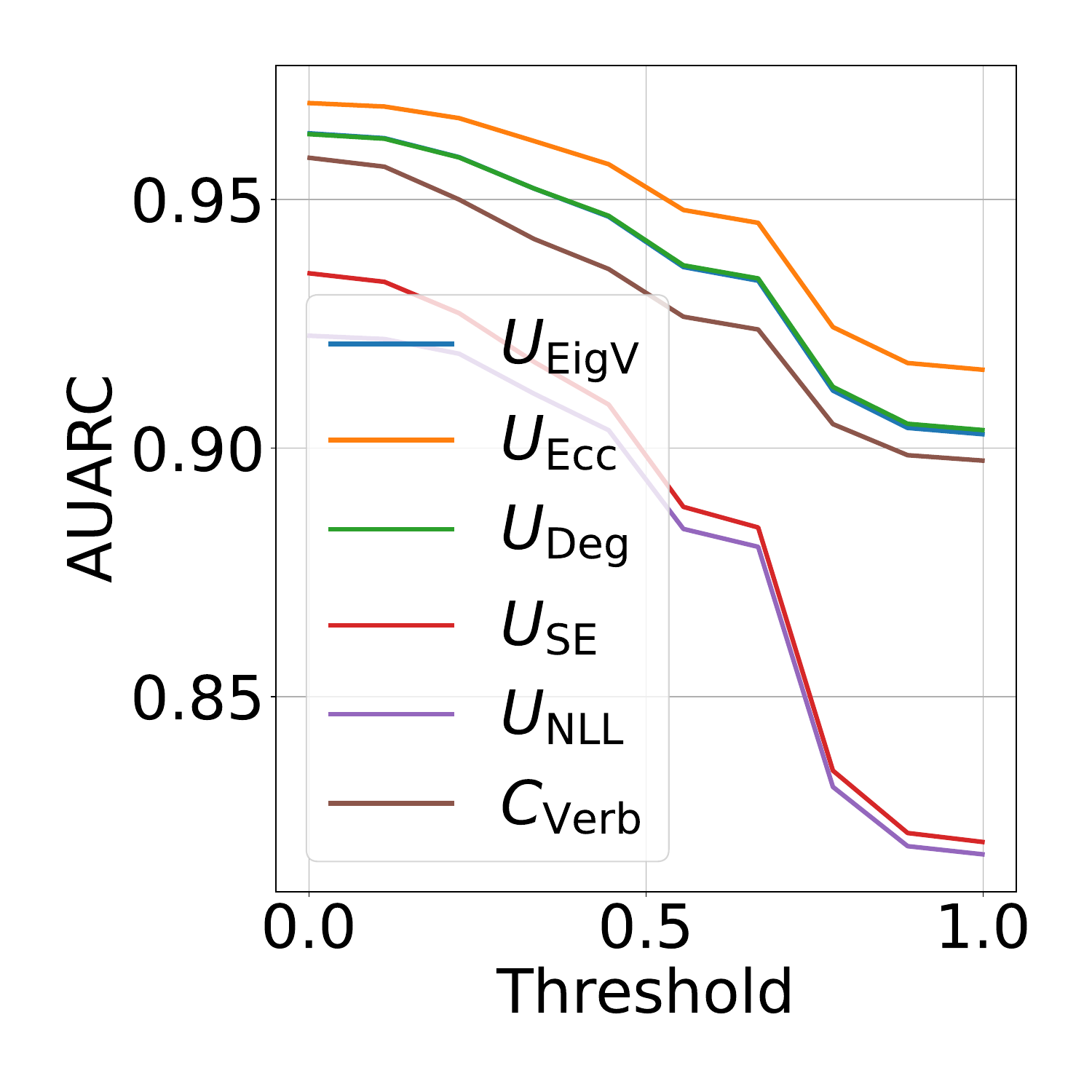}
  \caption{AUARC}
\end{subfigure}
\begin{subfigure}{.24\textwidth}
  \centering
  \includegraphics[width=\linewidth]{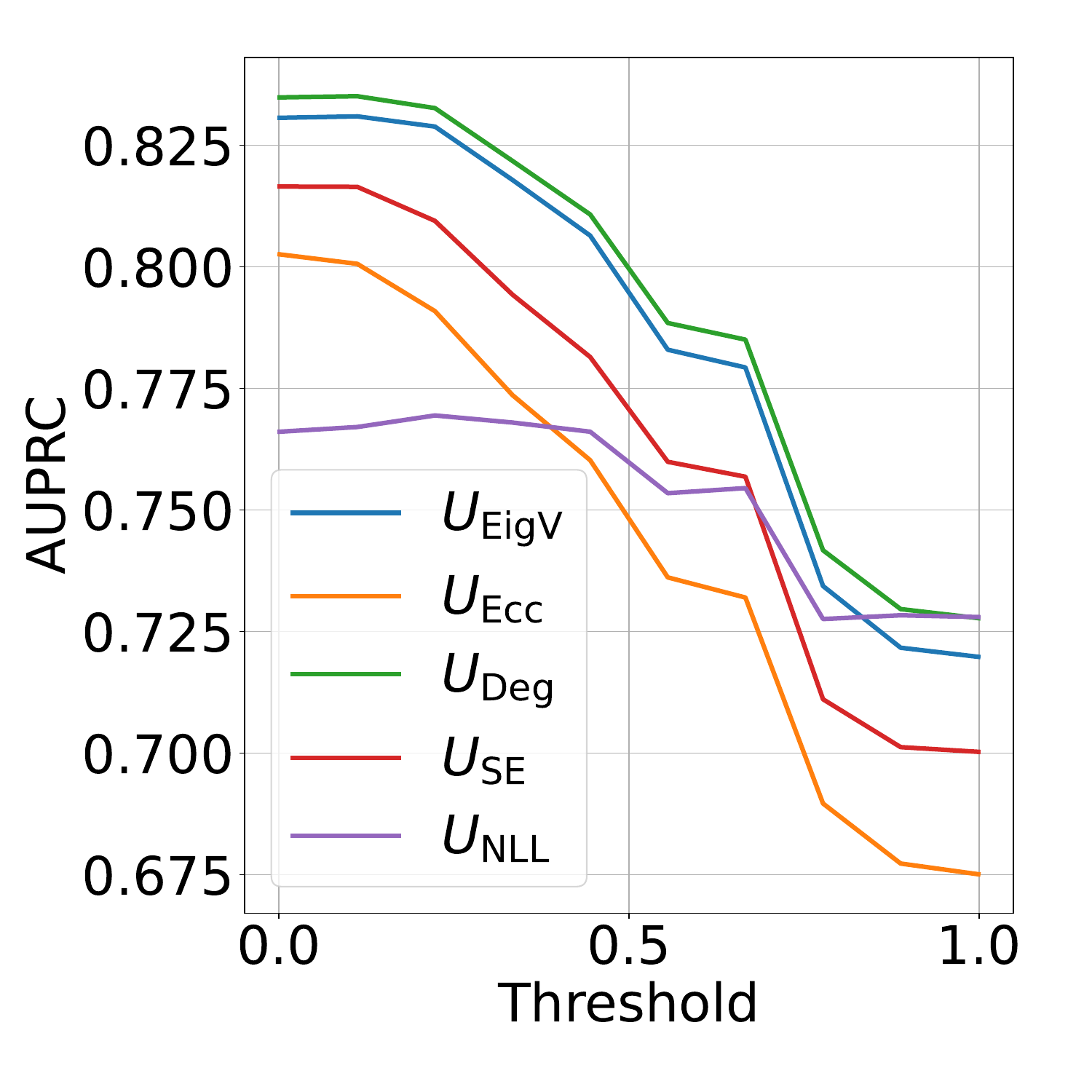}
  \caption{AUPRC}
\end{subfigure}%
\begin{subfigure}{.24\textwidth}
  \centering
  \includegraphics[width=\linewidth]{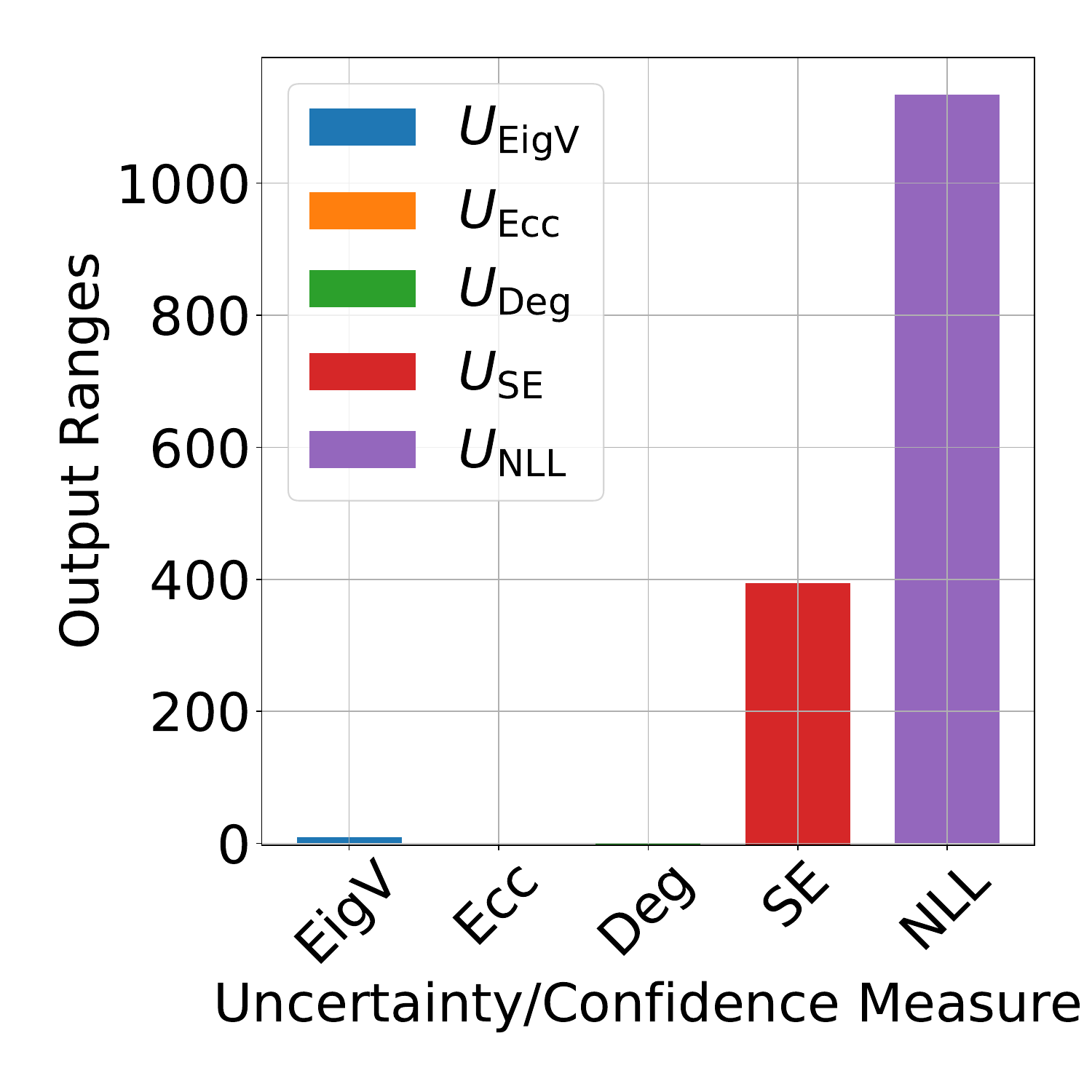}
  \caption{Output ranges}
\end{subfigure}
\caption{Results for  TriviaQA using Llama-2-7b-chat using temperature 1.0 and the Rouge score.}
\label{fig:rouge_llama_10}
\end{figure}

\section{Additional Experimental Results}

\begin{table}
    \centering
    \resizebox{\columnwidth}{!}{%
\begin{tabular}{llllllllll} \toprule
Model & Dataset & Correctness & Temperature  & $U_{\rm Ecc}$ & $U_{\rm Deg}$ & $U_{\rm EigV}$ & $U_{\rm NLL}$ & $U_{\rm SE}$ & $C_{\rm Verb}$ \\ \midrule
\multirow[c]{12}{*}{Llama-2} & \multirow[c]{4}{*}{nq-open} & bert & 0.6 & 0.302$_{\pm 0.044}$ & \bfseries 0.044$_{\pm 0.011}$ & 0.046$_{\pm 0.007}$ & 0.121$_{\pm 0.016}$ & 0.122$_{\pm 0.025}$ & nan \\ \cline{3-10}
 &  & meteor & 0.6 & 0.293$_{\pm 0.027}$ & \bfseries 0.072$_{\pm 0.010}$ & 0.077$_{\pm 0.015}$ & 0.167$_{\pm 0.021}$ & 0.137$_{\pm 0.024}$ & nan \\ \cline{3-10}
 &  & rougeL & 0.6 & 0.297$_{\pm 0.039}$ & 0.058$_{\pm 0.010}$ & \bfseries 0.051$_{\pm 0.010}$ & 0.147$_{\pm 0.021}$ & 0.124$_{\pm 0.019}$ & nan \\ \cline{3-10}
 &  & rouge1 & 0.6 & 0.297$_{\pm 0.038}$ & 0.057$_{\pm 0.011}$ & \bfseries 0.051$_{\pm 0.010}$ & 0.148$_{\pm 0.021}$ & 0.124$_{\pm 0.020}$ & nan \\ \cline{2-10}
 & \multirow[c]{4}{*}{squad} & bert & 0.6 & 0.308$_{\pm 0.041}$ & 0.071$_{\pm 0.013}$ & \bfseries 0.064$_{\pm 0.013}$ & 0.072$_{\pm 0.008}$ & 0.181$_{\pm 0.027}$ & nan \\ \cline{3-10}
 &  & meteor & 0.6 & 0.299$_{\pm 0.049}$ & 0.252$_{\pm 0.027}$ & \bfseries 0.247$_{\pm 0.029}$ & 0.419$_{\pm 0.018}$ & 0.407$_{\pm 0.024}$ & nan \\ \cline{3-10}
 &  & rougeL & 0.6 & 0.359$_{\pm 0.045}$ & \bfseries 0.139$_{\pm 0.033}$ & 0.150$_{\pm 0.027}$ & 0.187$_{\pm 0.028}$ & 0.332$_{\pm 0.036}$ & nan \\ \cline{3-10}
 &  & rouge1 & 0.6 & 0.360$_{\pm 0.044}$ & \bfseries 0.141$_{\pm 0.034}$ & 0.150$_{\pm 0.027}$ & 0.195$_{\pm 0.032}$ & 0.337$_{\pm 0.035}$ & nan \\ \cline{2-10}
 & \multirow[c]{4}{*}{triviaqa} & bert & 0.6 & 0.312$_{\pm 0.052}$ & \bfseries 0.020$_{\pm 0.005}$ & 0.028$_{\pm 0.007}$ & 0.244$_{\pm 0.012}$ & 0.061$_{\pm 0.008}$ & nan \\ \cline{3-10}
 &  & meteor & 0.6 & 0.305$_{\pm 0.048}$ & \bfseries 0.041$_{\pm 0.007}$ & 0.049$_{\pm 0.010}$ & 0.271$_{\pm 0.020}$ & 0.052$_{\pm 0.007}$ & nan \\ \cline{3-10}
 &  & rougeL & 0.6 & 0.305$_{\pm 0.050}$ & \bfseries 0.026$_{\pm 0.005}$ & 0.033$_{\pm 0.006}$ & 0.206$_{\pm 0.020}$ & 0.051$_{\pm 0.007}$ & nan \\ \cline{3-10}
 &  & rouge1 & 0.6 & 0.307$_{\pm 0.049}$ & \bfseries 0.026$_{\pm 0.005}$ & 0.034$_{\pm 0.006}$ & 0.209$_{\pm 0.019}$ & 0.052$_{\pm 0.007}$ & nan \\ \cline{1-10}
\multirow[c]{24}{*}{Llama-2-chat} & \multirow[c]{8}{*}{nq-open} & \multirow[c]{2}{*}{bert} & 0.6 & 0.199$_{\pm 0.040}$ & \bfseries 0.046$_{\pm 0.008}$ & 0.052$_{\pm 0.010}$ & 0.101$_{\pm 0.015}$ & 0.062$_{\pm 0.010}$ & nan \\ \cline{4-10}
 &  &  & 1.0 & 0.236$_{\pm 0.033}$ & \bfseries 0.035$_{\pm 0.008}$ & 0.038$_{\pm 0.007}$ & 0.097$_{\pm 0.017}$ & 0.055$_{\pm 0.012}$ & nan \\ \cline{3-10}
 &  & \multirow[c]{2}{*}{meteor} & 0.6 & 0.190$_{\pm 0.039}$ & \bfseries 0.062$_{\pm 0.008}$ & 0.067$_{\pm 0.010}$ & 0.176$_{\pm 0.018}$ & 0.072$_{\pm 0.009}$ & nan \\ \cline{4-10}
 &  &  & 1.0 & 0.224$_{\pm 0.034}$ & \bfseries 0.044$_{\pm 0.006}$ & 0.046$_{\pm 0.007}$ & 0.209$_{\pm 0.023}$ & 0.074$_{\pm 0.015}$ & nan \\ \cline{3-10}
 &  & \multirow[c]{2}{*}{rougeL} & 0.6 & 0.198$_{\pm 0.039}$ & \bfseries 0.053$_{\pm 0.011}$ & 0.057$_{\pm 0.010}$ & 0.167$_{\pm 0.013}$ & 0.060$_{\pm 0.012}$ & nan \\ \cline{4-10}
 &  &  & 1.0 & 0.227$_{\pm 0.035}$ & 0.035$_{\pm 0.007}$ & \bfseries 0.033$_{\pm 0.006}$ & 0.211$_{\pm 0.021}$ & 0.069$_{\pm 0.016}$ & nan \\ \cline{3-10}
 &  & \multirow[c]{2}{*}{rouge1} & 0.6 & 0.199$_{\pm 0.039}$ & \bfseries 0.054$_{\pm 0.010}$ & 0.057$_{\pm 0.010}$ & 0.167$_{\pm 0.014}$ & 0.061$_{\pm 0.013}$ & nan \\ \cline{4-10}
 &  &  & 1.0 & 0.227$_{\pm 0.035}$ & 0.034$_{\pm 0.007}$ & \bfseries 0.033$_{\pm 0.006}$ & 0.212$_{\pm 0.021}$ & 0.069$_{\pm 0.015}$ & nan \\ \cline{2-10}
 & \multirow[c]{8}{*}{squad} & \multirow[c]{2}{*}{bert} & 0.6 & 0.208$_{\pm 0.033}$ & 0.065$_{\pm 0.014}$ & 0.075$_{\pm 0.017}$ & \bfseries 0.048$_{\pm 0.007}$ & 0.063$_{\pm 0.012}$ & nan \\ \cline{4-10}
 &  &  & 1.0 & 0.276$_{\pm 0.039}$ & 0.067$_{\pm 0.011}$ & 0.063$_{\pm 0.010}$ & \bfseries 0.038$_{\pm 0.006}$ & 0.098$_{\pm 0.012}$ & nan \\ \cline{3-10}
 &  & \multirow[c]{2}{*}{meteor} & 0.6 & 0.216$_{\pm 0.038}$ & 0.303$_{\pm 0.026}$ & 0.265$_{\pm 0.022}$ & \bfseries 0.063$_{\pm 0.013}$ & 0.182$_{\pm 0.029}$ & nan \\ \cline{4-10}
 &  &  & 1.0 & 0.300$_{\pm 0.046}$ & 0.292$_{\pm 0.035}$ & 0.250$_{\pm 0.027}$ & \bfseries 0.064$_{\pm 0.011}$ & 0.274$_{\pm 0.021}$ & nan \\ \cline{3-10}
 &  & \multirow[c]{2}{*}{rougeL} & 0.6 & 0.239$_{\pm 0.036}$ & 0.177$_{\pm 0.026}$ & 0.143$_{\pm 0.020}$ & \bfseries 0.052$_{\pm 0.011}$ & 0.127$_{\pm 0.020}$ & nan \\ \cline{4-10}
 &  &  & 1.0 & 0.304$_{\pm 0.036}$ & 0.179$_{\pm 0.033}$ & 0.137$_{\pm 0.024}$ & \bfseries 0.053$_{\pm 0.012}$ & 0.210$_{\pm 0.027}$ & nan \\ \cline{3-10}
 &  & \multirow[c]{2}{*}{rouge1} & 0.6 & 0.238$_{\pm 0.037}$ & 0.183$_{\pm 0.027}$ & 0.148$_{\pm 0.022}$ & \bfseries 0.053$_{\pm 0.010}$ & 0.129$_{\pm 0.021}$ & nan \\ \cline{4-10}
 &  &  & 1.0 & 0.303$_{\pm 0.035}$ & 0.185$_{\pm 0.033}$ & 0.143$_{\pm 0.025}$ & \bfseries 0.053$_{\pm 0.012}$ & 0.213$_{\pm 0.026}$ & nan \\ \cline{2-10}
 & \multirow[c]{8}{*}{triviaqa} & \multirow[c]{2}{*}{bert} & 0.6 & 0.140$_{\pm 0.024}$ & 0.062$_{\pm 0.016}$ & 0.061$_{\pm 0.015}$ & \bfseries 0.020$_{\pm 0.004}$ & 0.027$_{\pm 0.007}$ & nan \\ \cline{4-10}
 &  &  & 1.0 & 0.213$_{\pm 0.030}$ & 0.025$_{\pm 0.006}$ & 0.034$_{\pm 0.006}$ & \bfseries 0.014$_{\pm 0.002}$ & 0.036$_{\pm 0.006}$ & nan \\ \cline{3-10}
 &  & \multirow[c]{2}{*}{meteor} & 0.6 & 0.145$_{\pm 0.027}$ & 0.067$_{\pm 0.017}$ & 0.064$_{\pm 0.015}$ & \bfseries 0.034$_{\pm 0.009}$ & 0.075$_{\pm 0.016}$ & nan \\ \cline{4-10}
 &  &  & 1.0 & 0.206$_{\pm 0.032}$ & \bfseries 0.035$_{\pm 0.007}$ & 0.046$_{\pm 0.005}$ & 0.049$_{\pm 0.008}$ & 0.084$_{\pm 0.007}$ & nan \\ \cline{3-10}
 &  & \multirow[c]{2}{*}{rougeL} & 0.6 & 0.141$_{\pm 0.021}$ & 0.062$_{\pm 0.014}$ & 0.061$_{\pm 0.014}$ & \bfseries 0.024$_{\pm 0.005}$ & 0.034$_{\pm 0.005}$ & nan \\ \cline{4-10}
 &  &  & 1.0 & 0.204$_{\pm 0.035}$ & 0.027$_{\pm 0.006}$ & 0.040$_{\pm 0.004}$ & \bfseries 0.022$_{\pm 0.002}$ & 0.051$_{\pm 0.007}$ & nan \\ \cline{3-10}
 &  & \multirow[c]{2}{*}{rouge1} & 0.6 & 0.141$_{\pm 0.021}$ & 0.062$_{\pm 0.014}$ & 0.062$_{\pm 0.013}$ & \bfseries 0.024$_{\pm 0.005}$ & 0.034$_{\pm 0.006}$ & nan \\ \cline{4-10}
 &  &  & 1.0 & 0.203$_{\pm 0.035}$ & 0.027$_{\pm 0.006}$ & 0.040$_{\pm 0.004}$ & \bfseries 0.022$_{\pm 0.002}$ & 0.051$_{\pm 0.007}$ & nan \\ \cline{1-10}
\multirow[c]{24}{*}{GPT-3.5} & \multirow[c]{4}{*}{meadow} & bert & 1.0 & 0.284$_{\pm 0.035}$ & 0.178$_{\pm 0.030}$ & 0.174$_{\pm 0.025}$ & \bfseries 0.112$_{\pm 0.022}$ & 0.177$_{\pm 0.027}$ & 0.288$_{\pm 0.033}$ \\ \cline{3-10}
 &  & meteor & 1.0 & 0.292$_{\pm 0.045}$ & 0.134$_{\pm 0.027}$ & 0.137$_{\pm 0.026}$ & \bfseries 0.074$_{\pm 0.012}$ & 0.132$_{\pm 0.018}$ & 0.263$_{\pm 0.050}$ \\ \cline{3-10}
 &  & rougeL & 1.0 & 0.278$_{\pm 0.045}$ & 0.130$_{\pm 0.022}$ & 0.131$_{\pm 0.025}$ & \bfseries 0.056$_{\pm 0.010}$ & 0.113$_{\pm 0.022}$ & 0.289$_{\pm 0.046}$ \\ \cline{3-10}
 &  & rouge1 & 1.0 & 0.290$_{\pm 0.047}$ & 0.126$_{\pm 0.018}$ & 0.135$_{\pm 0.020}$ & \bfseries 0.059$_{\pm 0.013}$ & 0.113$_{\pm 0.018}$ & 0.299$_{\pm 0.047}$ \\ \cline{2-10}
 & \multirow[c]{4}{*}{nq-open} & bert & 1.0 & 0.151$_{\pm 0.025}$ & 0.050$_{\pm 0.012}$ & 0.065$_{\pm 0.014}$ & \bfseries 0.039$_{\pm 0.008}$ & 0.050$_{\pm 0.007}$ & 0.487$_{\pm 0.005}$ \\ \cline{3-10}
 &  & meteor & 1.0 & 0.154$_{\pm 0.027}$ & 0.050$_{\pm 0.011}$ & 0.063$_{\pm 0.011}$ & \bfseries 0.046$_{\pm 0.011}$ & 0.060$_{\pm 0.009}$ & 0.452$_{\pm 0.018}$ \\ \cline{3-10}
 &  & rougeL & 1.0 & 0.151$_{\pm 0.022}$ & 0.048$_{\pm 0.011}$ & 0.062$_{\pm 0.012}$ & \bfseries 0.034$_{\pm 0.009}$ & 0.052$_{\pm 0.008}$ & 0.487$_{\pm 0.006}$ \\ \cline{3-10}
 &  & rouge1 & 1.0 & 0.153$_{\pm 0.023}$ & 0.048$_{\pm 0.011}$ & 0.063$_{\pm 0.012}$ & \bfseries 0.034$_{\pm 0.009}$ & 0.051$_{\pm 0.008}$ & 0.487$_{\pm 0.006}$ \\ \cline{2-10}
 & \multirow[c]{4}{*}{squad} & bert & 1.0 & 0.204$_{\pm 0.025}$ & 0.237$_{\pm 0.024}$ & 0.240$_{\pm 0.019}$ & \bfseries 0.065$_{\pm 0.012}$ & 0.113$_{\pm 0.013}$ & 0.181$_{\pm 0.029}$ \\ \cline{3-10}
 &  & meteor & 1.0 & 0.181$_{\pm 0.012}$ & 0.151$_{\pm 0.016}$ & 0.193$_{\pm 0.020}$ & \bfseries 0.054$_{\pm 0.017}$ & 0.086$_{\pm 0.014}$ & 0.182$_{\pm 0.032}$ \\ \cline{3-10}
 &  & rougeL & 1.0 & 0.222$_{\pm 0.025}$ & 0.270$_{\pm 0.023}$ & 0.269$_{\pm 0.016}$ & \bfseries 0.037$_{\pm 0.010}$ & 0.100$_{\pm 0.011}$ & 0.168$_{\pm 0.035}$ \\ \cline{3-10}
 &  & rouge1 & 1.0 & 0.226$_{\pm 0.024}$ & 0.276$_{\pm 0.023}$ & 0.270$_{\pm 0.017}$ & \bfseries 0.039$_{\pm 0.010}$ & 0.103$_{\pm 0.011}$ & 0.168$_{\pm 0.035}$ \\ \cline{2-10}
 & \multirow[c]{12}{*}{triviaqa} & \multirow[c]{3}{*}{bert} & 0.5 & 0.215$_{\pm 0.042}$ & 0.212$_{\pm 0.040}$ & 0.212$_{\pm 0.041}$ & \bfseries 0.043$_{\pm 0.006}$ & 0.052$_{\pm 0.009}$ & nan \\ \cline{4-10}
 &  &  & 1.0 & 0.152$_{\pm 0.025}$ & 0.129$_{\pm 0.020}$ & 0.133$_{\pm 0.020}$ & \bfseries 0.039$_{\pm 0.007}$ & 0.052$_{\pm 0.012}$ & 0.182$_{\pm 0.025}$ \\ \cline{4-10}
 &  &  & 1.5 & 0.142$_{\pm 0.018}$ & 0.053$_{\pm 0.011}$ & 0.074$_{\pm 0.012}$ & \bfseries 0.031$_{\pm 0.007}$ & 0.081$_{\pm 0.009}$ & nan \\ \cline{3-10}
 &  & \multirow[c]{3}{*}{meteor} & 0.5 & 0.215$_{\pm 0.049}$ & 0.211$_{\pm 0.045}$ & 0.208$_{\pm 0.047}$ & \bfseries 0.179$_{\pm 0.021}$ & 0.234$_{\pm 0.019}$ & nan \\ \cline{4-10}
 &  &  & 1.0 & 0.156$_{\pm 0.026}$ & 0.131$_{\pm 0.024}$ & \bfseries 0.131$_{\pm 0.022}$ & 0.146$_{\pm 0.011}$ & 0.209$_{\pm 0.012}$ & 0.194$_{\pm 0.036}$ \\ \cline{4-10}
 &  &  & 1.5 & 0.137$_{\pm 0.024}$ & \bfseries 0.059$_{\pm 0.011}$ & 0.077$_{\pm 0.012}$ & 0.119$_{\pm 0.010}$ & 0.176$_{\pm 0.015}$ & nan \\ \cline{3-10}
 &  & \multirow[c]{3}{*}{rougeL} & 0.5 & 0.214$_{\pm 0.046}$ & 0.210$_{\pm 0.042}$ & 0.207$_{\pm 0.041}$ & \bfseries 0.041$_{\pm 0.007}$ & 0.050$_{\pm 0.008}$ & nan \\ \cline{4-10}
 &  &  & 1.0 & 0.151$_{\pm 0.024}$ & 0.126$_{\pm 0.019}$ & 0.129$_{\pm 0.019}$ & \bfseries 0.038$_{\pm 0.007}$ & 0.059$_{\pm 0.009}$ & 0.181$_{\pm 0.026}$ \\ \cline{4-10}
 &  &  & 1.5 & 0.138$_{\pm 0.025}$ & 0.059$_{\pm 0.012}$ & 0.079$_{\pm 0.011}$ & \bfseries 0.034$_{\pm 0.008}$ & 0.104$_{\pm 0.007}$ & nan \\ \cline{3-10}
 &  & \multirow[c]{3}{*}{rouge1} & 0.5 & 0.216$_{\pm 0.046}$ & 0.212$_{\pm 0.043}$ & 0.209$_{\pm 0.042}$ & \bfseries 0.040$_{\pm 0.007}$ & 0.050$_{\pm 0.008}$ & nan \\ \cline{4-10}
 &  &  & 1.0 & 0.152$_{\pm 0.024}$ & 0.126$_{\pm 0.018}$ & 0.130$_{\pm 0.021}$ & \bfseries 0.039$_{\pm 0.007}$ & 0.060$_{\pm 0.009}$ & 0.176$_{\pm 0.027}$ \\ \cline{4-10}
 &  &  & 1.5 & 0.137$_{\pm 0.023}$ & 0.060$_{\pm 0.011}$ & 0.078$_{\pm 0.012}$ & \bfseries 0.034$_{\pm 0.008}$ & 0.105$_{\pm 0.008}$ & nan \\  \bottomrule
\end{tabular}
}
    \caption{RCE results for various experimental configurations.}
    \label{tab:all}
\end{table}

\begin{table}[H]
    \centering
    \resizebox{\columnwidth}{!}{%
\begin{tabular}{p{12cm}lllllll}
\toprule
Prompt & Reference & Generation & $\mathbb{P}(U_{\rm Ecc} \leq u)$ & $\mathbb{P}(U_{\rm Deg} \leq u)$ & $\mathbb{P}(U_{\rm EigV} \leq u)$ & $\mathbb{P}(U_{\rm SE} \leq u)$ & $\mathbb{P}(U_{\rm NLL} \leq u)$ \\
\midrule
Q: Who did Dr. Crippen murder? & his wife & His wife & 0.999 & 0.881 & 0.822 & 0.649 & 0.247 \\
Q: What are the only two musical notes which have no flats? & c and f & B and F & 0.999 & 0.761 & 0.769 & 0.898 & 0.691 \\
Q: Which Eastenders actor has played the policeman Nick Rowan on TV? & nick berry & Mark Jordon & 0.999 & 0.972 & 0.978 & 0.954 & 0.918 \\
Q: Which `B` was the name of the mechanical shark used in the original `Jaws` film? & bruce & Bruce & 0.999 & 0.761 & 0.769 & 0.337 & 0.183 \\
Q: Which actor does the interviewing in 'Interview with a Vampire'? & christian slater & Brad Pitt & 0.999 & 0.858 & 0.856 & 0.861 & 0.893 \\
Q: What did my true love bring to me on the Sixth Day of Christmas? & six geese-a-laying & Six geese a-laying & 0.999 & 0.761 & 0.769 & 0.736 & 0.688 \\
Q: In January 1957, Russell Endean became the first batsman to be dismissed from a test cricket match for doing what? & handling the ball & Handling the ball & 0.999 & 0.761 & 0.769 & 0.901 & 0.368 \\
Q: What are the first names of the two dancing instructors in the UK television series ‘Hi De Hi’? & barry and yvonne & Barry and Yvonne & 0.999 & 0.761 & 0.769 & 0.846 & 0.627 \\
Q: Who became the host of the UK television game show Blankety Blank in 1984? & les dawson & Les Dawson & 0.999 & 0.761 & 0.769 & 0.180 & 0.040 \\
Q: How much, in pounds sterling, does the Best in Show Winner receive at the annual Crufts Dog Show? & 100 pounds & £100 & 0.999 & 0.920 & 0.908 & 0.830 & 0.787 \\
Q: In the Billy Bunter stories, what is the surname of Bunter’s form teacher? & quelch & Quelch & 0.999 & 0.761 & 0.769 & 0.999 & 0.558 \\
Q: Which play is featured in the film The Producers? & springtime for hitler & Springtime for Hitler & 0.999 & 0.761 & 0.769 & 0.967 & 0.341 \\
Q: What provoked the war between Honduras and El Salvador in 1969? & a football match & A soccer match & 0.999 & 0.761 & 0.769 & 0.535 & 0.711 \\
Q: Which character was played by Linda Thorson in The Avengers? & tara king & Tara King & 0.999 & 0.824 & 0.885 & 0.919 & 0.399 \\
Q: According to a traditional English proverb, what is better than none? & half a loaf & A bad excuse & 0.999 & 0.972 & 0.978 & 0.931 & 0.908 \\
Q: In which Welsh village is there only one gay, apparently?! & llandewi breffi & Llanddewi Brefi & 0.999 & 0.926 & 0.963 & 0.950 & 0.906 \\
Q: On September 28th, NASA announced that what had been detected on Mars? & flowing water & Possible signs of life & 0.999 & 0.965 & 0.963 & 0.813 & 0.930 \\
Q: What are the first four words of the Bible, as recorded in Genesis? & in the beginning god & In the beginning, God & 0.653 & 0.650 & 0.651 & 0.574 & 0.557 \\
Q: Which national anthem was originally called the 'War Song for the Rhine Army'? & marsellaise & German national anthem & 0.694 & 0.858 & 0.837 & 0.785 & 0.888 \\
Q: Name the UK budget holiday company specialising in Turkey and Greece which went bust in July 2010? & goldtrail & Goldtrail & 0.999 & 0.920 & 0.902 & 0.894 & 0.655 \\
Q: Who has been President of France twice, but never been elected to the position? & alain poher & François Mitterrand & 0.999 & 0.920 & 0.902 & 0.854 & 0.864 \\
Q: What is the name of Madonna's proposed chain of fitness clubs? & hard candy fitness & Hard Candy Fitness & 0.999 & 0.761 & 0.769 & 0.996 & 0.183 \\
Q: Elvis Presley sang a few lines in German on which US hit song? & wooden heart & Wooden Heart & 0.999 & 0.761 & 0.769 & 0.998 & 0.270 \\
Q: What was the name of the book that was a collection of Aubrey Beardsley's work, published by Leonard Smithers in 1897? & a book of fifty drawings & The Yellow Book & 0.999 & 0.761 & 0.769 & 0.950 & 0.775 \\
Q: Dishes prepared with spinach can be referred to as what? & la florentine & Spinach dishes & 0.999 & 0.920 & 0.902 & 0.943 & 0.899 \\
Q: Which English civil engineer's most famous project was the construction of Tower Bridge over the River Thames in London? & sir john wolfe-barry & Sir John Wolfe Barry & 0.999 & 0.761 & 0.769 & 0.830 & 0.633 \\
Q: Where did the space probe New Horizons launched by NASA in 2006 aim to investigate? & pluto and the kuiper belt & Pluto and the Kuiper Belt & 0.999 & 0.905 & 0.904 & 0.905 & 0.576 \\
Q: Where woud you find a nave or an apse? & in a church & In a church & 0.999 & 0.761 & 0.769 & 0.236 & 0.185 \\
Q: What is the name of Jay-Z and Beyonce's daughter? & blue ivy & Blue Ivy & 0.999 & 0.976 & 0.965 & 0.975 & 0.354 \\
Q: 'Feel Like Making Love' and 'The First Time Ever I Saw Your Face' were hit singles for which female artist? & roberta flack & Roberta Flack & 0.999 & 0.761 & 0.769 & 0.864 & 0.046 \\
Q: In the nursery rhyme, who pulled pussy out of the well? & little tommy stout & Tommy & 0.999 & 0.976 & 0.987 & 0.962 & 0.882 \\
Q: "In the film of the same name, what was the name of ""The Hustler""?" & """fast eddie"" felson" & Fast Eddie Felson & 0.999 & 0.761 & 0.769 & 0.708 & 0.692 \\
Q: In Camberwick Green on Children's TV who was the commander of Pippin Fort? & captain snort & Captain Snort & 0.999 & 0.761 & 0.769 & 0.961 & 0.156 \\
Q: In Chigley on Children's TV who owned the steam railway and drove the steam engine 'Bessie'? & lord belborough & Lord Belborough & 0.999 & 0.761 & 0.769 & 0.951 & 0.401 \\
Q: Who won the gold medal in the women's Skeleton Bob at the 2010 Vancouver Winter Olympics? & amy williams & Amy Williams & 0.999 & 0.881 & 0.822 & 0.676 & 0.265 \\
Q: What decoration, a Cross, was first awarded in 1995 to Corporal Wayne Mills for his actions in Bosnia? & conspicuous gallantry & George Cross & 0.999 & 0.844 & 0.783 & 0.801 & 0.899 \\
Q: What was the French sounding winner of the 2011 Epsom Derby? & pour moi & Pour Moi & 0.999 & 0.761 & 0.769 & 0.321 & 0.101 \\
Q: Who originally provided the voice for TV's 'Basil Brush'? & ivan owen & Ivan Owen & 0.999 & 0.761 & 0.769 & 0.987 & 0.454 \\
Q: "Which actress played 'Valeria"" in the film Carry On Screaming?" & fenella fielding & Fenella Fielding & 0.999 & 0.761 & 0.769 & 0.862 & 0.206 \\
Q: Which of the 'Spice Girls' advertised 'Milky Way' ob t.v.? & emma bunton (baby spice) & Victoria Beckham (Posh Spice) & 0.999 & 0.949 & 0.963 & 0.985 & 0.847 \\
Q: Give any year in the life of the Portuguese prince known as Henry the Navigator. & 1394-1460 & 1394-1460 & 0.999 & 0.761 & 0.769 & 0.680 & 0.671 \\
Q: On which horse did Sir Gordon Richards ride his only Epsom Derby winner? & pinza & Pinza & 0.999 & 0.824 & 0.885 & 0.987 & 0.229 \\
Q: What was the name of the aeroplane in which Wiley Post became the first pilot to fly solo around the world? & 'winnie mae' & Winnie Mae & 0.999 & 0.761 & 0.769 & 0.849 & 0.654 \\
Q: Who was the husband of Rebekah Brooks from 2002 to 2009? & ross kemp & Ross Kemp & 0.999 & 0.761 & 0.769 & 0.826 & 0.746 \\
Q: Whole Again and Eternal Flame were Number Ones for which girl group in 2001? & atomic kitten & Atomic Kitten & 0.999 & 0.761 & 0.769 & 0.180 & 0.026 \\
Q: During a penalty shoot out in soccer where should the non participating players be & in the centre circle & Outside of the penalty area & 0.999 & 0.985 & 0.987 & 0.987 & 0.960 \\
Q: On which game show was Bobby Charlton once a contestant and winner & double your money & A Question of Sport & 0.999 & 0.961 & 0.963 & 0.987 & 0.952 \\
Q: From 'On Her Majesty's Secret Service' (1969), as Bond passes a janitor in Draco's headquarters, the man can be heard whistling what? & the goldfinger (1964) theme & "Goldfinger" & 0.999 & 0.944 & 0.940 & 0.984 & 0.886 \\
Q: A Paris grocer was jailed for two years in 1978 stabbing wife what? & a wedge of hard cheese & Knife & 0.999 & 0.976 & 0.987 & 0.974 & 0.849 \\
\bottomrule
\end{tabular}
}
    \caption{Examples of correctness and the according uncertainty levels.}
    \label{tab:all_qual}
\end{table}

\subsection{Qualitative Illustration}
\label{app:qual}

\noindent \tcbset{colback=blue!5!white,colframe=blue!75!black,fonttitle=\bfseries,width=\textwidth,nobeforeafter}
\begin{tcolorbox}[box align=center]
{\fontfamily{qcr}\selectfont 
    \noindent $\bx$: In 1840 the world’s first postage stamps printed were the Penny Black and which other?\\
    $\by$: twopenny blue \\
    $\widehat\by$: The Penny Red \\
    $\PP(U_{\rm SE} \leq u)$: 0.825\\
    $\PP(U_{\rm NLL} \leq u)$: 0.864
    \vspace{-1.5mm}
}
\end{tcolorbox}

\noindent
\tcbset{colback=red!5!white,colframe=red!75!black,fonttitle=\bfseries,width = \textwidth}
\begin{tcolorbox}[box align=center]
{\fontfamily{qcr}\selectfont 
    \noindent $\bx$: Championship dragon boat racing calls for a specialised long boat, a team of paddlers (typically 20), a sweeper to steer and which other of these?  \\
    $\by$: a drummer and drum\\
    $\widehat \by$: A drummer \\
    $\PP(U_{\rm SE} \leq u)$: 0.946\\
    $\PP(U_{\rm NLL} \leq u)$: 0.704
    \vspace{-1.5mm}
}
\end{tcolorbox}

\noindent \tcbset{colback=blue!5!white,colframe=blue!75!black,fonttitle=\bfseries,width=\textwidth,nobeforeafter}
\begin{tcolorbox}[box align=center]
{\fontfamily{qcr}\selectfont 
    \noindent $\bx$: Who has the highest suicide rate in the UK?\\
    $\by$: men - by a ratio of roughly 4 to 1 \\
    $\widehat\by$: Middle-aged men \\
    $\PP(U_{\rm SE} \leq u)$: 0.745\\
    $\PP(U_{\rm NLL} \leq u)$: 0.894
    \vspace{-1.5mm}
}
\end{tcolorbox}

\noindent
\tcbset{colback=red!5!white,colframe=red!75!black,fonttitle=\bfseries,width = \textwidth}
\begin{tcolorbox}[box align=center]
{\fontfamily{qcr}\selectfont 
    \noindent $\bx$: Which East Midlands club holds the Football League record for most games played?  \\
    $\by$:  nots county\\
    $\widehat \by$: Notts County\\
    $\PP(U_{\rm SE} \leq u)$: 0.842\\
    $\PP(U_{\rm NLL} \leq u)$: 0.793
    \vspace{-1.5mm}
}
\end{tcolorbox}
\vspace{1mm}

We provide more instances to show the qualitative effect of our \ref{eqn:rank-ece}-based assessment in Table~\ref{tab:all_qual}.


\subsection{Recalibration with Histogram Binning}\label{app:recalibrastion}
We use equal-mass histogram binning to recalibrate, in a post-hoc manner, the performance of an uncertainty (or confidence) measure on a specific benchmark. Specifically, given a dataset $\{(u_i, a_i)\}_{i=1}^n$ of uncertainty and correctness values computed over a benchmark, where each
$u_i\hspace{-1pt}=\hspace{-1pt}U(\bx;\widehat \by_i)$, $a_i\hspace{-1pt}=\hspace{-1pt}A(\bx_i ;\widehat \by_i)$, and $\widehat \by_i$ is a response generated by the LM. Then, we first randomly split it into the calibration set $\{(u_i, a_i)\}_{i=1}^{n_{\rm cal}}$ and the test set  $\{(u_i, a_i)\}_{i=n_{\rm cal}+1}^{n}$. Similar to the operations in Sec.~\ref{sec:rce+diagram}, we partition the range of $U$ into $B$ bins $\{{\rm bin}_{b}\}_{b=1}^B$ whose boundaries are quantiles of $\{(u_i, a_i)\}_{i=n_{\rm cal}+1}^{n}$. Then, we estimate the expected correctness level over the ${\rm bin}_{b}$ as 
\begin{equation*}
    {\rm crc}_{b,{\rm cal}}:=\frac{1}{|\cI_{b,{\rm cal}}|}\sum_{i\in \cI_{b,{\rm cal}}}a_i
\end{equation*}
where $\cI_{b,{\rm cal}}\triangleq \{i:1\leq i\leq n_{\rm cal}, u_i\in {\rm bin}_{b}\}$. 
We re-calibrate the measure $U$, defining $U_{\rm cal}$ via $U_{\rm cal}(\bx;\widehat \by)={\rm crc}_{b,{\rm cal}}$ for any $U(\bx;\widehat \by)\in {\rm bin}_{b}$. 
We split the benchmark data equally into calibration and test sets and evaluate the performance of the calibrated measure on the test set.  
Table~\ref{tab:all-SE-GPT-calibrated} and Fig.~\ref{fig:gpt_1.0_cal_meteor} and~\ref{fig:gpt_1.5_cal_triviaqa} list the RCE results of $U_{\rm SE}$ for GPT-3.5-turbo before and after calibration. 
We observe the calibrated measure is significantly better rank-calibrated, showing the effectiveness of this strategy.

\begin{figure*}[t]
\centering
\begin{subfigure}{.5\textwidth}
  \centering
  \includegraphics[width=0.5\linewidth]{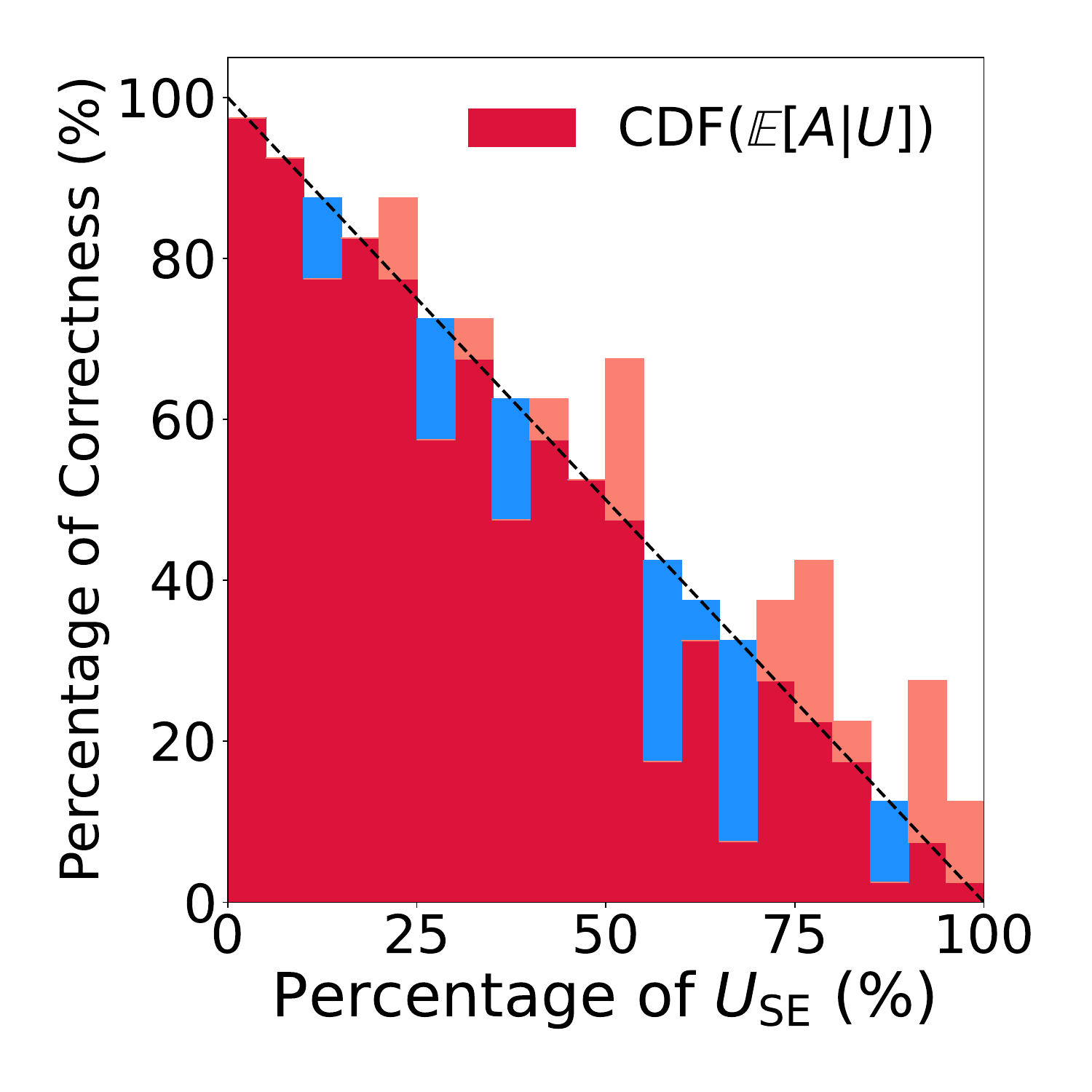}
  \hspace{-4mm}
  \includegraphics[width=0.5\linewidth]{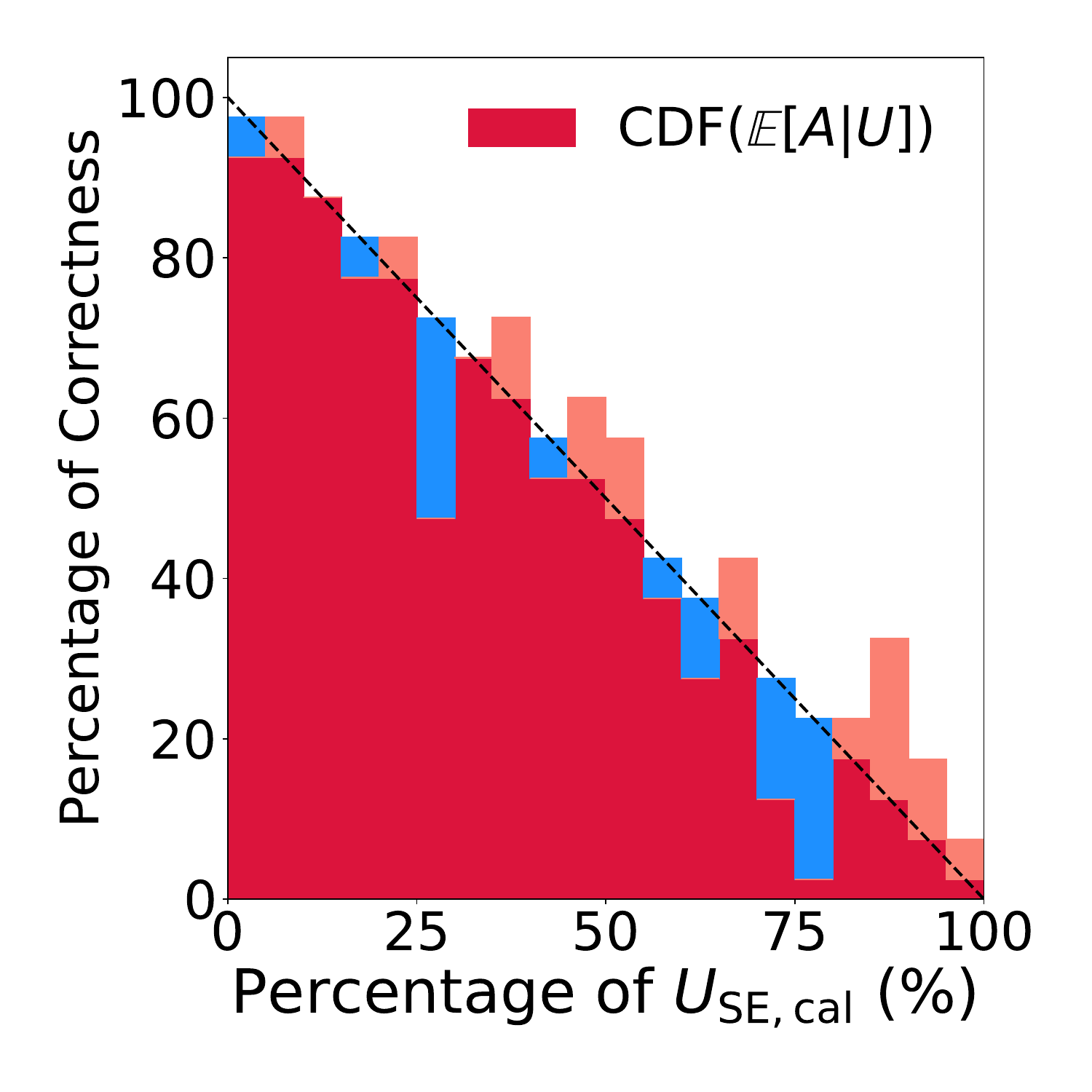}
  \caption{Meadow}
  \label{fig:rce_se_cal_meadow}
\end{subfigure}%
\begin{subfigure}{.5\textwidth}
  \centering
  \includegraphics[width=0.5\linewidth]{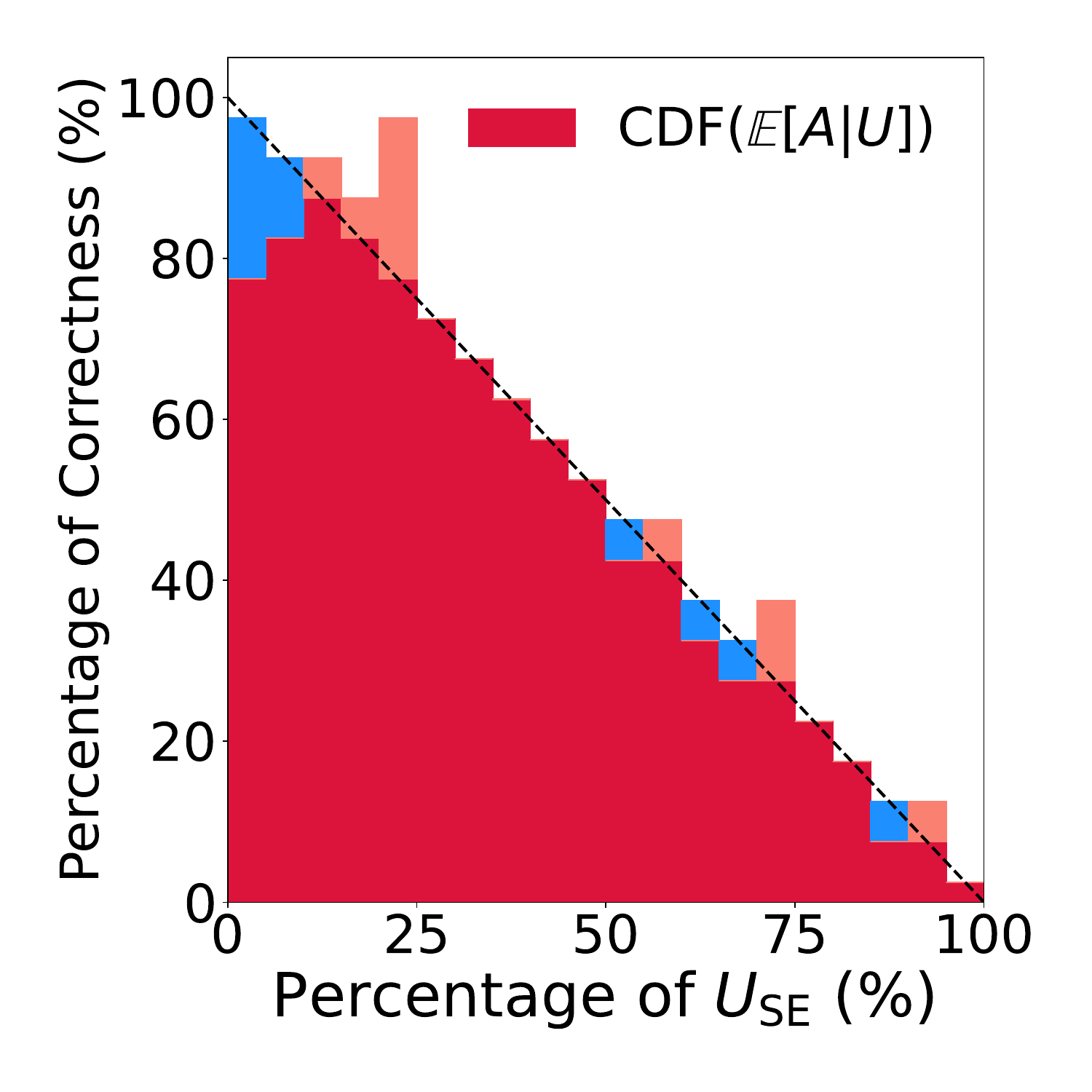}
  \hspace{-4mm}
  \includegraphics[width=0.5\linewidth]{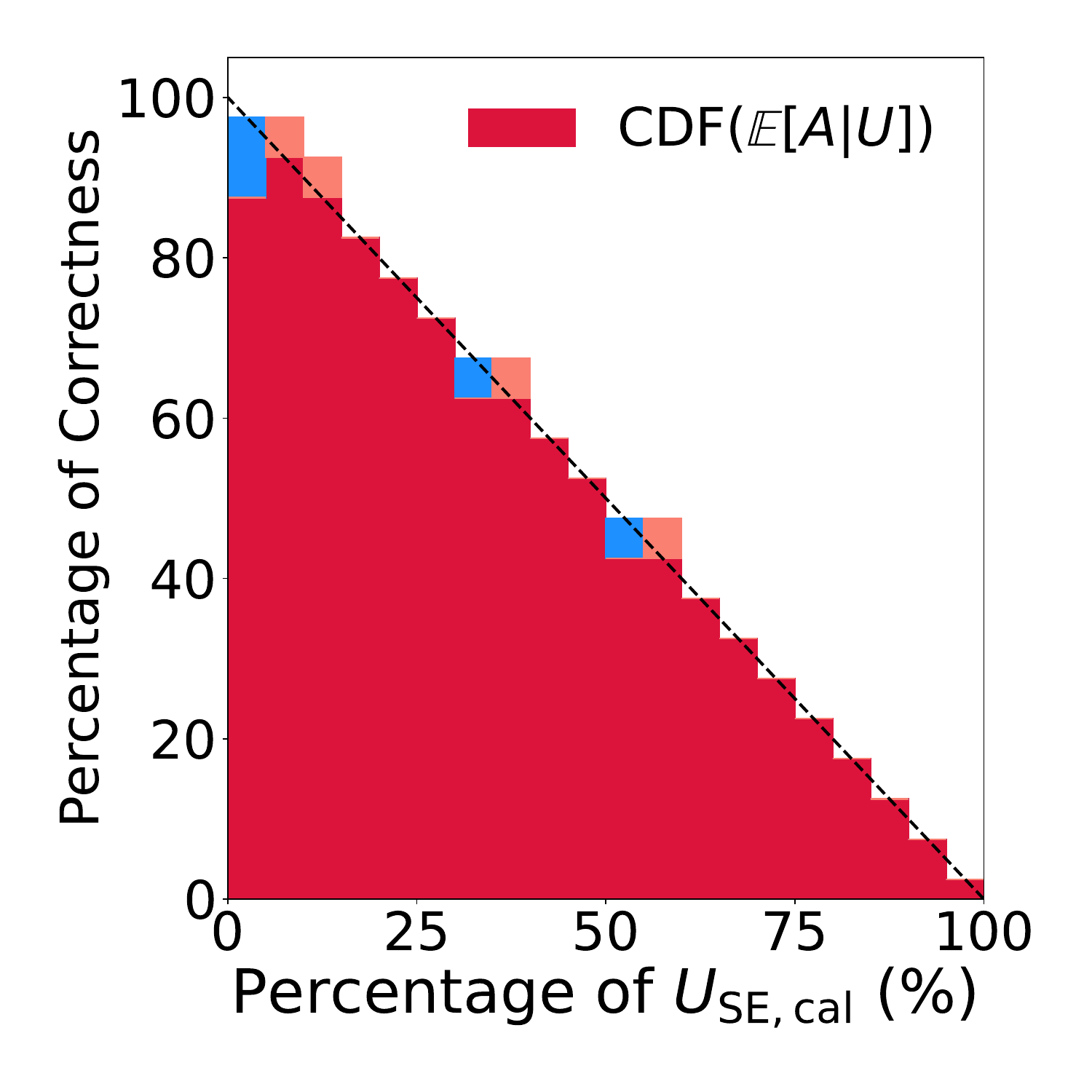}
  \caption{NQ-Open}
  \label{fig:rce_se_cal_nq-open}
\end{subfigure}\\
\begin{subfigure}{.5\textwidth}
  \centering
  \includegraphics[width=0.5\linewidth]{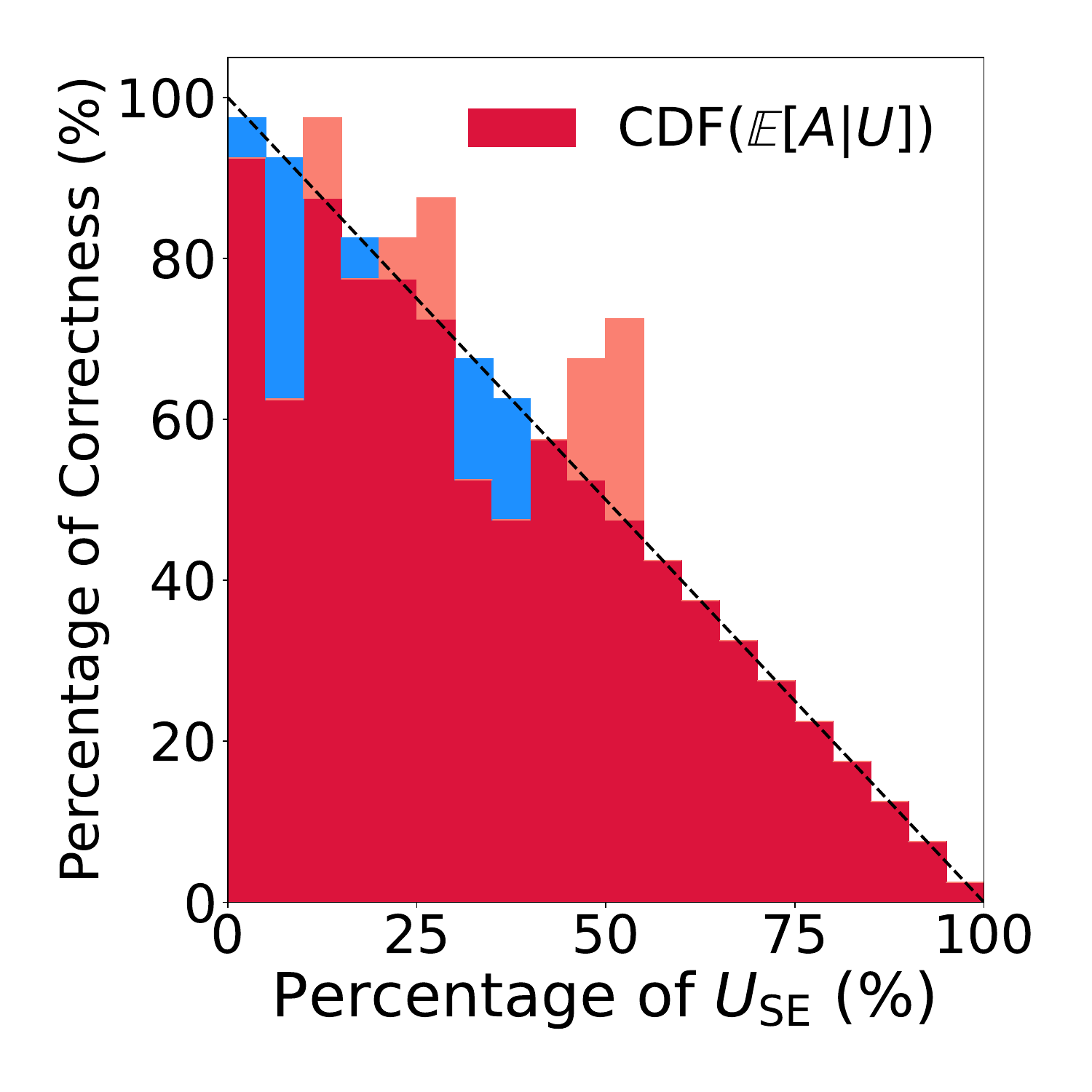}
  \hspace{-4mm}
  \includegraphics[width=0.5\linewidth]{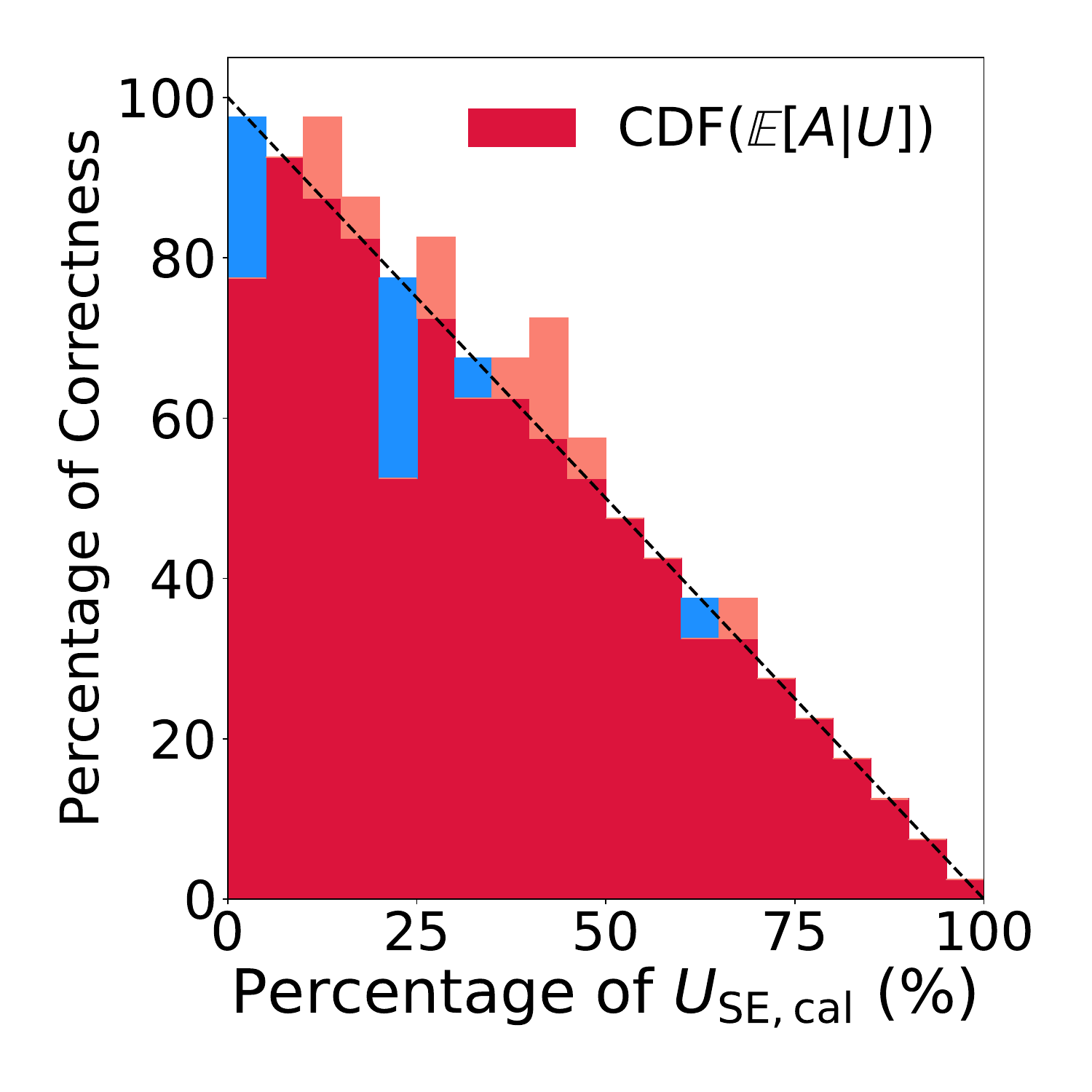}
  \caption{Squad}
  \label{fig:rce_se_cal_squad}
\end{subfigure}
\hspace{-2mm}
\begin{subfigure}{.5\textwidth}
  \centering
  \includegraphics[width=0.5\linewidth]{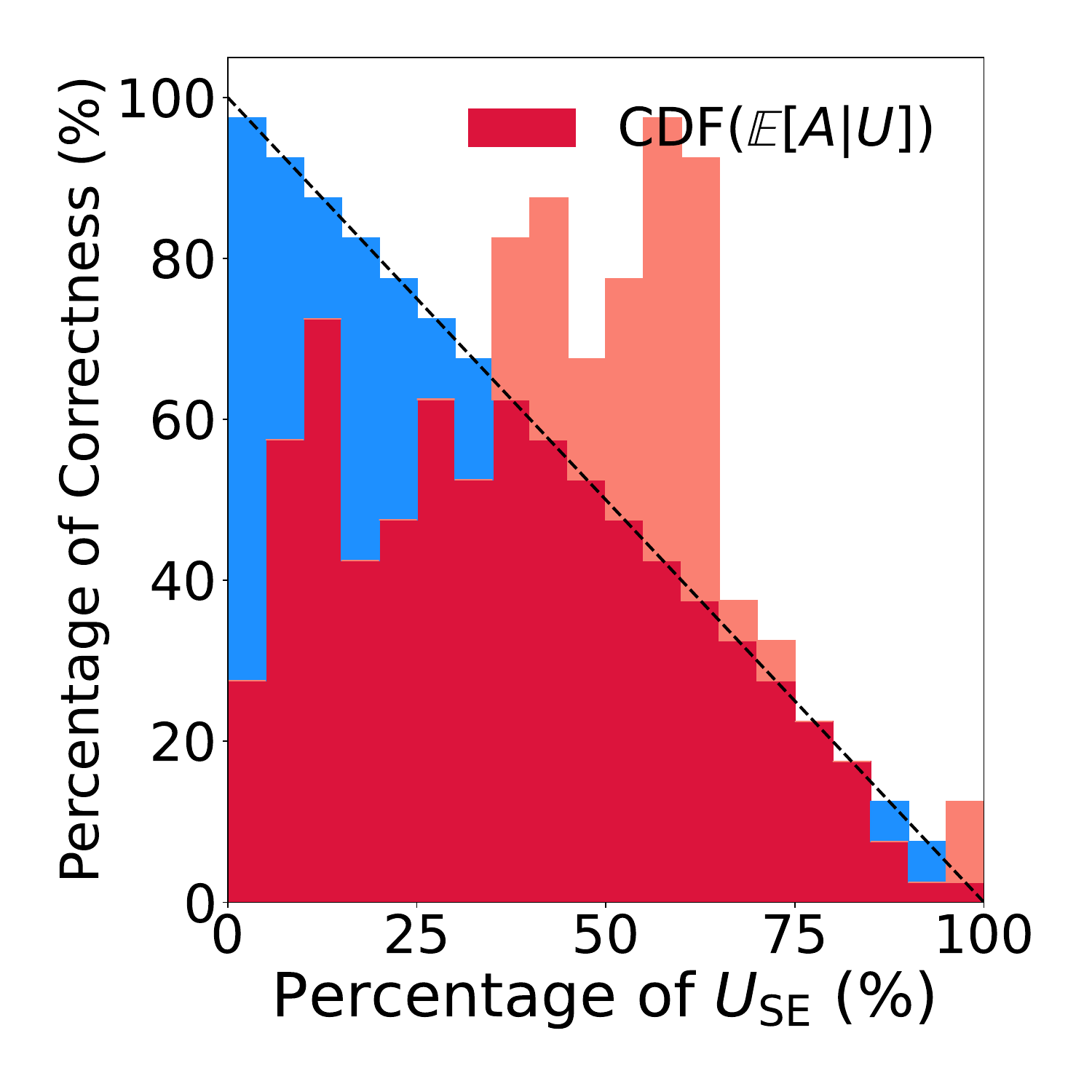}
  \hspace{-5mm}
  \includegraphics[width=0.5\linewidth]{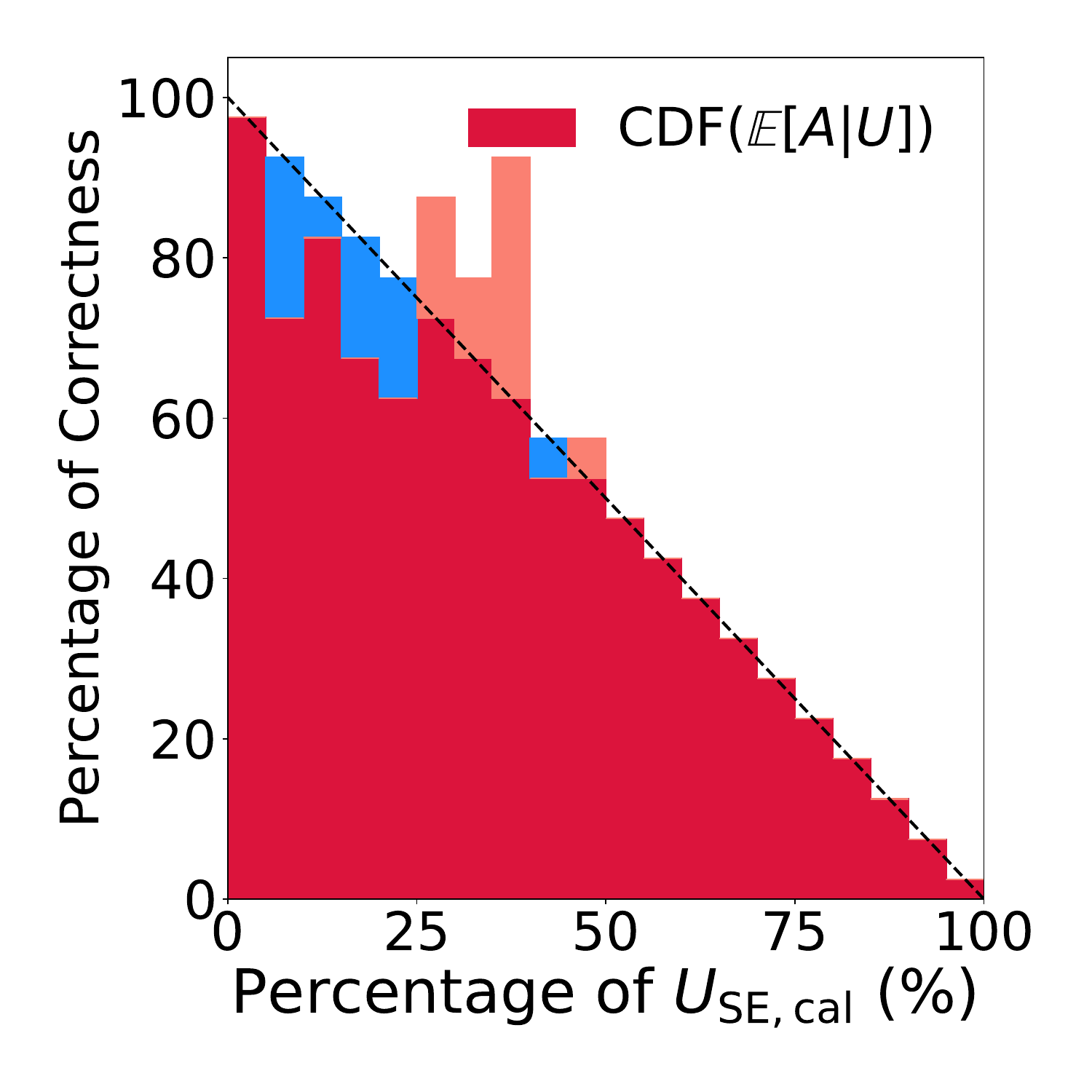}
  \caption{TrviaQA}
  \label{fig:rce_se_cal_triviaqa}
\end{subfigure}
\caption{
Indication diagrams of $U_{\rm SE}$ and $U_{\rm SE, cal}$ (post-calibrated) for GPT-3.5-turbo (temperature 1.0) on various benchmarks with the Meteor correctness.}
\label{fig:gpt_1.0_cal_meteor}
\vspace{-4mm}
\end{figure*}

\begin{figure*}
\centering
\begin{subfigure}{.5\textwidth}
  \centering
  \includegraphics[width=0.5\linewidth]{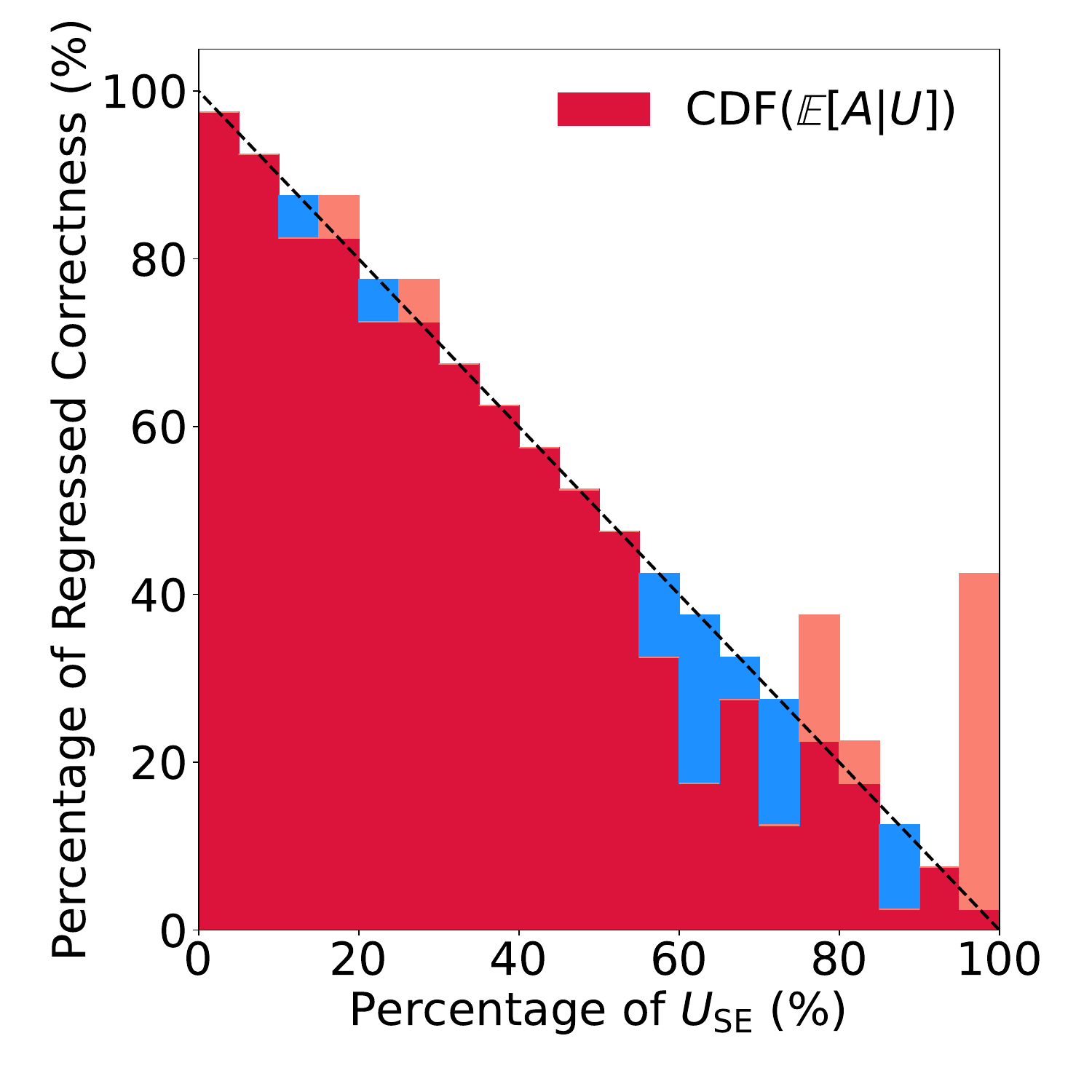}
  \hspace{-4mm}
  \includegraphics[width=0.5\linewidth]{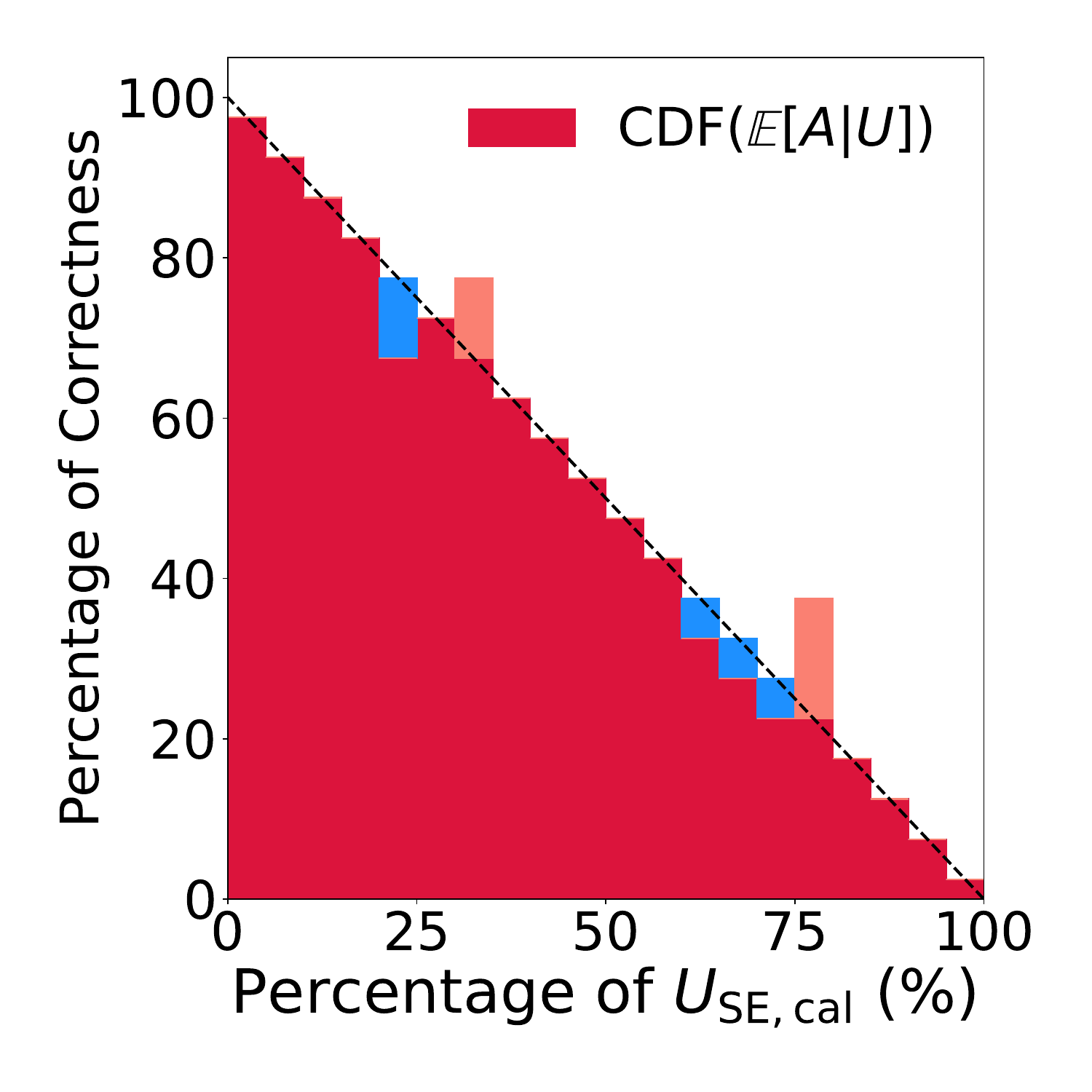}
  \caption{Bert Similarity}
\end{subfigure}%
\begin{subfigure}{.5\textwidth}
  \centering
  \includegraphics[width=0.5\linewidth]{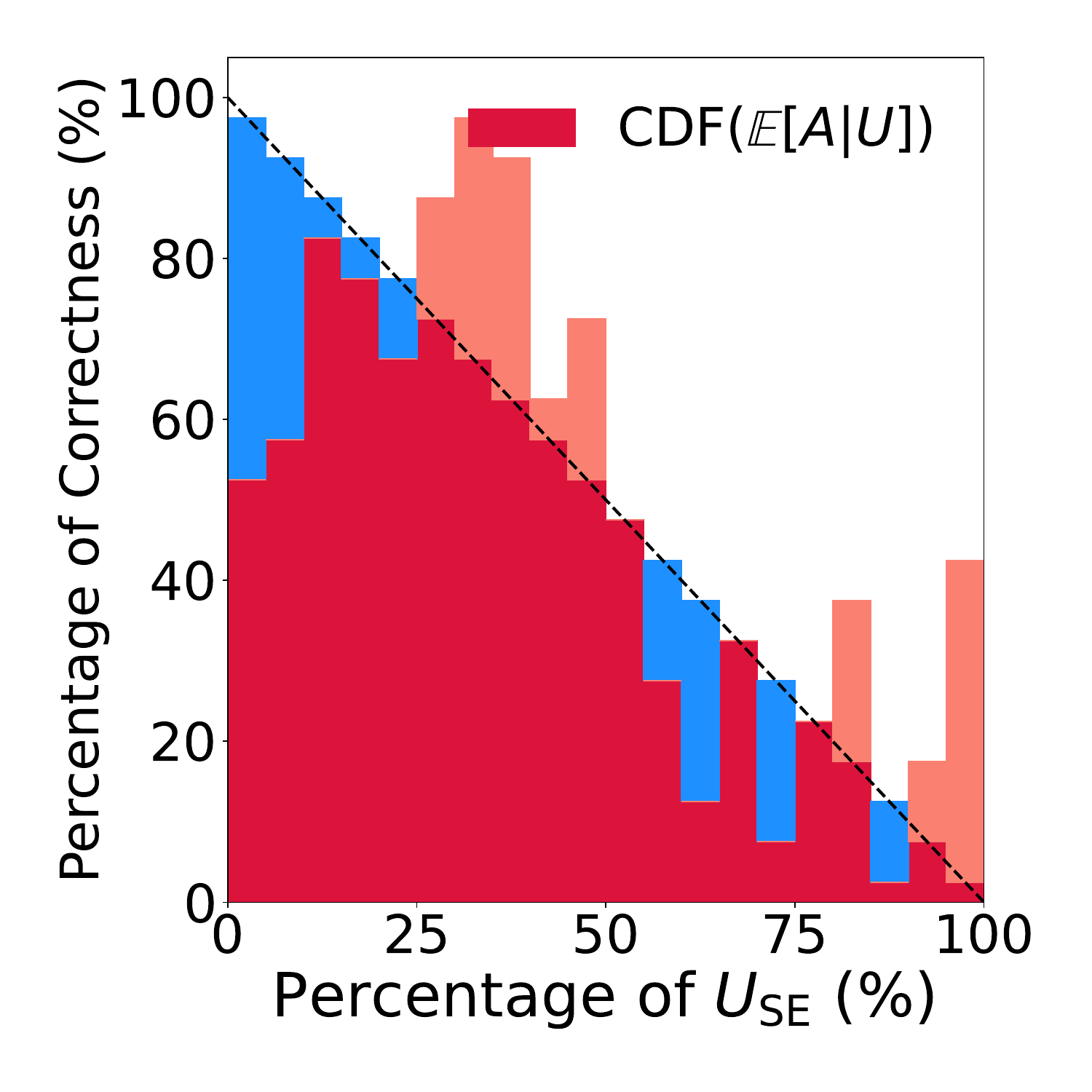}
  \hspace{-4mm}
  \includegraphics[width=0.5\linewidth]{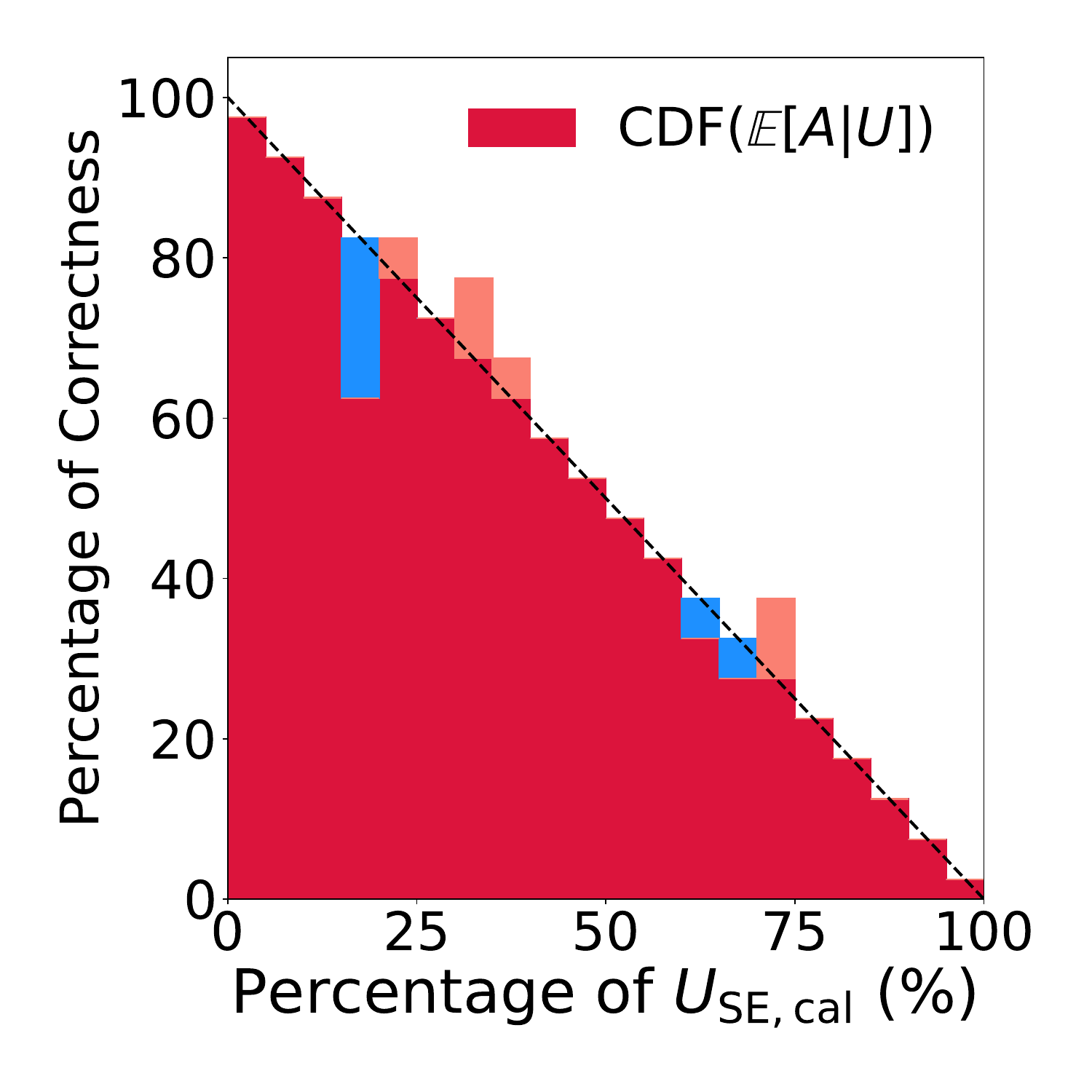}
  \caption{Meteor Score}
\end{subfigure}\\
\begin{subfigure}{.5\textwidth}
  \centering
  \includegraphics[width=0.5\linewidth]{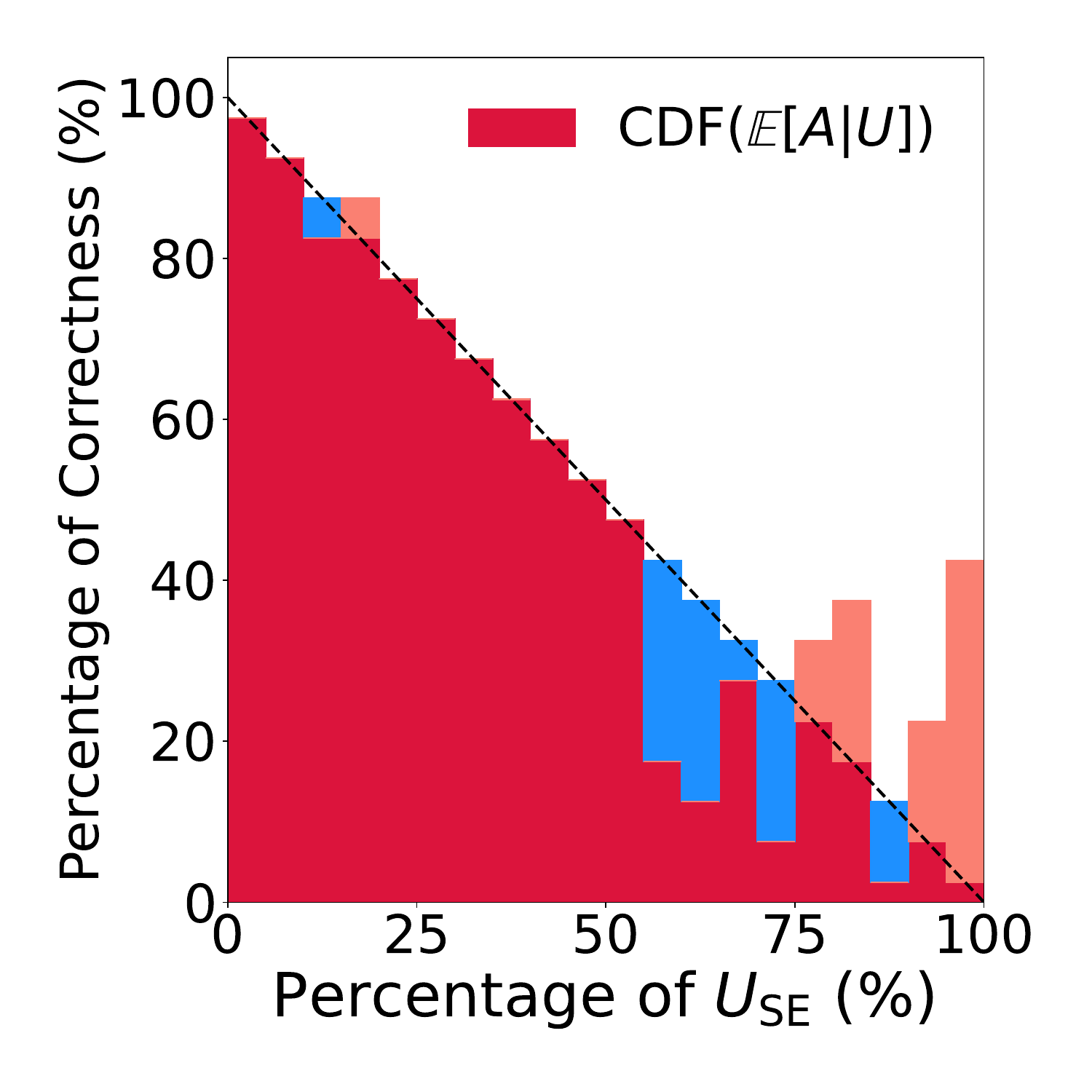}
  \hspace{-4mm}
  \includegraphics[width=0.5\linewidth]{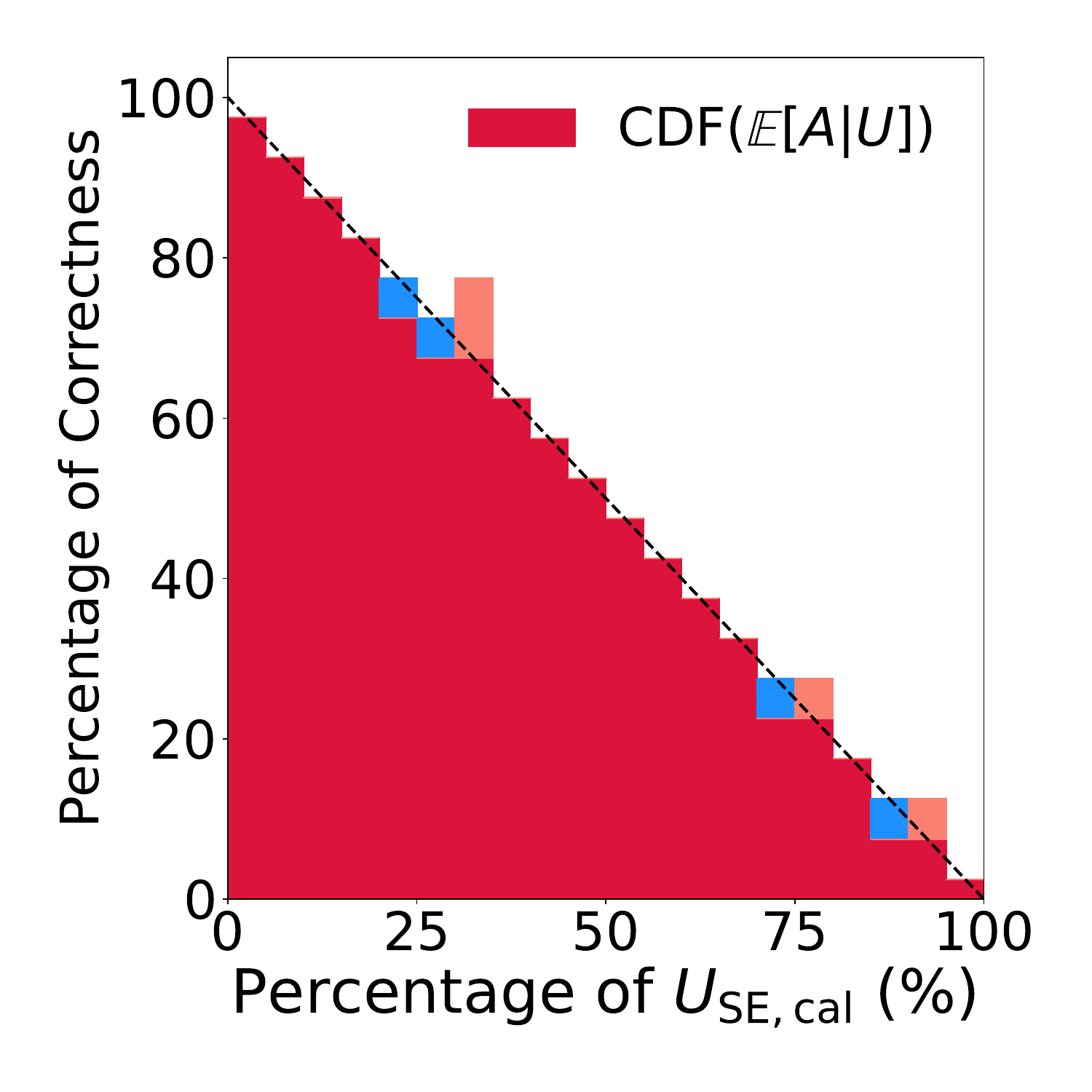}
  \caption{Rouge Score}
\end{subfigure}
\hspace{-2mm}
\begin{subfigure}{.5\textwidth}
  \centering
  \includegraphics[width=0.5\linewidth]{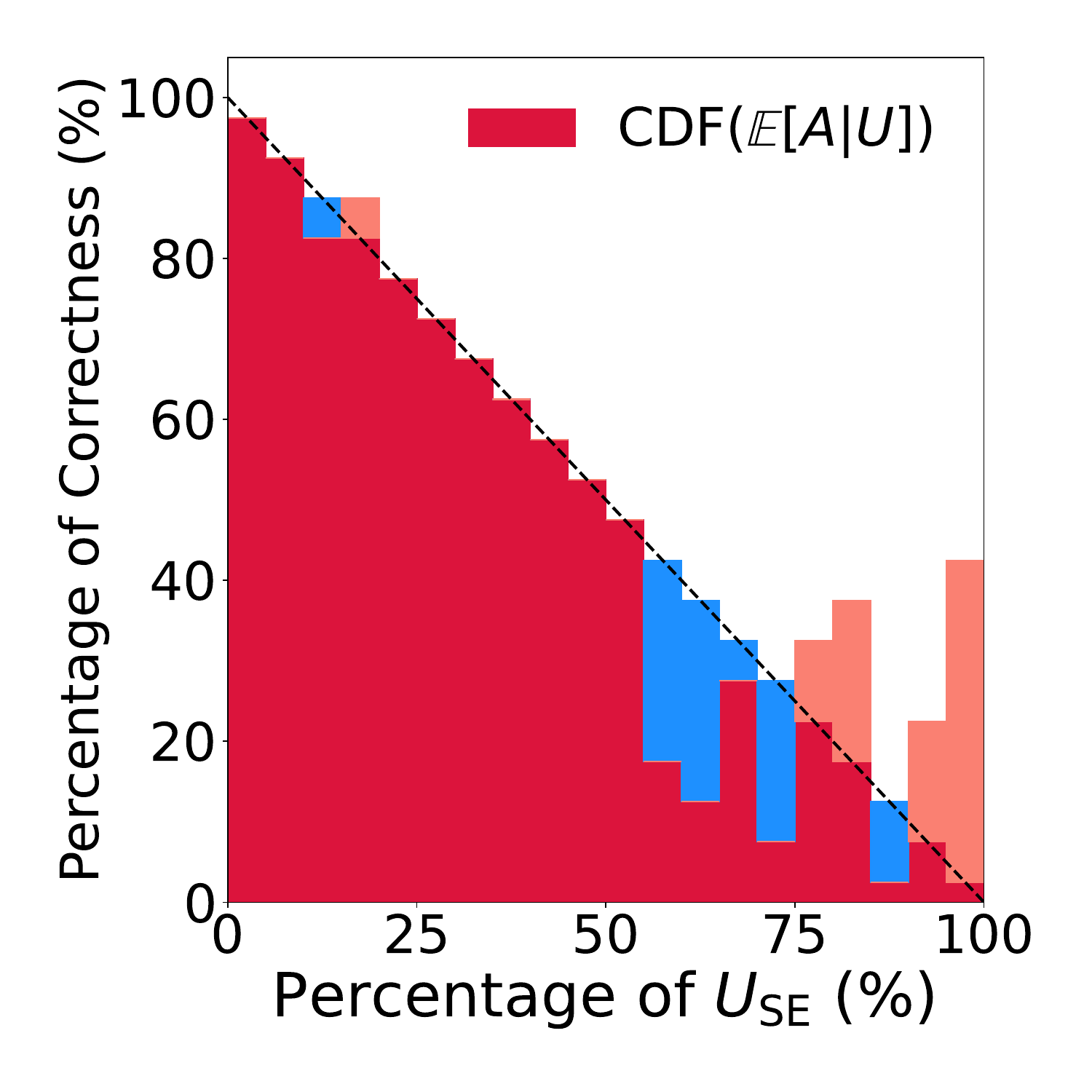}
  \hspace{-5mm}
  \includegraphics[width=0.5\linewidth]{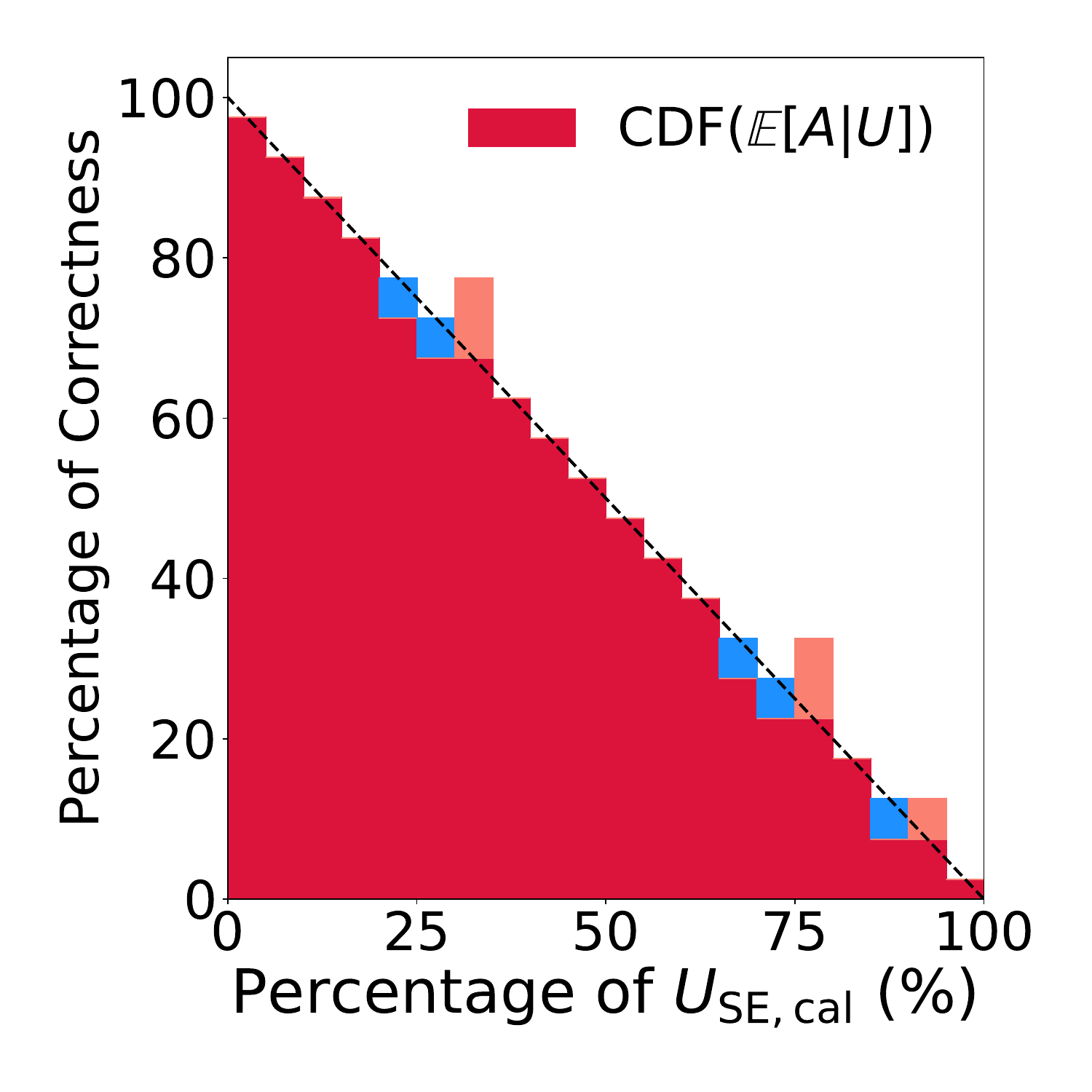}
  \caption{Rouge1 Score}
\end{subfigure}
\caption{
Indication diagrams of $U_{\rm SE}$ and $U_{\rm SE, cal}$ (post-calibrated) for GPT-3.5-turbo (temperature 1.5) on TriviaQA with various correctness scores.}
\label{fig:gpt_1.5_cal_triviaqa}
\end{figure*}

While effective, one should note that such a post-hoc recalibration strategy concerns a specific benchmark and is not a focus of our work. We leave devising benchmark-agnostic calibrated uncertainty/confidence measures for future work.

\subsection{Critical Difference Diagrams}\label{app:cd_diagram}
Here, we propose to combine the RCE metric with the \textit{critical difference} (CD) diagram~\citep{demvsar2006statistical}. 
Critical Difference diagrams are built on the Wilcoxon signed rank test and the Friedman test, giving a non-parametric comparison of multiple approaches aggregated
over several trials.

\begin{figure}[H]
  \centering\includegraphics[width=0.8\textwidth]{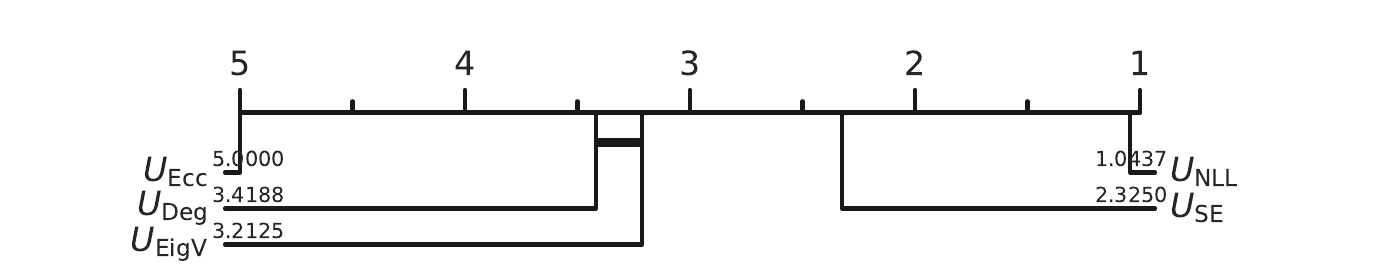}
\caption{CD diagram of Llama-2-chat on TriviaQA.}
\label{fig:cd_triviaqa}
\end{figure}

As a demonstration, the CD diagram of assessed measures for Llama-2-chat on TriviaQA is shown in Fig.~\ref{fig:cd_triviaqa}. 
The positions of various methods represent their averaged ranks over various experimental configurations (\eg, temperature, LM, bootstrap, etc), 
where a lower averaged rank indicates that the corresponding measure (\eg, $1.04$ for $U_{\rm NLL}$) performs better than others in an averaged sense. 
Here, a thick horizontal segment  connects measures (\eg, $U_{\rm Deg}$ and $U_{\rm EigV}$) 
if the difference between their averaged ranks is within the critical length determined by related hypothesis testing procedures. 
Measures that are disconnected (\eg, $U_{\rm Ecc}$, $U_{\rm Deg}$, and $U_{\rm NLL}$) 
have statistically significant differences in performance.

\subsection{Robustness Analysis}
\label{sec:robustness}
The RCE of uncertainty measures in practice may be affected by several factors. Therefore, we conduct ablation studies to analyze whether RCE is robust to two crucial key factors: correctness scores and model temperatures.

\paragraph{Correctness functions.}
We show RCEs for various models and correctness scores on TriviaQA and SQuAD in Fig~\ref{fig:robust_correct}. Each result is obtained using bootstrapping with 20 fixed seeds. 
We observe that the ranking of uncertainty measures is robust to correctness scores. 
For instance, we show the critical diagrams using GPT-3.5 on TriviaQA with varying correctness scores in Fig~\ref{fig:cds_correct}. In this setting, $U_{\rm NLL}$, $U_{\rm SE}$ and $C_{\rm Verb}$ rank consistently higher across different correctness scores. Second, as shown in Table~\ref{tab:all},
RCE values using different correctness scores are relatively stable. 
For instance, when using GPT-3.5 on TriviaQA, the RCE values of NLL are 0.065, 0.054, 0.037, and 0.039 with bert\_similarity, meteor, rouge-L, and rouge-1 scores, which are close. 

\begin{figure}
    \centering
    \includegraphics[width=0.8\textwidth]{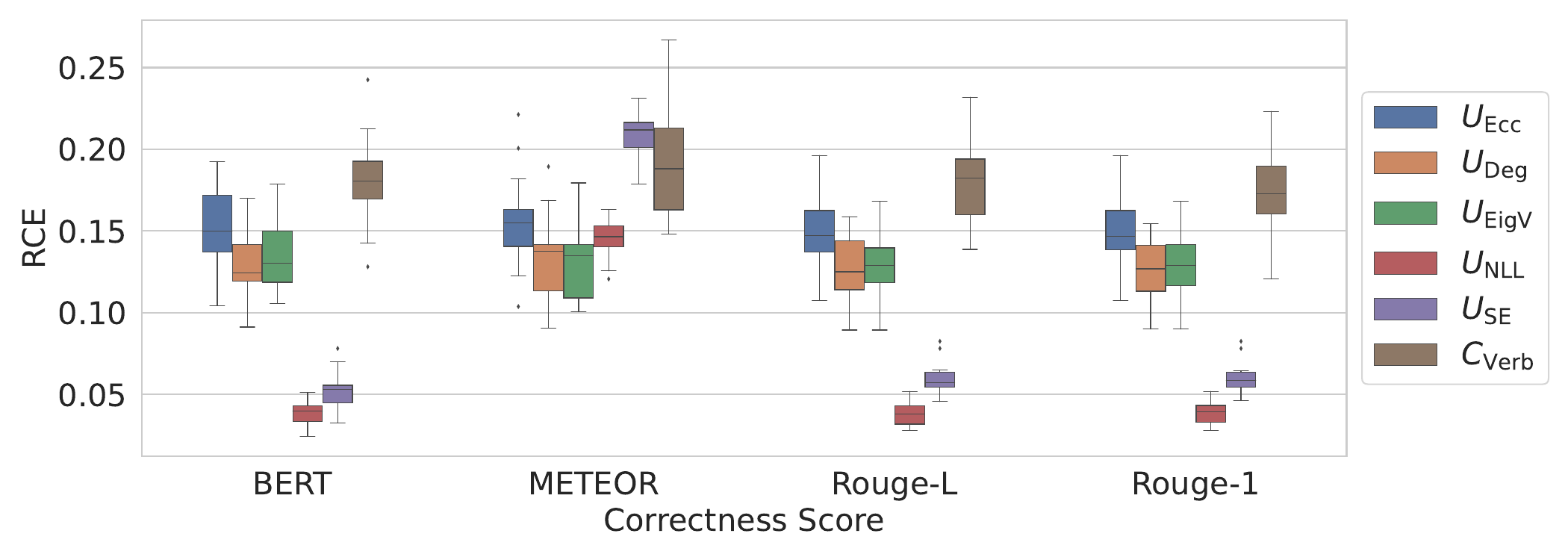}
    \includegraphics[width=0.8\textwidth]{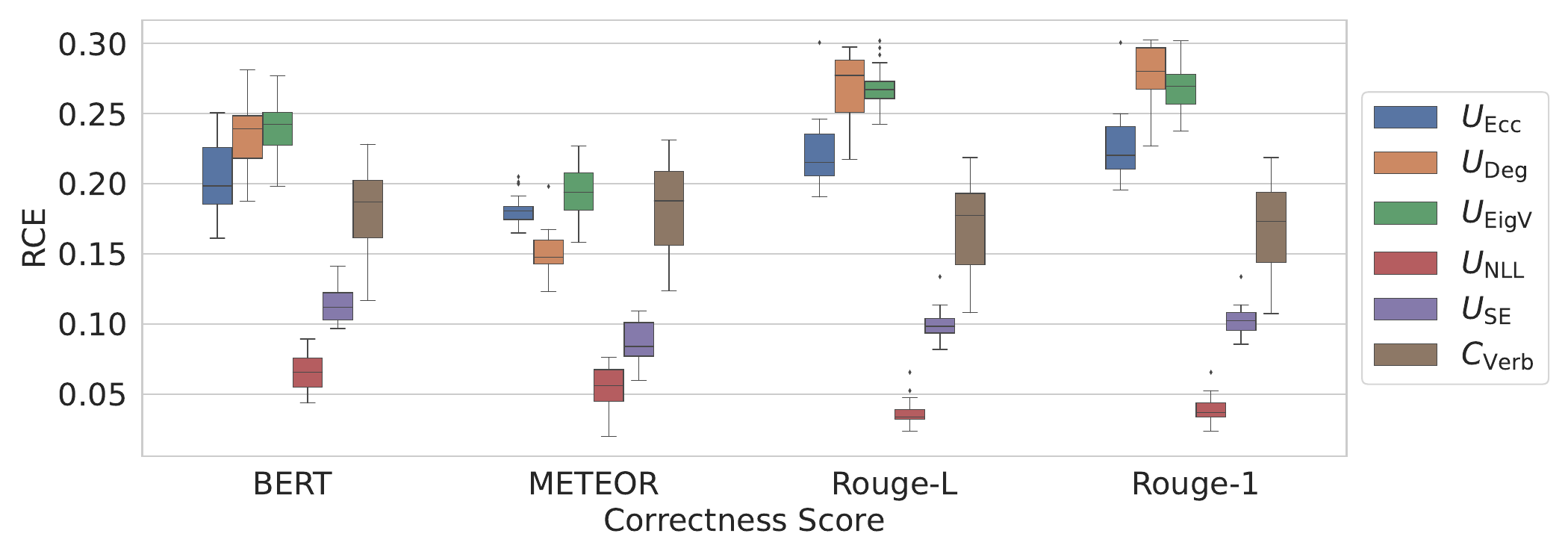}
    \includegraphics[width=0.8\textwidth]{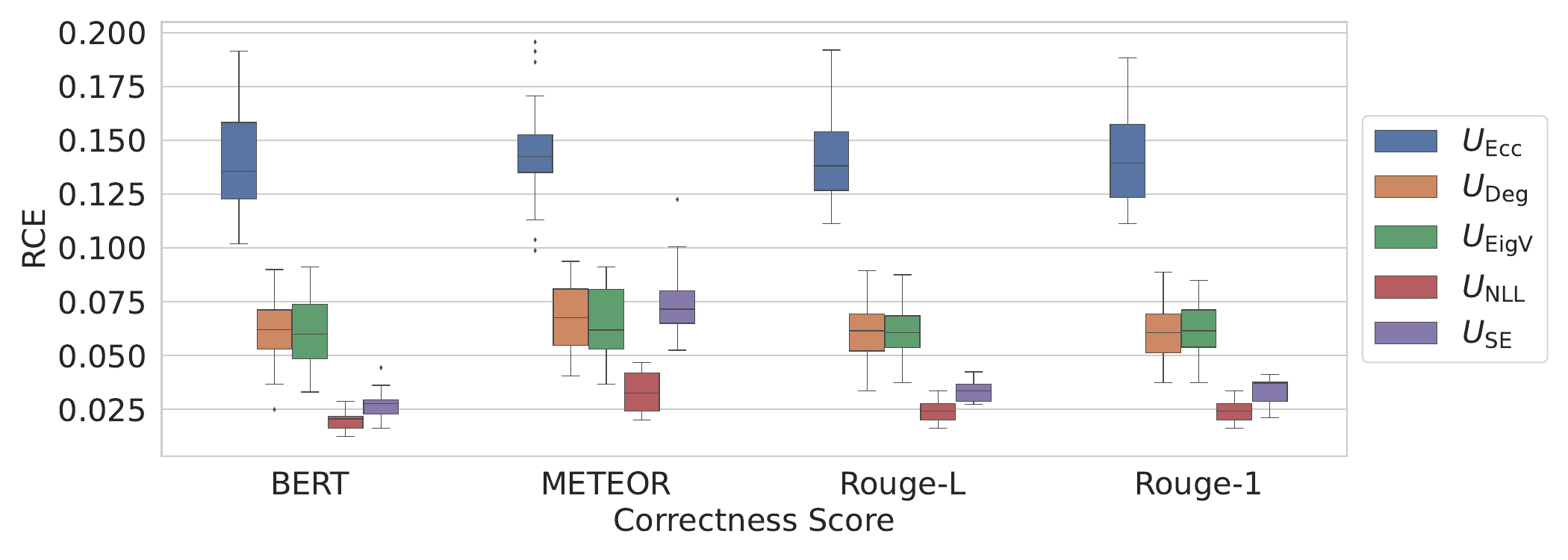}
    \includegraphics[width=0.8\textwidth]{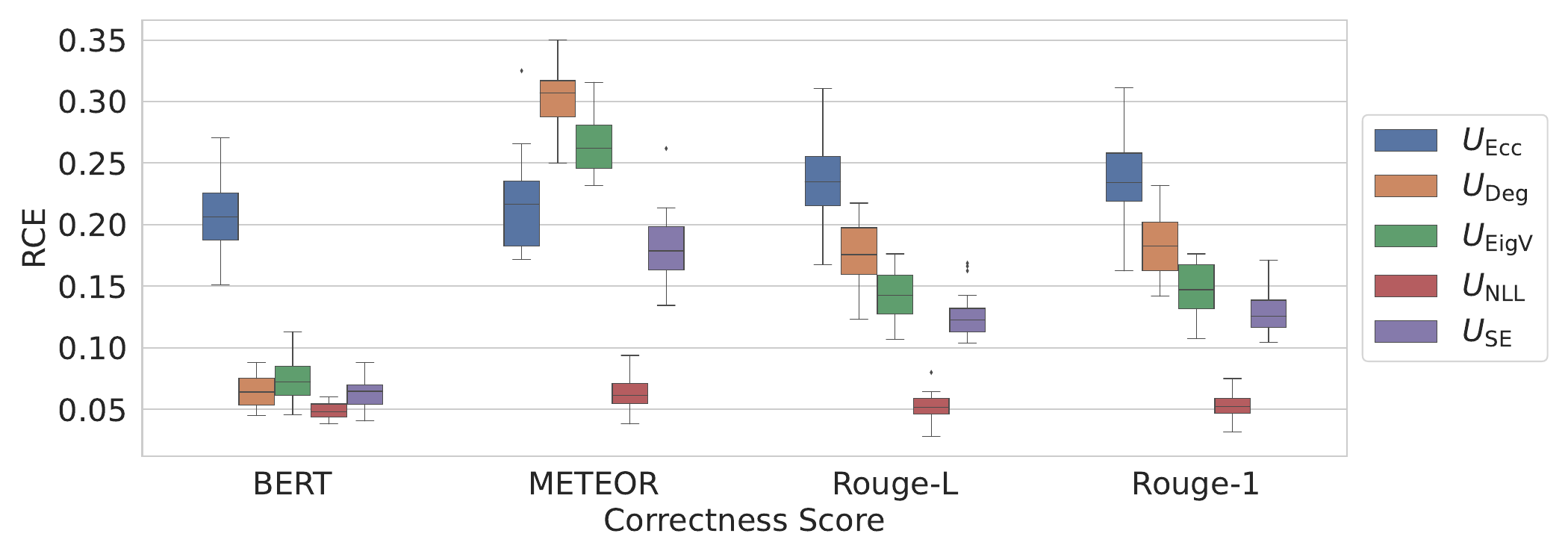}
    \caption{Box plots with various correctness functions under various configurations. The first row is for GPT-3.5-turbo on TriviaQA; the second row is for GPT-3.5-turbo on SQuAD; the third is for Llama-2-7b-chat on TriviaQA; and the fourth row is for Llama-2-7b-chat on SQuAD. 
    }
    \label{fig:robust_correct}
\end{figure}

\begin{figure}
    \centering
    \includegraphics[width=0.8\textwidth]{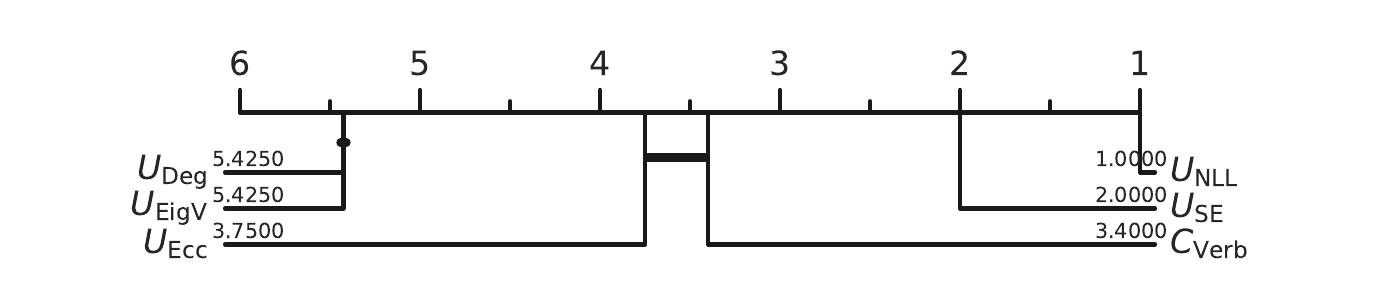}
    \includegraphics[width=0.8\textwidth]{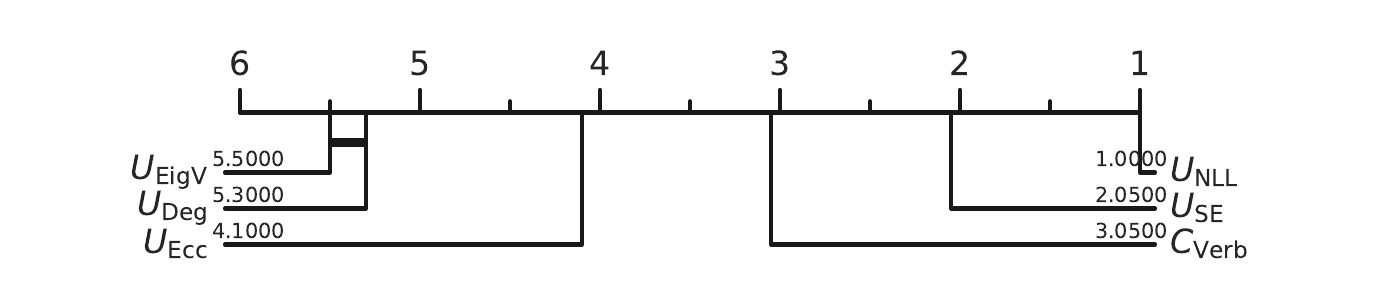}
    \includegraphics[width=0.8\textwidth]{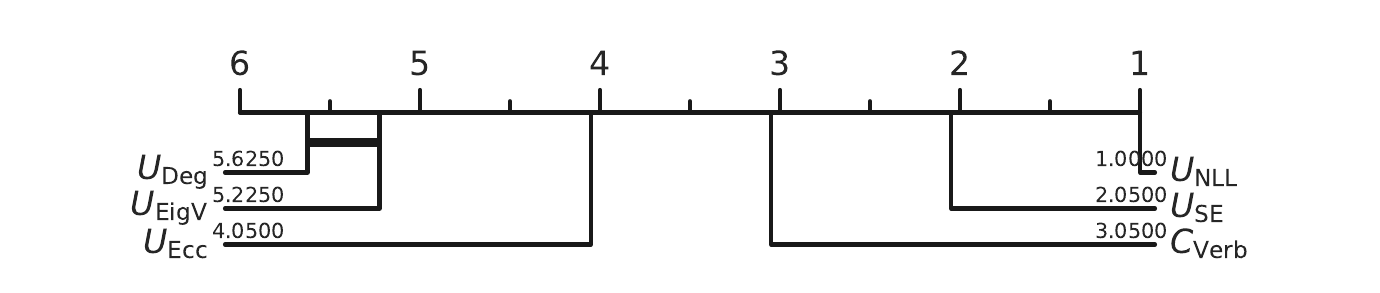}
    \includegraphics[width=0.8\textwidth]{figures/cds/gpt_correctness_squad/gpt-3.5-turbo_bert_similarity.pdf}
    \caption{CD diagrams using GPT-3.5 on TriviaQA with different correctness scores.}
    \label{fig:cds_correct}
\end{figure}

\paragraph{Temperature setting.}
We show the RCEs for various models and temperatures on TriviaQA and SQuAD in Fig.~\ref{fig:robust_temperature}. As above, each result is obtained using bootstrapping with 20 fixed seeds. The findings are similar to those regarding correctness scores. First, as shown in Fig.~\ref{fig:cds_temperature}, while RCE values are not constant, $U_{\rm NLL}$ ranks consistently highest across different temperatures. When only the best uncertainty measure is considered, the RCE rankings at different temperatures give consistent results. Second, the RCE values are stable across different temperatures. For instance, when using GPT-3.5 with the Rouge-L score, the RCE values are 0.041, 0.038, 0.034 with temperatures 0.5, 1.0, and 1.5.

\begin{figure}
    \centering
    \includegraphics[width=0.8\textwidth]{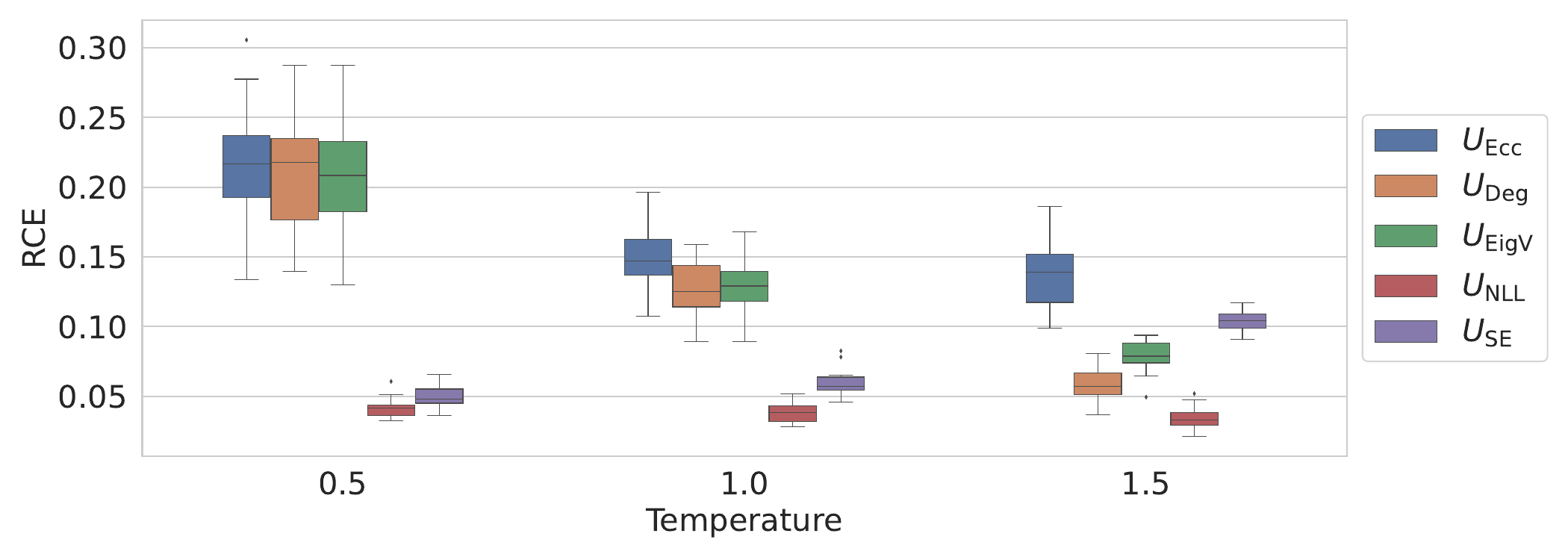}
    \includegraphics[width=0.8\textwidth]{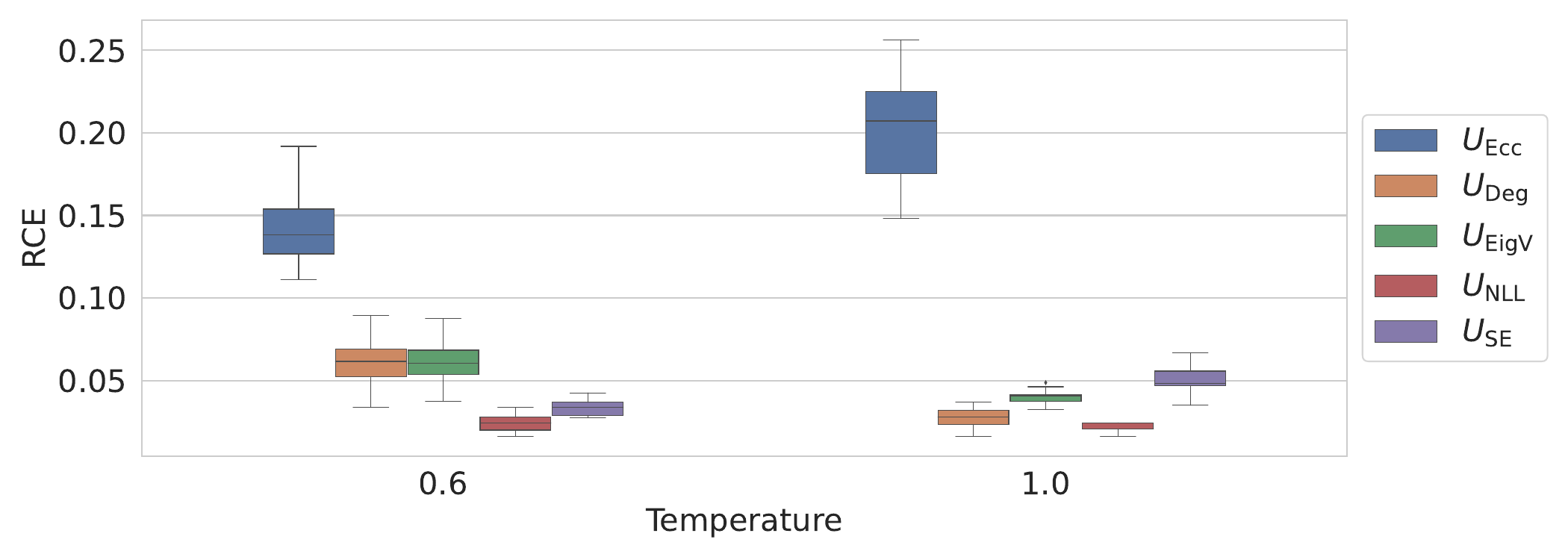}
    \caption{Box plots based on the generations of GPT-3.5-turbo and Llama-2-7b-chat with varying temperatures. The first row represents GPT-3.5-turbo with temperatures 0.5, 1.0, and 1.5, while the second row represents Llama-2-7b-chat with temperatures 0.6 and 1.0. Both results are evaluated on TriviaQA dataset.}
    \label{fig:robust_temperature}
\end{figure}

\begin{figure}
    \centering
    \includegraphics[width=0.8\textwidth]{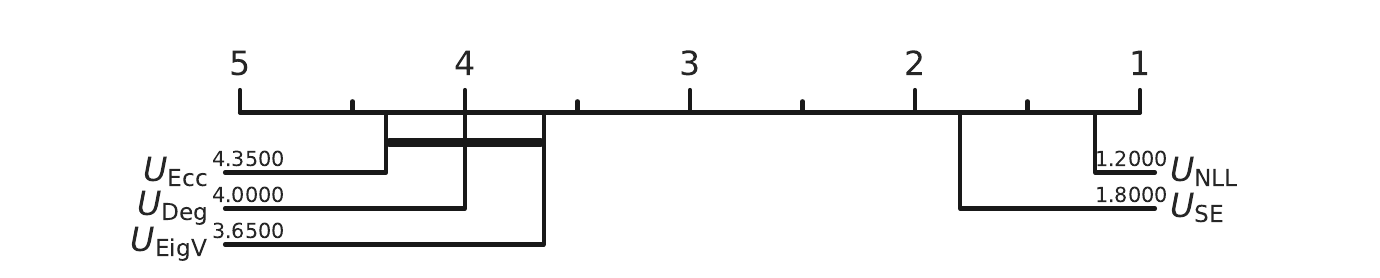}
    \includegraphics[width=0.8\textwidth]{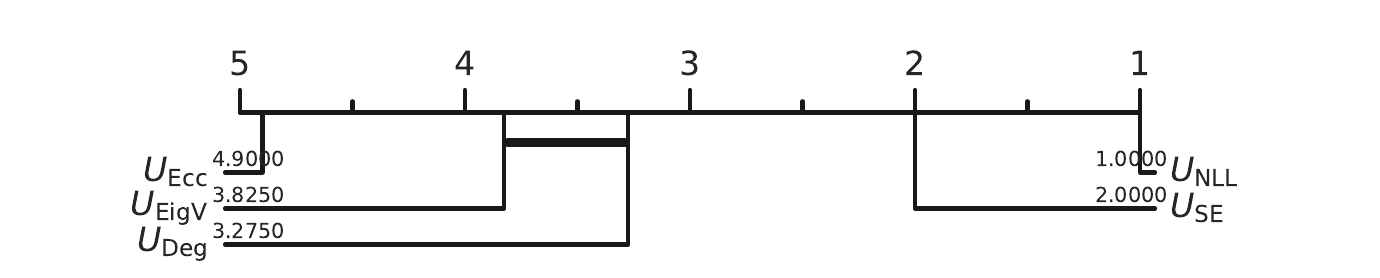}
    \includegraphics[width=0.8\textwidth]{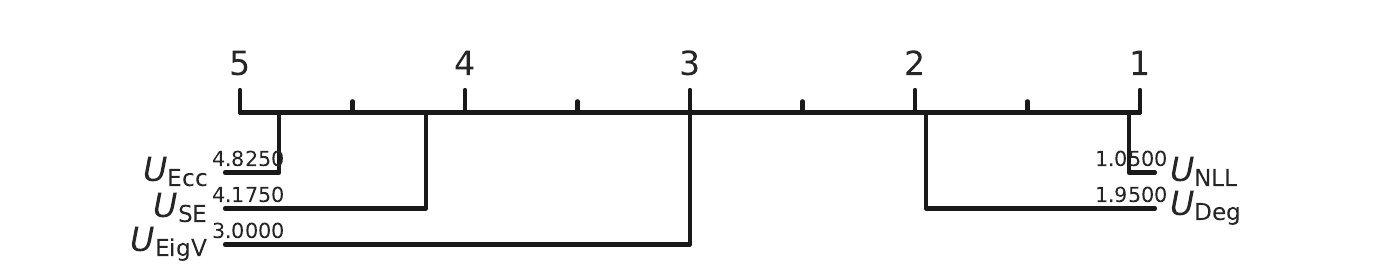}
    \caption{CD diagrams on using GPT-3.5 TriviaQA with temperature 0.5, 1.0, and 1.5.}
    \label{fig:cds_temperature}
\end{figure}

\paragraph{Influence of sample size.}
We show that the empirical RCE is robust regarding the influence of sample size, which is crucial in scenarios where labeled data is hard to acquire. To this end, we conducted a new experiment using less data in the RCE computation, simulating scenarios where only a small amount of labeled data can is available. Specifically, we utilize 20\%, 40\%, 80\%, 100\% of the TriviaQA dataset in computing the empirical RCE values of uncertainty/confidence measures for the GPT-3.5 model with temperature 1.0. The RCE results under the Bert-similarity and RougeL correctness are in Table~\ref{tab:small-dataset}. The binning scheme is the same as the one used in the paper (\ie, 20 equal-mass bins). From the new experimental results, we observe that the RCE results are fairly stable, up to reasonable standard deviations (denoted by the subscript numbers), for moderately large datasets.

\begin{table}
    \centering
    \small
\begin{tabular}{llllllll} \toprule
Proportion  & Correctness & $U_{\rm Ecc}$ & $U_{\rm Deg}$ & $U_{\rm EigV}$ & $U_{\rm NLL}$ & $U_{\rm SE}$ & $C_{\rm Verb}$ \\ \midrule
\multirow[c]{4}{*}{bert} & 20\% & 0.176$_{ \pm 0.022}$ & 0.153$_{ \pm 0.023}$ & 0.152$_{ \pm 0.024}$ & 0.058$_{ \pm 0.009}$ & 0.080$_{ \pm 0.015}$ & 0.254$_{ \pm 0.042}$ \\ \cline{2-8}
  & 40\% & 0.171$_{ \pm 0.020}$ & 0.151$_{ \pm 0.021}$ & 0.154$_{ \pm 0.020}$ & 0.048$_{ \pm 0.010}$ & 0.083$_{ \pm 0.013}$ & 0.211$_{ \pm 0.045}$ \\ \cline{2-8}
  & 80\% & 0.162$_{ \pm 0.022}$ & 0.153$_{ \pm 0.016}$ & 0.151$_{ \pm 0.017}$ & 0.043$_{ \pm 0.010}$ & 0.062$_{ \pm 0.012}$ & 0.203$_{ \pm 0.031}$ \\ \cline{2-8}
  & 100\% & 0.152$_{ \pm 0.025}$ & 0.129$_{ \pm 0.020}$ & 0.133$_{ \pm 0.020}$ & 0.039$_{ \pm 0.007}$ & 0.052$_{ \pm 0.012}$ & 0.182$_{ \pm 0.025}$ \\ \cline{1-8}
  \multirow[c]{4}{*}{rougeL} & 20\% & 0.178$_{ \pm 0.020}$ & 0.153$_{ \pm 0.024}$ & 0.153$_{ \pm 0.023}$ & 0.061$_{ \pm 0.010}$ & 0.098$_{ \pm 0.016}$ & 0.238$_{ \pm 0.035}$ \\ \cline{2-8}
  & 40\% & 0.172$_{ \pm 0.022}$ & 0.153$_{ \pm 0.021}$ & 0.156$_{ \pm 0.017}$ & 0.048$_{ \pm 0.009}$ & 0.090$_{ \pm 0.010}$ & 0.194$_{ \pm 0.040}$ \\ \cline{2-8}
  & 80\% & 0.156$_{ \pm 0.020}$ & 0.145$_{ \pm 0.017}$ & 0.146$_{ \pm 0.017}$ & 0.042$_{ \pm 0.009}$ & 0.073$_{ \pm 0.013}$ & 0.190$_{ \pm 0.030}$ \\ \cline{2-8}
  & 100\% & 0.151$_{ \pm 0.024}$ & 0.126$_{ \pm 0.019}$ & 0.129$_{ \pm 0.019}$ & 0.038$_{ \pm 0.007}$ & 0.059$_{ \pm 0.009}$ & 0.181$_{ \pm 0.026}$ \\
  \bottomrule
\end{tabular}
\caption{RCE results for GPT-3.5-turbo (temperature 1.0) performing on the TriviaQA data with various dataset sizes under the Bert-similarity and RougeL correctness.}
\vspace{-4mm}
    \label{tab:small-dataset}
\end{table}

\subsection{Conclusive Comparison}
While the RCE values and rankings are often stable when correctness score and temperature vary, there are exceptional situations where uncertainty measures rankings might fluctuate. This poses a challenge when aiming for conclusive comparisons for uncertainty measures across varying hyperparameter situations. 
To make conclusive comparisons aiming to identify 
a best method, 
we can use CD diagrams by taking multiple hyperparameter choices into account. For example, to draw conclusions agnostic to model temperature, we plot CD diagrams that show RCE rankings averaged from data collected at different temperatures, as shown in Fig.~\ref{fig:robust_comparison}. 
Based on these results, 
comparisons agnostic to the temperature can be made: $U_{\rm NLL}$ overall outperforms 
other methods with GPT-3.5 and Llama-2-chat on TriviaQA; $U_{\rm EigV}$ and $U_{\rm Deg}$ overall show statistically similar performance with Llama-2-chat on TriviaQA. 

\begin{figure}
    \centering
    \includegraphics[width=0.8\textwidth]{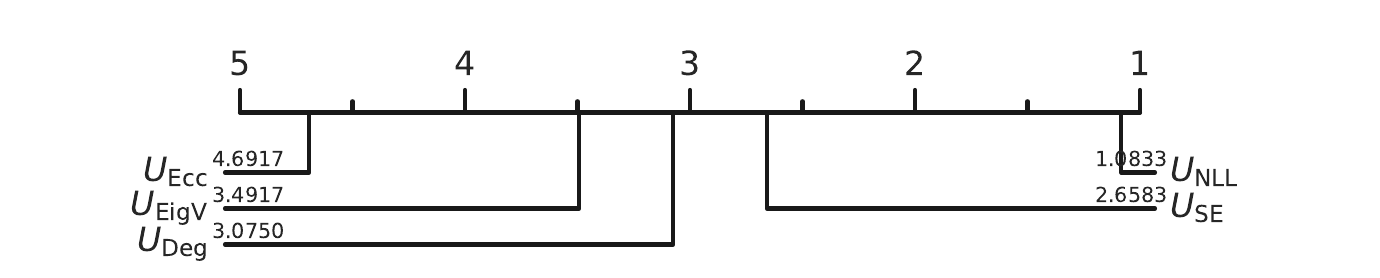}
    \includegraphics[width=0.8\textwidth]{figures/cds/llama_correctness_0.6_triviaqa/llama_metric.pdf}
    \caption{Conclusive comparison via critical difference diagrams. The first plot is with GPT-3.5-turbo on TriviaQA with temperatures 0.5, 1.0, and 1.5; the second is with Llama-2-chat on TriviaQA with temperatures 0.6 and 1.0.}
    \label{fig:robust_comparison}
\end{figure}

\subsection{Library Information}
The details of the main libraries used 
in our experiments
are as in Table~\ref{tab:lib}.

\begin{table}[H]
\centering
\begin{tabular}{|ll|ll|}
\toprule
Package & Version & Package  & Version \\
\midrule
transformer~\cite{wolf2020huggingfaces} & 4.32.1 & nltk~\cite{bird2009natural}  & 3.8.1 \\
spacy~\cite{spacy2}       & 3.6.1  & torch~\cite{NEURIPS2019_9015} & 2.0.1 \\
rouge-score~\cite{lin-2004-rouge} & 0.1.2  &     &    \\ \bottomrule
\end{tabular}
\caption{Information on main libraries used.}
\label{tab:lib}
\end{table}

\subsection{Artifact License and Terms}
We use four datasets, namely, Natural Questions, TriviaQA, SQuAD-1, and Meadow. Natural Questions is under the \textbf{CC BY-SA 3.0 license}, TriviaQA and Meadow are under the \textbf{Apache License 2.0}, and SQuAD-1 is under the \textbf{CC BY-SA 4.0 license}. 
We used two LLMs, 
namely \textit{ChatGPT-3.5} and \textit{Llama-2}. 
ChatGPT-3.5-turbo usage is subject to OpenAI's \textit{Sharing \& Publication Policy} and \textit{Usage Policies}. Llama-2 is under the {Llama-2 Community License}~\cite{Llamaacc89:online}. Our implementation and the data collected are under the \textbf{MIT License}.

Our use of the existing artifacts is consistent with their original intended use. Our created artifacts intend to verify our proposed method in our submission, which is consistent with the original access conditions. 

\section{AI Assistant Usage}
We used \textit{Copilot} to assist with coding.

\end{document}